\documentclass{article}

\usepackage[numbers]{natbib}  

\usepackage[preprint]{neurips_2022}

\usepackage[utf8]{inputenc} 
\usepackage[T1]{fontenc}    
\usepackage{hyperref}       
\usepackage{url}            
\usepackage{booktabs}       
\usepackage{amsfonts}       
\usepackage{amsmath}        
\allowdisplaybreaks   
\usepackage{amsthm}        
\usepackage{enumitem}       
\usepackage{enumerate}      
\usepackage{color}        
\usepackage{nicefrac}       
\usepackage{microtype}      
\usepackage{xcolor}         
\usepackage{extarrows}
\usepackage{graphicx}
\usepackage[T1]{fontenc}
\usepackage{caption}
\newtheorem{theorem}{Theorem}
\newtheorem{lemma}{Lemma}
\newtheorem{corollary}{Corollary}

\usepackage{graphicx}
\newcommand{\brows}[1]{%
	\begin{bmatrix}
		\begin{array}{@{\protect\rotvert\;}c@{\;\protect\rotvert}}
			#1
		\end{array}
	\end{bmatrix}
}
\newcommand{\rotvert}{\rotatebox[origin=c]{90}{$\vert$}}
\newcommand{\rowsvdots}{\multicolumn{1}{@{}c@{}}{\vdots}}

\title{Batch Normalization Is Blind to the First and Second Derivatives of the Loss}

\author{%
	Zhanpeng Zhou$^{a*}$,
	Wen Shen$^{a*}$,
	Huixin Chen$^{a}$\thanks{Equal contribution.} ,
	Ling Tang$^{a}$,
	Quanshi Zhang$^{a}$\thanks{Quanshi Zhang is the corresponding author. \texttt{zqs1022@sjtu.edu.cn}. He is with the Department of Computer Science and Engineering,
		the John Hopcroft Center and the MoE Key Lab of Artificial Intelligence, AI Institute, at the Shanghai Jiao Tong University, China.
	}
	\vspace{5pt}\\
	$^{a}$ Shanghai Jiao Tong University
}

\begin{document}

	\maketitle

	\begin{abstract}
		In this paper, we prove the effects of the BN operation on the back-propagation of the first and second derivatives of the loss. When we do the Taylor series expansion of the loss function, we prove that the BN operation will block the influence of the first-order term and most influence of the second-order term of the loss. We also find that such a problem is caused by the standardization phase of the BN operation. Experimental results have verified our theoretical conclusions, and we have found that the BN operation significantly affects feature representations in specific tasks, where losses of different samples share similar analytic formulas\footnote[1]{We will release all the codes and datasets when this paper is accepted.}.
		
		
	\end{abstract}

	\section{Introduction}\label{sec:intro}
	
	Batch normalization (BN) \citep{ioffe2015batch} plays a crucial role in deep learning and has achieved great success. However, in recent years, people gradually found some shortcomings of the BN operation in some specific applications, including the incompatibility with the dropout operation \cite{XiangLi2019UnderstandingTD}, hurting the classification accuracy in adversarial training
	\cite{xie2020adversarial,galloway2019batch}, decreasing the quality of images generated by generative models \cite{TimSalimans2016ImprovedTF}, causing gradient explosion when a stacked network was deep enough \cite{GregYang2019AMF}.

	In this paper, we analyze the BN operation from a new perspective, \emph{i.e.}, proving whether the BN operation is capable of faithfully reflecting all signals of the loss function. In other words, we hope to clarify which types of features are more likely to be learned by DNNs with the BN operation, and which types of features are less likely to be learned. Specifically, when we do the Taylor series expansion of the loss function \emph{w.r.t.} the output of the BN operation,
	{\bf we prove that the BN operation will block the back-propagation of the first and second derivatives of the loss function in the Taylor series expansion.} Therefore, the BN operation makes network layers before the BN layer blind to the first and second derivatives of the loss function, which hurts the trustworthiness of the BN operation. According to our proofs, this blindness problem is caused by the standardization phase (subtracting the mean value and dividing the standard deviation) in the BN operation.

	Specifically, given a mini-batch of $n$ samples, the BN operation normalizes features in an intermediate layer of all samples in two steps, \emph{i.e.}, the standardization phase and the affine transformation phase (learning a new mean and a new standard deviation). Let $\textbf{y}^{(i)}\in\mathbb{R}^D$ denote the output feature of the standardization phase on the $i$-th sample, the loss score computed based on $\textbf{y}^{(i)}$ is given as $\textrm{Loss}(\textbf{y}^{(i)})$\footnote[2]{In this paper, $\textrm{Loss}(\textbf{y})$ is a simplification of $\textrm{Loss}(h(\textbf{y}))$, where $h(\cdot):\mathbb{R}^D\rightarrow \mathbb{R}$ denotes the function of network layers between the BN operation and the loss function. We omit $h(\cdot)$ to simplify the description.}. A second-order Taylor series expansion of $\textrm{Loss}(\textbf{y}^{(i)})$ can be written as $\textrm{Loss}(\textbf{y}^{(i)})  =\textrm{Loss}(\tilde{\textbf{y}})+ (\textbf{y}^{(i)}-\tilde{\textbf{y}})^{\top}\textbf{g} +\frac{1}{2!}(\textbf{y}^{(i)}-\tilde{\textbf{y}})^{\top}\textbf{H}(\textbf{y}^{(i)}-\tilde{\textbf{y}})+R_2(\textbf{y}^{(i)}-\tilde{\textbf{y}})$, where $\textbf{g} \in \mathbb{R}^{D}$ and $\textbf{H} \in \mathbb{R}^{D\times D}$ denote the gradient and the Hessian matrix of the loss at the point $\tilde{\textbf{y}}\in\mathbb{R}^D$, respectively. $\tilde{\textbf{y}}$ is an arbitrary vector close to $\textbf{y}^{(i)}$.
	
	Then, let us consider the BN operation in the following two cases. The simplest case is that all samples in the mini-batch share the same analytic formula of the loss function, \textit{e.g., the loss for invertible generative models}~\citep{dinh2014nice,dinh2016density,kingma2018glow}, \textit{or the loss computed when samples in the same mini-batch have the same label.}
	In this case, we prove that the influence of the gradient ${\textbf{g}}$ and most influence of the Hessian matrix ${\bf H}$ cannot pass through the BN operation in back-propagation, if losses of all samples in the mini-batch are represented using the same Taylor series expansion at the same point $\tilde{\textbf{y}}$. In this way, all elements in ${\bf g}$ and many elements in ${\bf H}$ cannot affect layers under the BN operation.

	\newcommand\bbar[1]{\mkern1.8mu\overline{\mkern-1.8mu{#1}\mkern-1.8mu}\mkern1.8mu}	
	Besides, a more general case is that loss functions of different samples have different analytic formulas, owing to the diversity of labels, \emph{i.e.}, $\textrm{Loss}(y^{(i)},\textit{label}^{(i)})$. Thus, different samples correspond to different gradients $\textbf{g}^{(i)}$ and Hessian matrices $\textbf{H}^{(i)}$ at the point $\tilde{\textbf{y}}$. In this case, we can still compute the average gradient and the average Hessian matrix over all samples at the point $\tilde{\textbf{y}}$, \emph{i.e.}, $\bbar{\textbf{g}} = \frac{1}{n}\sum_{i=1}^{n}\textbf{g}^{(i)}$ and $\bbar{\textbf{H}} = \frac{1}{n}\sum_{i=1}^{n}\textbf{H}^{(i)}$. Just like the previous case, we prove that all elements in $\bbar{\textbf{g}}$ and many elements in $\bbar{\textbf{H}}$ cannot affect the learning of layers under the BN operation.

	In this paper, we have discovered the following three conclusions, as shown in Theorems~\ref{th:1}, \ref{th:2}, and \ref{th:3}.
	\\
	{\bf 1.} \emph{The gradient $\textbf{g}$ at the point $\tilde{\textbf{y}}$ cannot affect the learning of parameters before the BN operation.}
	\\
	{\bf 2.} \emph{Diagonal elements in the Hessian matrix $\textbf{H}$ at the point $\tilde{\textbf{y}}$ cannot affect network parameters before the BN operation.}
	\\
	{\bf 3.} \emph{For off-diagonal elements in $\textbf{H}$ at the point $\tilde{\textbf{y}}$, their impacts on the learning of network parameters before the BN operation are significantly reduced.}

	More crucially, our conclusions can be extended to neural networks whose second derivatives are all zero (\emph{e.g.}, a ReLU network has a zero Hessian matrix). For such neural networks, we can use the finite difference method~\cite{2009Scientific} to estimate an equivalent Hessian matrix, which is yielded by the change of gating states in nonlinear layers (\emph{e.g.}, the ReLU layer, or the Softmax layer).

	We attribute the above blindness problem to the standardization phase in the BN operation, while the affine transformation phase in the BN operation can learn some knowledge from the gradient ${\textbf{g}}$ to mitigate the blindness problem. In spite of that, experiments have shown that such a blindness problem affects the learning of specific neural networks. Moreover, we find that the blindness problem can be simply fixed, if we replace the BN operation with the layer normalization operation~\cite{zaremba2014recurrent}. More crucially, we believe that people have the right to know the uncommon yet non-ignorable risk in the learning process with BN operations, although such a risk does not always significantly damage the learning process. Nevertheless, the impact of the blindness problem on feature representations has been shown in experiments.

	\section{Related work}
	
	
	Normalization methods were usually designed to normalize features or parameters of deep neural networks (DNNs), so as to accelerate the training process or improve the generalization of a DNN. The BN operation~\citep{ioffe2015batch} normalized features of the same feature dimension over all samples in a mini-batch. Given a single input sample, the layer normalization (LN) \cite{zaremba2014recurrent} normalized the feature in a layer to a constant norm.
	The group normalization~\cite{wu2018group} further extended the LN by grouping feature dimensions and normalizing each group of feature dimensions separately.
	The spectral normalization ~\cite{miyato2018spectral} forced features in the discriminator in the generative adversarial network (GAN) to a unit spectral norm.
	The weight normalization (WN)~\cite{salimans2016weight} normalized the weight vector of each filter in the DNN.

	Normalization methods usually influence the representation capacity of a DNN from different perspectives.
	Santurkar \emph{et al.}~\cite{santurkar2018does} showed that the BN operation smoothed the optimization landscape.
	Luo \emph{et al.}~\cite{luo2018towards} proved that in specific applications, the BN operation was equivalent to the WN operation with a specific regularization term, which improved the generalization of a DNN.
	Morcos \emph{et al.}~\cite{morcos2018importance} found that the BN operation usually discouraged the reliance on very few feature dimensions in experiments.
	Li \emph{et al.}~\cite{XiangLi2019UnderstandingTD} showed that the BN operation was not compatible with the dropout operation. Xie \emph{et al.}~\cite{xie2020adversarial} and Galloway \emph{et al.}~\cite{galloway2019batch} discovered that in adversarial training, the BN operation usually reduced classification accuracy on both clean images and adversarial examples. The BN operation in the GAN usually decreased the visual quality of the generated images in experiments~\cite{TimSalimans2016ImprovedTF}. Yang \emph{et al.}~\cite{GregYang2019AMF} proved that a large number of BN layers in a stacked network would probably cause gradient explosion, although the BN operation in a shallow network could still smooth the loss landscape~\citep{santurkar2018does}.
	Lin \emph{et al.}~\cite{lin2021spectral} proved that the spectral normalization controled the variance of network parameters in a way that closely paralleled with the LeCun initialization~\cite{lecun2012efficient}. Xu \emph{et al.}~\cite{xu2019understanding} showed that the affine transformation in the layer normalization increased the risk of over-fitting by ablation experiments. Wan \emph{et al.}~\cite{wan2020spherical} found that gradients of the loss \emph{w.r.t.} network parameters before the BN operation were orthogonal on these parameters. This finding could also be proved by our discovery that the average gradient $\bbar{\textbf{g}}$ over all samples in the mini-batch could hardly affect the learning of network parameters before the BN operation (please see supplementary materials for the proof). The weight decay sometimes conflicted with the BN operation~\citep{van2017l2, li2020understanding}.
	When the BN operation or the WN operation was employed, backward gradients through a hidden layer was scale-invariant \emph{w.r.t.} network parameters~\cite{ba2016layer}.
	

	Unlike previous research, this study focuses on a different issue, \emph{i.e.}, exploring the BN's effects on pushing the DNN to learn specific types of knowledge. In fact, explaining knowledge representations of a DNN is an emerging direction in recent years. The information-bottleneck theory~\citep{tishby2015deep,shwartz2017opening,wolchover2017new,amjad2019learning} indicated that a DNN tended to expose task-relevant features and discarded task-irrelevant features, so as to learn discriminative features for classification. Deng \emph{et al.}~\cite{deng2021discovering} proved that a DNN usually encoded simple interactions between very few input variables and complex interactions between almost all input variables, but it was difficult to encode interactions between intermediate number of input variables.
	Unlike previous analysis of knowledge representations, we prove that the BN operation will block the influence of the first-order loss term and a considerable ratio of the influence of the second-order loss term.
	

	\section{Blindness of the BN operation}
	\label{algorithm}

	In this section, we aim to prove that the BN operation will block the back-propagation of the first and the second derivatives of the loss function. Given $n$ samples in a mini-batch, let $\textbf{X}=[\textbf{x}^{(1)}, \textbf{x}^{(2)}, \ldots, \textbf{x}^{(n)}]\in\mathbb{R}^{D\times n}$ denote features of these samples in an intermediate layer before a BN operation, where the $i$-th column $\textbf{x}^{(i)}\in\mathbb{R}^D$ corresponds to the $i$-th sample. Then, the BN operation $\textbf{Z} = \textit{BN}(\textbf{X})$ can be written as the following two phases.
	\begin{align}
		\label{eq:bn1}
		\textbf{Z} &=  \textit{diag}(\boldsymbol{\gamma})  \textbf{Y} + \boldsymbol{\beta}\textbf{1}_n^{\top}\qquad \qquad \qquad \qquad\ \ \  \text{(affine transformation phase)}\\
		\label{eq:bn2}
		\textbf{Y} &=  \textit{diag}(\boldsymbol{\sigma}\circ \boldsymbol{\sigma} +\varepsilon \textbf{1}_D)^{-\frac{1}{2}} ( \textbf{X} - \boldsymbol{\mu}\textbf{1}_n^{\top} ) 	\qquad \text{(standardization phase)}
	\end{align}
	where $\textbf{Z}=[\textbf{z}^{(1)},\textbf{z}^{(2)},\ldots,\textbf{z}^{(n)}]\in\mathbb{R}^{D \times n}$ denotes the output features of the BN operation; $\textbf{Y}=[\textbf{y}^{(1)},\textbf{y}^{(2)},\ldots,\textbf{y}^{(n)}]\in\mathbb{R}^{D \times n}$ denotes the standardized features; $\boldsymbol{\gamma},\boldsymbol{\beta}\in\mathbb{R}^D$ are used to scale and shift the standardized features;
	$\boldsymbol{\mu} = \frac{1}{n}\textbf{X}\textbf{1}_n\in\mathbb{R}^D$; $\boldsymbol{\sigma}=[\sqrt{\boldsymbol{\Sigma}_{1,1}},\ldots,\sqrt{\boldsymbol{\Sigma}_{D,D}}]^{\top}\in\mathbb{R}^D$ represents a vector of the standard deviations corresponding to diagonal elements in the covariance matrix $\boldsymbol{\Sigma} = \frac{1}{n} (\textbf{X} - \boldsymbol{\mu}\textbf{1}_n^{\top}) (\textbf{X} - \boldsymbol{\mu}\textbf{1}_n^{\top})^{\top}\in\mathbb{R}^{D\times D}$; $\textbf{1}_n\in\mathbb{R}^n$ is an all-ones vector; $\circ$ denotes the element-wise product; $\varepsilon$ is a tiny positive constant to avoid dividing zero. We ignore the $\varepsilon$ term to simplify further proofs. $\textit{diag}(\cdot)$ transforms a vector to a diagonal matrix. In this way, the training loss on the $i$-th sample can be represented as a function of the standardized feature $\textbf{y}^{(i)}$. We use the Taylor series expansion at a fixed point $\tilde{\textbf{y}}\in\mathbb{R}^D$ (which is an arbitrary vector close to $\textbf{y}^{(i)}$) to decompose the loss function \emph{w.r.t.} $\textbf{y}^{(i)}$ into terms of multiple orders, as follows.
	\begin{equation}\label{eq:TaylorExp}
		\textrm{Loss}(\textbf{y}^{(i)};\tilde{\textbf{y}})=\textrm{Loss}(\tilde{\textbf{y}})+ (\textbf{y}^{(i)}-\tilde{\textbf{y}})^{\top}\textbf{g} +\frac{1}{2!}(\textbf{y}^{(i)}-\tilde{\textbf{y}})^{\top}\textbf{H}(\textbf{y}^{(i)}-\tilde{\textbf{y}})+R_2(\textbf{y}^{(i)}-\tilde{\textbf{y}})
	\end{equation}
	where $\textbf{g} \in \mathbb{R}^{D}$ and $\textbf{H} \in \mathbb{R}^{D \times D}$ denote the gradient and the Hessian matrix of $\textrm{Loss}(\textbf{y}^{(i)})$, respectively, at the fixed point $\tilde{\textbf{y}}$; $R_2(\textbf{y}^{(i)}-\tilde{\textbf{y}})$ denotes terms of greater-than-two orders.
	
	In this paper, we consider the following two cases to discuss the effects of the BN operation on the first and second derivatives of the loss function over all samples in a mini-batch.
	
	$\bullet$ \textbf{Case 1: All samples in a mini-batch share the same analytic formula of the loss function.}
	In fact, many applications belong to this case. (1) \textbf{Example 1}: in the training of some invertible generative models~\citep{dinh2014nice,dinh2016density,kingma2018glow}, all training samples share the same analytic formula of the loss function, $\textrm{Loss}(\textbf{y}^{(i)})=-\log p(\textit{input}^{(i)})$, which measures the log-likelihood of sample $\textit{input}^{(i)}$ estimated by the model. (2) \textbf{Example 2}: in distributed learning~\cite{dean2012large}, if samples of different categories are stored in different data centers, then samples in the same data center may have the same label. In this specific application, cross-entropy losses $\textit{crossEntropy}(\textbf{y}^{(i)}, \textit{label}^*)$ of samples in the same data center can be all represented in the same analytic formula. (3) \textbf{Example 3}: a type of adversarial training \cite{zheng2020efficient} requires people to generate multiple adversarial examples in different steps of the multi-step attack, given the same input sample. Thus, just like in Example 2, if we put this set of adversarial samples in the same mini-batch to train the DNN, then all adversarial examples in the mini-batch have the same label, thereby sharing the same analytic formula of the loss function.
	
	In order to demonstrate the blindness of the BN operation, we prove that for any arbitrary loss function (including the cross-entropy loss, the MSE loss, the logistic loss, and etc.), the overall loss of all samples in a mini-batch $\textrm{Loss}^{\text{batch}}=\sum_{i=1}^{n}\textrm{Loss}(\textbf{y}^{(i)};\tilde{\textbf{y}})$ can be re-written as the sum of four compositional terms, as follows.
	\begin{equation}\label{eq:Taylor_simple}
		\textrm{Loss}^{\text{batch}}(\textbf{g},\textbf{H}) \xlongequal{\textbf{decomposed into}}
		\textrm{Loss}^{\text{constant}} + \textrm{Loss}^{\text{grad}}(\textbf{g}) +\textrm{Loss}^{\text{Hessian}}(\textbf{H}) + \sum\nolimits_i R_2(\textbf{y}^{(i)}-\tilde{\textbf{y}})
	\end{equation}
	where $\textrm{Loss}^{\text{constant}}=n\textrm{Loss}(\tilde{\textbf{y}})$ is a constant \emph{w.r.t.} input features $\textbf{X}$ of the BN operation; $\textrm{Loss}^{\text{grad}}(\textbf{g})=\sum_{i=1}^{n}(\textbf{y}^{(i)} - \tilde{\textbf{y}})^{\top}\textbf{g}$ and $\textrm{Loss}^{\text{Hessian}}(\textbf{H})=\sum_{i=1}^{n}\frac{1}{2!}(\textbf{y}^{(i)}-\tilde{\textbf{y}})^{\top}\textbf{H}(\textbf{y}^{(i)}-\tilde{\textbf{y}})$ denote the first-order and second-order terms of the loss function in the Taylor series expansion, respectively.

	\begin{theorem}\label{th:1}
		\emph{(proven in supplementary materials)}.
		{\rm$\textrm{Loss}^{\text{grad}}(\textbf{g})$} and {\rm$\textrm{Loss}^{\text{constant}}$} have no gradients on input features {\rm$\textbf{X}$} of the BN operation, i.e., {\rm$\frac{\partial \textrm{Loss}^{\text{grad}}(\textbf{g})}{\partial \textbf{X}} = \textbf{0}$} and {\rm$\frac{\partial \textrm{Loss}^{\text{constant}}}{\partial \textbf{X}} = \textbf{0}$}.
	\end{theorem}

	\begin{lemma}\label{le:1}
		\emph{(proven in supplementary materials)}.
		Let us set $\varepsilon=0$ (the tiny positive constant $\varepsilon$ is only used to avoid dividing zero, so we can ignore $\varepsilon$ in all following paragraphs to simplify the proof). Then,
		{\rm$\textrm{Loss}^{\text{Hessian}}(\textbf{H})$} can be decomposed into two terms, $i.e.$,
		{\rm\begin{equation*}
				\textrm{Loss}^{\text{Hessian}}(\textbf{H})= \textrm{Loss}^{\text{diag}}(\textbf{H}^{\text{diag}}) + \textrm{Loss}^{\text{off}}(\textbf{H}^{\text{off}}),
		\end{equation*}}
		where {\rm\small$\textrm{Loss}^{\text{diag}}(\textbf{H}^{\text{diag}}) = \sum_{i=1}^{n}\frac{1}{2!} (\textbf{y}^{(i)}-\tilde{\textbf{y}})^{\top} \textbf{H}^{\text{diag}}(\textbf{y}^{(i)}-\tilde{\textbf{y}})$}, and {\rm\small$\textrm{Loss}^{\text{off}} (\textbf{H}^{\text{off}})= \sum_{i=1}^{n}\frac{1}{2!} (\textbf{y}^{(i)}-\tilde{\textbf{y}})^{\top} \textbf{H}^{\text{off}}(\textbf{y}^{(i)}-\tilde{\textbf{y}})$; $\textbf{H}^{\text{diag}}$} and {\rm\small$\textbf{H}^{\text{off}} = \textbf{H} - \textbf{H}^{\text{diag}}$} denote the matrix only containing diagonal elements and the matrix only containing off-diagonal elements in {\rm$\textbf{H}$}, respectively.
	\end{lemma}
	
	\begin{theorem}\label{th:2}
		\emph{(proven in supplementary materials)}.
		{\rm$\textrm{Loss}^{\text{diag}}(\textbf{H}^{\text{diag}}) $} has no gradients on input features {\rm$\textbf{X}$} of the BN operation, i.e., {\rm$\frac{\partial\textrm{Loss}^{\text{diag}}(\textbf{H}^{\text{diag}}) }{\partial \textbf{X}} = \textbf{0}$}.
	\end{theorem}

	\begin{lemma}\label{le:2}
		\emph{(proven in supplementary materials)}.
		Let {\rm\small$\textbf{x}_d=[\textbf{X}_{d,1},\textbf{X}_{d,2},\ldots,\textbf{X}_{d,n}]^{\top}\in\mathbb{R}^n$} denote the $d$-th feature dimension of all the $n$ samples in a mini-batch. Then, {\rm\small$\textrm{Loss}^{\text{off}}(\textbf{H}^{\text{off}})$} can be decomposed into the loss term depending on {\rm$\textbf{x}_d$} (i.e.,
		{\rm\small$L_d(\textbf{H}_{d,:}^{\text{off}})= \textbf{H}^{\text{off}}_{d,:}(\textbf{Y} -\tilde{\textbf{y}} \textbf{1}_n ^{\top})  \textbf{y}_d$}) and the loss term independent with {\rm$\textbf{x}_d$}, thereby {\rm\small$\frac{\partial \textrm{Loss}^{\text{off}}(\textbf{H}^{\text{off}})}{\partial \textbf{x}_d} = \frac{\partial L_d(\textbf{H}_{d,:}^{\text{off}})}{\partial \textbf{x}_d}$}, where {\rm\small$\textbf{H}_{d,:}^{\text{off}}=[\textbf{H}_{d,1}^{\text{off}},\textbf{H}_{d,2}^{\text{off}},\ldots,\textbf{H}_{d,D}^{\text{off}}]\in\mathbb{R}^D$} denotes the $d$-th row of {\rm\small$\textbf{H}^{\text{off}}$}, {\rm\small$\textbf{y}_d=[\textbf{Y}_{d,1},\textbf{Y}_{d,2},\ldots,\textbf{Y}_{d,n}]^{\top}\in\mathbb{R}^n$}.
		Furthermore, {\rm\small$L_d(\textbf{H}_{d,:}^{\text{off}})$} can be decomposed as
		{\rm\begin{equation*}
				L_d(\textbf{H}_{d,:}^{\text{off}})=L_d^{\text{linear}}(\textbf{H}_{d,:}^{\text{off}})+L_d^{\text{non}}(\textbf{H}_{d,:}^{\text{off}}),
		\end{equation*}}
		where {\rm\small$L_d^{\text{linear}}(\textbf{H}_{d,:}^{\text{off}})=\textbf{H}_{d,:}^{\text{off}} \textbf{Y}^{\text{linear}}\textbf{y}_d$} and {\rm\small$L_d^{\text{non}}(\textbf{H}_{d,:}^{\text{off}})= \textbf{H}_{d,:}^{\text{off}} (\textbf{Y}^{\text{non}}-\tilde{\textbf{y}} \textbf{1}_n ^{\top})\textbf{y}_d$}. {\rm\small$\textbf{Y}^{\text{linear}}= [ \textbf{o}_d^{\top} \textbf{y}_1,  \textbf{o}_d^{\top} \textbf{y}_2, \cdots,  \textbf{o}_d^{\top} \textbf{y}_D ]^{\top} \textbf{o}_d^{\top}$} and {\rm\small$\textbf{Y}^{\text{non}}=\textbf{Y}-\textbf{Y}^{\text{linear}}$}, where {\rm\small $\textbf{o}_d=\frac{\textbf{y}_d}{\|\textbf{y}_d\|}$} denotes the unit vector in the direction of {\rm $\textbf{y}_d$}. Then, {\rm\small$L_d^{\text{linear}}(\textbf{H}_{d,:}^{\text{off}})= \lVert \textbf{y}_d \rVert  \cdot  \textbf{H}_{d,:}^{\text{off}} [ \textbf{o}_d^{\top} \textbf{y}_1,  \textbf{o}_d^{\top} \textbf{y}_2, \cdots,  \textbf{o}_d^{\top} \textbf{y}_D ]^{\top} $}.
		Therefore, {\rm\small$\frac{\partial \textrm{Loss}^{\text{off}}(\textbf{H}^{\text{off}})}{\partial \textbf{x}_d} = \frac{\partial L_d(\textbf{H}_{d,:}^{\text{off}})}{\partial \textbf{x}_d} = \frac{\partial L_d^{\text{linear}}(\textbf{H}_{d,:}^{\text{off}})}{\partial \textbf{x}_d} + \frac{\partial L_d^{\text{non}}(\textbf{H}_{d,:}^{\text{off}})}{\partial \textbf{x}_d}$}.
	\end{lemma}
	
	\begin{theorem}\label{th:3}
		\emph{(proven in supplementary materials)}.
		{\rm\small$\forall d$, $L_d^{\text{linear}}(\textbf{H}_{d,:}^{\text{off}})$} has no gradients on the $d$-th feature dimension over all $n$ samples {\rm$\textbf{x}_d\in\mathbb{R}^n$}, i.e.,
		{\rm\small\begin{equation*}
				\forall d,\ \frac{\partial L_d^{\text{linear}}(\textbf{H}_{d,:}^{\text{off}})}{\partial \textbf{x}_d} = \textbf{0},
		\end{equation*}}
		thus {\rm\small$\frac{\partial ^2 L_d^{\text{linear}}(\textbf{H}_{d,:}^{\text{off}})}{\partial \textbf{x}_d \partial \textbf{H}_{d,:}^{\text{off}}} = \frac{\partial ^2 }{\partial \textbf{x}_d \partial \textbf{H}_{d,:}^{\text{off}}} (A  [ \textbf{o}_d^{\top} \textbf{y}_1,  \textbf{o}_d^{\top} \textbf{y}_2, \cdots,  \textbf{o}_d^{\top} \textbf{y}_D ]^{\top}   ) = \textbf{0}$}, s.t. {\rm\small$A=\lVert \textbf{y}_d \rVert  \cdot  \textbf{H}_{d,:}^{\text{off}} $.} On the other hand, {\rm\small$\frac{\partial^2 L_d ^{\text{non}}(\textbf{H}_{d,:}^{\text{off}})}{\partial \textbf{x}_{d} \partial \textbf{H}_{d, :}^{\text{off}}} =  \frac{1}{\boldsymbol{\sigma}_d} \cdot [(\textbf{y}_{1}- \| \textbf{y}_{1} \| \cos(\textbf{y}_1, \textbf{y}_d)\textbf{o}_d ), \ldots,(\textbf{y}_{D} - \|\textbf{y}_{D} \| \cos(\textbf{y}_D, \textbf{y}_d) \textbf{o}_d  )]  \ne \textbf{0}$}.
	\end{theorem}
	Lemma~\ref{le:2} shows that the gradient of \(\textrm{Loss}^{\text{off}}(\textbf{H}^{\text{off}})\) \emph{w.r.t.} $\textbf{x}_d$ can be decomposed into two terms, \emph{i.e.}, {$\forall d, \frac{\partial \textrm{Loss}^{\text{off}}(\textbf{H}^{\text{off}})}{\partial \textbf{x}_d} = \frac{\partial L_d^{\text{linear}}(\textbf{H}_{d,:}^{\text{off}})}{\partial \textbf{x}_d} + \frac{\partial L_d^{\text{non}}(\textbf{H}_{d,:}^{\text{off}})}{\partial \textbf{x}_d}$}, where {$\frac{\partial L_d^{\text{linear}}(\textbf{H}_{d,:}^{\text{off}})}{\partial \textbf{x}_d} $} reflects the interaction utility yielded by feature components in all dimensions that are linear-correlated with $\textbf{y}_d$, \emph{i.e.}, $[ \textbf{o}_d^{\top} \textbf{y}_1,  \textbf{o}_d^{\top} \textbf{y}_2, \cdots,  \textbf{o}_d^{\top} \textbf{y}_D ]^{\top}$, subject to $\textbf{o}_d=\frac{\textbf{y}_d}{\|\textbf{y}_d\|}$. {$\frac{\partial L_d^{\text{non}}(\textbf{H}_{d,:}^{\text{off}})}{\partial \textbf{x}_d}$} reflects the interaction utility between $\textbf{y}_d$ and feature components that are not linearly-correlated with $\textbf{y}_d$, \emph{i.e.}, removing linearly-correlated components, $\forall j, \textbf{y}_{j}- \|\textbf{y}_j\| \cos(\textbf{y}_j, \textbf{y}_d) \textbf{o}_d$.
	Empirically, linearly-correlated feature components usually represent similar concepts, thereby having stronger interaction utilities than non-correlated feature components. We have conducted experiments to verify such an empirical claim in Section~\ref{sec:exp}. According to Theorem~\ref{th:3}, the interaction utility between linearly-correlated feature components cannot pass through the BN operation, \emph{i.e.}, {$\frac{\partial L_d^{\text{linear}}(\textbf{H}_{d,:}^{\text{off}})}{\partial \textbf{x}_d} = \textbf{0}$}. Therefore, for each dimension $d$, we can consider that a considerable ratio of the influence of {$\frac{\partial \textrm{Loss}^{\text{off}}(\textbf{H}^{\text{off}})}{\partial \textbf{x}_d}$} cannot pass through the BN operation.
	\begin{corollary}\label{co:1}
		\emph{(proven in supplementary materials)}.
		Based on Theorems~\ref{th:1}, \ref{th:2}, and \ref{th:3}, we can prove that in the training phase of a neural network with BN operations, {\rm\small$\frac{\partial ^2\textrm{Loss}^{\text{batch}}(\textbf{g},\textbf{H})}{\partial \textbf{X} \partial \textbf{g}}=\textbf{0},\frac{\partial ^2\textrm{Loss}^{\text{batch}}(\textbf{g},\textbf{H})}{\partial \textbf{X} \partial \textbf{H}^{\text{diag}}}=\textbf{0}$}, and {\rm\small$\forall d,\frac{\partial ^2\textrm{Loss}^{\text{batch}}(\textbf{g},\textbf{H})}{\partial \textbf{x}_d \partial \textbf{H}_{d,:}^{\text{off}}}=\frac{\partial ^2 L_d^{\text{linear}}(\textbf{H}_{d,:}^{\text{off}})}{\partial \textbf{x}_d \partial \textbf{H}_{d,:}^{\text{off}}}+\frac{\partial ^2 L_d^{\text{non}}(\textbf{H}_{d,:}^{\text{off}})}{\partial \textbf{x}_d \partial \textbf{H}_{d,:}^{\text{off}}}$}; where {\rm\small$\frac{\partial ^2 L_d^{\text{linear}}(\textbf{H}_{d,:}^{\text{off}})}{\partial \textbf{x}_d \partial \textbf{H}_{d,:}^{\text{off}}} = \textbf{0}$;}. In contrast, in the testing phase, {\rm\small$\frac{\partial ^2\textrm{Loss}^{\text{batch}}(\textbf{g},\textbf{H})}{\partial \textbf{X} \partial \textbf{g}}\ne\textbf{0},\frac{\partial ^2\textrm{Loss}^{\text{batch}}(\textbf{g},\textbf{H})}{\partial \textbf{X} \partial \textbf{H}^{\text{diag}}}\ne\textbf{0}$}, and {\rm\small$\frac{\partial ^2 L_d^{\text{linear}}(\textbf{H}_{d,:}^{\text{off}})}{\partial \textbf{x}_d \partial \textbf{H}_{d,:}^{\text{off}}} \ne \textbf{0}$}.
	\end{corollary}

	\textbf{Findings:} {\small$\frac{\partial ^2\textrm{Loss}^{\text{batch}}(\textbf{g},\textbf{H})}{\partial \textbf{X} \partial \textbf{g}^{\top}}=\textbf{0},\frac{\partial ^2\textrm{Loss}^{\text{batch}}(\textbf{g},\textbf{H})}{\partial \textbf{X} \partial \textbf{H}^{\text{diag}}}=\textbf{0}$}, and {\small$\forall d, \frac{\partial ^2 L_d^{\text{linear}}(\textbf{H}_{d,:}^{\text{off}})}{\partial \textbf{x}_d \partial \textbf{H}_{d,:}^{\text{off}}} = \textbf{0}$} in Corollary~\ref{co:1} show that the BN operation will block the back-propagation of the following three types of influence, \emph{i.e.}, (1) the influence of the first derivatives in $\textbf{g}$, (2) the influence of the diagonal elements in the Hessian matrix $\textbf{H}$, and (3) a considerable ratio of the influence of off-diagonal elements in the Hessian matrix $\textbf{H}$. More crucially, the BN's blocking of the above influence exists in the training phase, but it does not exist in the testing phase, which may bring in potential risks of neural networks with BN operations.

	\textbf{The reason for blindness.} According to Equations~(\ref{eq:bn1}) and (\ref{eq:bn2}), the BN operation contains two phases, the affine transformation phase and the standardization phase. We find that derivatives of $\boldsymbol{\mu}$ and $\boldsymbol{\sigma}$ in the standardization phase eliminate the influence of $\textrm{Loss}^{\text{grad}}(\textbf{g}),\textrm{Loss}^{\text{diag}}(\textbf{H}^{\text{diag}})$, and $L_d^{\text{linear}}(\textbf{H}_{d,:}^{\text{off}})$. Please see supplementary materials for the proof.
	
	Furthermore, although $\boldsymbol{\gamma}$ in the affine transformation phase can \textbf{alleviate} the blindness to $\textrm{Loss}^{\text{grad}}(\textbf{g})$ by encoding the gradient $\textbf{g}$ of the first-order term of the loss, $\textrm{Loss}^{\text{diag}}(\textbf{H}^{\text{diag}})$ and $L_d^{\text{linear}}(\textbf{H}_{d,:}^{\text{off}})$ still cannot influence parameters in all layers before the BN operations. Beyond this, the distinctive contribution of this study is to warn people for the potential defect of the BN operation in learning from specific loss information.

	\textbf{What if the neural network has no second derivations or has zero second derivatives.} This paper investigates the effects of the BN operation on the first and second derivatives of the loss function. However, some neural networks have no second derivatives or zero second derivatives, \emph{e.g.}, the loss is the output of a ReLU network that is only composed of linear layers, ReLU layers and MaxPooling layers, without other non-linear operations. Nevertheless, the instability of gating states in non-linear layers still changes the gradients of the loss. This allows the loss in Equation~(\ref{eq:TaylorExp}) to be approximated by an equivalent Hessian matrix $\textbf{H}$, which can be simply computed using the finite difference method~\cite{2009Scientific}. All our conclusions can be transferred to the equivalent Hessian matrix, which ensures the broad applicability of our analysis. Please see supplementary materials for proofs.

	\textbf{Discussions on whether we can apply the same Taylor series expansion to all samples.}
	Theoretically, if losses of all samples share the same analytic loss formula, then these losses can be represented using the same Taylor series expansion. However, people usually need another assumption, \emph{i.e.},
	the third or even higher-order terms of the loss is not large, in order to ensure the accuracy of the Taylor series expansion. To this end, we can set $\tilde{\textbf{y}}$ to the mean value $\mathbb{E}_i [\textbf{y}^{(i)}]$ to boost the fitness of the Taylor series expansion of losses on most samples. In particular, we have proven that the third and higher-order derivatives of the sigmoid function have small strengths when the classification is confident (please see supplementary materials for the proof).

	Even though $\textbf{y}^{(i)}$ is not close to $\tilde{\textbf{y}}$ in more general applications, even though the first- and second-order terms do not dominate the loss, this study is still of considerable value. It is because \textbf{the most crucial issue is whether the loss function can faithfully pass its ``\emph{all}'' effects through the BN operation.} Although the first two terms in the Taylor series expansion just take a small portion of all loss effects, people still have the right to know the truth of the BN's behavior. Therefore, the clarification of detailed potential defects of the BN's behavior has distinctive values to the safe and trustworthy AI.

	$\bullet$ \textbf{Case 2: Different samples in a mini-batch have different analytic formulas of loss functions}. For example, losses of different samples are computed with different ground-truth labels, thereby having different Taylor series expansions of losses.

	\begin{lemma}\label{le:3}
		\emph{(proven in supplementary materials)}
		The loss function for each $i$-th sample, $1\le i \le n$, can be written as {\rm $\textrm{Loss}(\textbf{y}^{(i)};\tilde{\textbf{y}}) = \textrm{Loss}(\tilde{\textbf{y}}) + \textrm{Loss}^{\text{grad}}(\bbar{\textbf{g}}) +\textrm{Loss}^{\text{diag}}(\bbar{\textbf{H}}^{\text{diag}}) + \textrm{Loss}^{\text{off}}(\bbar{\textbf{H}}^{\text{off}}) + \textrm{Loss}^{\text{grad}}(\textbf{g}^{\prime (i)}) +\textrm{Loss}^{\text{diag}}(\textbf{H}^{\prime (i) \text{diag}}) + \textrm{Loss}^{\text{off}}(\textbf{H}^{\prime (i) \text{off}}) + R_2(\textbf{y}^{(i)}-\tilde{\textbf{y}})$}, where {\rm$\bbar{\textbf{g}} = \frac{1}{n}\sum_{i=1}^{n}\textbf{g}^{(i)}$} and {\rm$\bbar{\textbf{H}} = \frac{1}{n}\sum_{i=1}^{n}\textbf{H}^{(i)}$} denote the average gradient and the average Hessian matrix, respectively, which are shared by all samples in the mini-batch. {\rm$\textbf{g}^{\prime (i)} = \textbf{g}^{(i)} - \bbar{\textbf{g}}$} and {\rm$\textbf{H}^{\prime (i)} = \textbf{H}^{(i)} - \bbar{\textbf{H}}$} denote the distinctive gradient and the distinctive Hessian matrix of the $i$-th sample, respectively.
	\end{lemma}
	
	Just like Case 1, we prove that $\textrm{Loss}^{\text{grad}}(\bbar{\textbf{g}})$ and $\textrm{Loss}^{\text{diag}}(\bbar{\textbf{H}}^{\text{diag}})$ have no gradients on features before the BN operation, and $\textrm{Loss}^{\text{off}}(\bbar{\textbf{H}}^{\text{off}})$ does not have strong effects on features before the BN operation. Please see supplementary materials for the proof.

	\section{Experiments}\label{sec:exp}
	
	In this section, we conducted experiments to verify theorems above, and analyzed the impact of the blindness problem on feature representations of DNNs.
	
	$\bullet$ \textbf{Experimental verification of Theorems~\ref{th:1} and \ref{th:2}}. We conducted two experiments on synthetic studies and real applications, respectively, to verify that $\textrm{Loss}^{\text{grad}}(\textbf{g})$ and $\textrm{Loss}^{\text{diag}}(\textbf{H}^{\text{diag}})$ had no gradients on features before the BN operation.
	
	\emph{Verifying the blindness of the BN operation on the synthesized loss functions.} In this experiment, we directly designed the loss function as a polynomial with derivatives of different orders, in order to evaluate the effects of the BN operation on derivatives of different orders. To this end, we synthesized a group of five loss functions with derivatives of different orders, $\forall 0\le k\le 4$, $\textrm{Loss}_k( \textbf{y} | \boldsymbol{\lambda}) =\sum_{i=1}^{n} \sum_{k^\prime=k}^4
	\lambda_{k^\prime} (y^{(i)})^{k^\prime}$, where parameters $\boldsymbol{\lambda} = [\lambda_0,\lambda_1,\ldots,\lambda_4]^{\top} \sim N(\boldsymbol{\mu}=\textbf{0},\Sigma=I_{5\times 5})$ were sampled from a Gaussian distribution for 1000 times. Thus, we constructed a dataset\footnotemark[1] containing 1000 groups of loss functions. Then, for each group of losses \emph{w.r.t.} a specific $\boldsymbol{\lambda}$, we put each $k$-th loss, $\textrm{Loss}_k$, on a BN operation, which was on the top of a 5-layer MLP network (each layer containing $M$=100 neurons), to construct a neural network. The input of the MLP was high-dimensional noises sampled from a Gaussian distribution $N(\boldsymbol{\mu}=\textbf{0},\Sigma=I_{M\times M})$, and the output was a scalar after a linear operation. Thus, we obtained a group of five MLPs with \{$\textrm{Loss}_k$\}, $0\le k\le4$, for testing.

	We measured {$\Delta \text{grad}_q = \mathbb{E}_{\boldsymbol{\lambda}} \left[ \lVert \frac{\partial \textrm{Loss}_q( \textbf{y} | \boldsymbol{\lambda})  }{\partial \textbf{x}}-\frac{\partial \textrm{Loss}_{q+1}( \textbf{y}| \boldsymbol{\lambda}) }{\partial \textbf{x}} \rVert  /  \lVert \frac{\partial \textrm{Loss}_q( \textbf{y}| \boldsymbol{\lambda})  }{\partial \textbf{x}} \rVert \right],q=0,1,2,3$}, to examine whether the $q$-th order term of the loss could pass its influence through the BN operation. Table~\ref{tab:analytic} shows that $\Delta \text{grad}_0 = 0, \Delta \text{grad}_1 \approx 0$, and $\Delta \text{grad}_2 \approx 0$, which proved that the zeroth-order, the first-order, and the second-order terms\footnote{In this experiment, the Hessian matrix reduced to the scalar second derivative.} of the loss could not pass their influence through the BN operation. Besides, $\Delta \text{grad}_3=0.51 \pm 0.33$ indicated that the third-order term of the loss in the Taylor series expansion could successfully pass its influence through the BN operation.

	\begin{table}
		\centering
		\caption{Verifying the BN's effects on the back-propagation of derivatives of different orders. $\Delta \text{grad}_q \approx 0$ indicated that derivatives of the $q$-th order did not pass through the BN operation. The small deviation was caused by the accumulation of tiny systematic computational errors in a DNN.}
		\vspace{1pt}
		\label{tab:analytic}
		\resizebox{0.5\linewidth}{!}{
			\begin{tabular}{cccc}
				\toprule
				$\Delta \text{grad}_0$     & $\Delta \text{grad}_1$ & $\Delta \text{grad}_2$ & $\Delta \text{grad}_3$  \\
				\midrule
				0 {\small $\pm$ 0}& 1.79e-7 {\small $\pm$ 2.07e-7} & 2.85e-7
				{\small $\pm$ 3.94e-8} & 0.51 {\small $\pm$ 0.33} \\
				\bottomrule
		\end{tabular}}
		\vspace{-5pt}
	\end{table}

	\emph{Verifying the blindness of the BN operation in image classification}. We trained a VGG-11 network \cite{vgg} for image classification on the CIFAR-10 dataset \cite{krizhevsky2009learning}, in which we added a BN layer before the second top FC layer. Then, we tested the blindness of the BN operation to the gradient $\textbf{g}$ and diagonal elements in the Hessian matrix $\textbf{H}$ of the VGG-11 network. Specifically, in order to determine ground-truth effects of the first-order and the second-order terms in the Taylor series expansion of the classification loss ($\textrm{Loss}^* = \sum_{i=1}^{n}\textrm{Loss}^{\text{cls}}(\textbf{y}^{(i)})$), we manually added noisy first derivatives and noisy second derivatives to construct three additional losses ($\textrm{Loss}_2$, $\textrm{Loss}_3$, and $\textrm{Loss}_4$). \emph{I.e.}, {$\textrm{Loss}_2 = \sum_{i=1}^{n} (\textrm{Loss}^{\text{cls}}(\textbf{y}^{(i)}) + \boldsymbol{\epsilon}^{\top} \textbf{y}^{(i)})$, $\textrm{Loss}_3 = \sum_{i=1}^{n}(\textrm{Loss}^{\text{cls}}(\textbf{y}^{(i)}) + (\textbf{y}^{(i)}-\tilde{\textbf{y}})^{\top} \textit{diag}(\boldsymbol{\epsilon})(\textbf{y}^{(i)}-\tilde{\textbf{y}}))$}, and {$\textrm{Loss}_4 = \sum_{i=1}^{n}(\textrm{Loss}^{\text{cls}}(\textbf{y}^{(i)}) + (\textbf{y}^{(i)}-\tilde{\textbf{y}})^{\top} \textbf{E}^{\text{off}}(\textbf{y}^{(i)}-\tilde{\textbf{y}}))$}. We set $\tilde{\textbf{y}}=\frac{1}{n}\sum_{i=1}^{n}\textbf{y}^{(i)}=\textbf{0}$ in this experiment. Each element in $\boldsymbol{\epsilon}\in\mathbb{R}^{D}$ and $\textbf{E}^{\text{off}}\in\mathbb{R}^{D\times D}$ was sampled from a Gaussian distribution $N(\mu=0,\sigma^2=0.1^2)$. Diagonal elements in $\textbf{E}^{\text{off}}$ were set to 0.

	In this way, if {\small$\frac{\partial \textrm{Loss}^* }{\partial \textbf{X}}=\frac{\partial \textrm{Loss}_2 }{\partial \textbf{X}} =\frac{\partial \textrm{Loss}_3 }{\partial \textbf{X}} $}, then it proved that the first derivatives in $\textbf{g}$ and diagonal elements in $\textbf{H}$ could not pass their influence through the BN operation. To this end, we used metrics {\small$\Delta\text{grad}^{\text{first}} = \lVert \frac{\partial \textrm{Loss}^* }{\partial \textbf{X}}-\frac{\partial \textrm{Loss}_2 }{\partial \textbf{X}} \rVert_F / \lVert \frac{\partial \textrm{Loss}^* }{\partial \textbf{X}} \rVert_F$, $\Delta\text{grad}^{\text{second,diag}} = \lVert \frac{\partial \textrm{Loss}^* }{\partial \textbf{X}}-\frac{\partial \textrm{Loss}_3 }{\partial \textbf{X}} \rVert_F / \lVert \frac{\partial \textrm{Loss}^*  }{\partial \textbf{X}} \rVert_F$}, and {\small$\Delta\text{grad}^{\text{second,off}} =\lVert \frac{\partial \textrm{Loss}^*  }{\partial \textbf{X}}-\frac{\partial \textrm{Loss}_4 }{\partial \textbf{X}} \rVert_F / \lVert \frac{\partial \textrm{Loss}^*}{\partial \textbf{X}} \rVert_F$} to measure the influence of the first derivatives in $\textbf{g}$, diagonal elements in $\textbf{H}$, and off-diagonal elements in $\textbf{H}$ on the gradient $\frac{\partial \textrm{Loss}^*  }{\partial \textbf{X}}$. According to Table~\ref{tab:noise}, $\Delta\text{grad}^{\text{first}} \approx 0$ and $\Delta\text{grad}^{\text{second,diag}}\approx 0$ proved that elements in $\textbf{g}$ and diagonal elements in $\textbf{H}$ could not pass their influence through the BN operation. Besides, the large value of $\Delta\text{grad}^{\text{second,off}}$ indicated that off-diagonal elements in $\textbf{H}$ could pass their influence through the BN operation.
	
	\begin{table}[t]
		\centering
		\caption{Verifying the BN's effects on the back-propagation of the first and second derivatives. $\Delta\text{grad}^{\text{first}}\approx 0$ and $\Delta\text{grad}^{\text{second,diag}} \approx 0$ indicated that the gradient $\textbf{g}$ and diagonal elements in Hessian matrix $\textbf{H}$ could not pass their influence through the BN operation. The small deviation was caused by the accumulation of tiny systematic computational errors in a DNN.}
		\vspace{1pt}
		\label{tab:noise}
		\resizebox{0.5\linewidth}{!}{\begin{tabular}{ccc}
				\toprule
				$\Delta\text{grad}^{\text{first}}$     & $\Delta\text{grad}^{\text{second,diag}} $     & $\Delta\text{grad}^{\text{second,off}} $ \\
				\midrule
				1.37e-8 {\small $\pm$ 2.72e-8} & 2.72e-8 {\small $\pm$ 5.41e-8} & 1019.49 {\small $\pm$ 116.45} \\
				\bottomrule
		\end{tabular}}
		\vspace{-5pt}
	\end{table}
	
	$\bullet$ \textbf{Experimental verification of Theorem~\ref{th:3}}. In this experiment, we aimed to verify that {$L_d^{\text{linear}}(\textbf{H}_{d,:}^{\text{off}})$} had no gradients on features before the BN operation. To this end, we directly computed the norm of the gradient {\small$\lVert  \frac{\partial L_d^{\text{linear}}(\textbf{H}_{d,:}^{\text{off}})}{\partial \textbf{x}_d} \rVert$, where $\textbf{H}_{d,:}^{\text{off}}$} was computed following~\cite{cohen2020gradient}. To comprehensively test the BN on different layers, we revised the AlexNet~\cite{krizhevsky2012imagenet} by adding five additional FC layers before the top FC layer, and revised the LeNet~\cite{lecun1989backpropagation} by adding seven additional FC layers before the top FC layer. Please see supplementary materials about the revised architectures. For each DNN, we added a BN layer before the 1st, 2nd, and 3rd top FC layers, respectively, to construct \textit{AlexNet-1, AlexNet-2, AlexNet-3} and \textit{LeNet-1, LeNet-2, LeNet-3}. Table~\ref{tab:th3} reports {\small$\lVert  \frac{\partial L_d^{\text{linear}}(\textbf{H}_{d,:}^{\text{off}})}{\partial \textbf{x}_d} \rVert$} averaged over different feature dimensions ($d$) and different batches. It proved that {$L_d^{\text{linear}}(\textbf{H}_{d,:}^{\text{off}})$} had no gradients on features before the BN operation.
	
	\begin{table}[t]
		\centering
		\caption{Gradients of {\small$L_d^{\text{linear}}(\textbf{H}_{d,:}^{\text{off}}) $} \emph{w.r.t.} the input feature $\textbf{x}_d$ of the BN operation. The small deviation was caused by the accumulation of tiny systematic computational errors in a DNN.}
		\vspace{1pt}
		\label{tab:th3}
		\resizebox{\linewidth}{!}{\begin{tabular}{lcccccc}
				\toprule
				& AlexNet-1     & AlexNet-2  & AlexNet-3  &  LeNet-1 & LeNet-2 & LeNet-3 \\
				\midrule
				{\small$\lVert  \frac{\partial L_d^{\text{linear}}(\textbf{H}_{d,:}^{\text{off}})}{\partial \textbf{x}_d} \rVert$} &  7.12e-12 {\small $\pm$ 8.64e-12}   &  2.85e-11 {\small $\pm$ 5.11e-11} &  3.86e-11 {\small $\pm$ 7.45e-11} &  1.01e-10 {\small $\pm$ 1.98e-9} &  5.17e-12 {\small $\pm$ 8.75e-12} &  1.45e-11 {\small $\pm$ 2.28e-11} \\
				\bottomrule
		\end{tabular}}
		\vspace{-8pt}
	\end{table}
	
	\begin{table}[t]
		\centering
		\caption{Comparing the relative significance between {\small$\partial L_d^{\text{linear}}(\textbf{H}_{d,:}^{\text{off}}) / \partial \textbf{y}_d $} and {\small$\partial L_d^{\text{non}}(\textbf{H}_{d,:}^{\text{off}}) / \partial \textbf{y}_d $}.}
		\vspace{1pt}
		\label{tab:dominating}
		\resizebox{0.95\linewidth}{!}{\begin{tabular}{lcccccc}
				\toprule
				& AlexNet-1     & AlexNet-2  & AlexNet-3 &  LeNet-1 & LeNet-2 & LeNet-3 \\
				\midrule
				{\small$\lVert \frac{ \partial L_d^{\text{linear}}(\textbf{H}_{d,:}^{\text{off}})}{\partial \textbf{y}_d} \rVert / \lVert \frac{\partial L_d(\textbf{H}_{d,:}^{\text{off}})}{\partial \textbf{y}_d} \rVert$}&  {\bf0.84 {\small $\pm$ 0.19} } &  {\bf0.73 {\small $\pm$ 0.25} } & {\bf 0.85 {\small $\pm$ 0.22} } & {\bf 0.67 {\small $\pm$ 0.29} } &{\bf  0.68 {\small $\pm$ 0.29} } & {\bf 0.77 {\small $\pm$ 0.27}  }\\
				{\small$\lVert \frac{ \partial L_d^{\text{non}}(\textbf{H}_{d,:}^{\text{off}})}{\partial \textbf{y}_d} \rVert / \lVert \frac{\partial L_d(\textbf{H}_{d,:}^{\text{off}})}{\partial \textbf{y}_d} \rVert$} & 0.47 {\small $\pm$ 0.22} &  0.56 {\small $\pm$ 0.29} &  0.40 {\small $\pm$ 0.28} &  0.63 {\small $\pm$ 0.27} &  0.61 {\small $\pm$ 0.29} &  0.51 {\small $\pm$ 0.27}   \\
				\bottomrule
		\end{tabular}}
	\end{table}

	$\bullet$ \textbf{Examining that interactions between linearly-correlated feature components took a main part of the loss {$L_d(\textbf{H}_{d,:}^{\text{off}})$}.} According to Lemma~\ref{le:2}, {$L_d(\textbf{H}_{d,:}^{\text{off}})=L_d^{\text{linear}}(\textbf{H}_{d,:}^{\text{off}})+L_d^{\text{non}}(\textbf{H}_{d,:}^{\text{off}})$}. In this experiment, we aimed to show that the {$L_d^{\text{linear}}(\textbf{H}_{d,:}^{\text{off}})$} term had considerable gradients \emph{w.r.t.} $\textbf{y}_d$, which took a large part of {\small$\frac{\partial L_d(\textbf{H}_{d,:}^{\text{off}})}{\partial \textbf{y}_d}$}. If so, it means that the BN's effects of eliminating all gradient of {$L_d^{\text{linear}}(\textbf{H}_{d,:}^{\text{off}})$} on $\textbf{x}_d$, \emph{i.e.}, {\small$\frac{\partial L_d^{\text{linear}}(\textbf{H}_{d,:}^{\text{off}})}{\partial \textbf{x}_d}=\textbf{0}$}, had significant influence on the training process. To this end, we computed {\small$\lVert \frac{ \partial L_d^{\text{linear}}(\textbf{H}_{d,:}^{\text{off}})}{\partial \textbf{y}_d} \rVert / \lVert \frac{\partial L_d(\textbf{H}_{d,:}^{\text{off}})}{\partial \textbf{y}_d} \rVert$} to measure the relative significance of the compositional influence of {\small$ \frac{ \partial L_d^{\text{linear}}(\textbf{H}_{d,:}^{\text{off}})}{\partial \textbf{y}_d}$} to {\small$ \frac{\partial L_d(\textbf{H}_{d,:}^{\text{off}})}{\partial \textbf{y}_d} $}, and we computed {\small$\lVert \frac{ \partial L_d^{\text{non}}(\textbf{H}_{d,:}^{\text{off}})}{\partial \textbf{y}_d} \rVert / \lVert \frac{\partial L_d(\textbf{H}_{d,:}^{\text{off}})}{\partial \textbf{y}_d} \rVert$} to measure the relative significance of the compositional influence of {\small$\frac{ \partial L_d^{\text{non}}(\textbf{H}_{d,:}^{\text{off}})}{\partial \textbf{y}_d} $} to {\small$ \frac{\partial L_d(\textbf{H}_{d,:}^{\text{off}})}{\partial \textbf{y}_d} $}. We conducted experiments on \textit{AlexNet-1, AlexNet-2, AlexNet-3, LeNet-1, LeNet-2}, and \textit{LeNet-3} as introduced above. Table~\ref{tab:dominating} reports the relative significance averaged over different feature dimensions ($d$) and different batches. It proved that {$L_d^{\text{linear}}(\textbf{H}_{d,:}^{\text{off}})$} had considerable impacts on the overall gradient of {$ L_d(\textbf{H}_{d,:}^{\text{off}})$} \emph{w.r.t.} $\textbf{y}_d$, which demonstrated that the BN's elimination of the gradient {$\frac{ \partial L_d^{\text{linear}}(\textbf{H}_{d,:}^{\text{off}})}{\partial \textbf{y}_d}$} had non-negligible impacts.

	$\bullet$ \textbf{Analyzing the features learned by neural networks with BN operations.}
	We conducted experiments to compare DNNs with BN operations and DNNs without BN operations in different applications, in order to analyze the influence of the BN operation on feature representations.
	
	\emph{Experiment 1:} We measured the BN's effects on feature representations of the invertible generative model RealNVP~\cite{dinh2016density}. The training of RealNVPs belonged to Case 1 that losses of all samples in a mini-batch shared the same analytic formula, \emph{i.e.}, $-\log p(\textit{input})$. The vanilla RealNVP in \cite{dinh2016density} had BN operations, thus being termed \textit{RealNVP-BN}. Besides, we constructed another RealNVP by replacing all BN layers with LN layers\footnote{Please see supplementary materials about how to invert features in RealNVP-LN.}, as a baseline (namely \textit{RealNVP-LN}) for comparison. All RealNVPs were trained on the MNIST dataset \cite{lecun1998gradient}. In addition, the supplementary material shows results on more RealNVP models with various revised architectures, which also yielded similar conclusions.
	
	We tested whether a RealNVP model could successfully distinguish real images and fake images. This was the key capacity for a generative model. Specifically, a well-trained RealNVP was supposed to predict high log-likelihood on real images $I_1^{\text{real}}$ and $I_2^{\text{real}}$, and yield low log-likelihood on fake images. We generated fake images by linear interpolation, $I_{\alpha}^{\text{inter}} = \alpha I_1^{\text{real}} + (1-\alpha) I_2^{\text{real}}, \alpha\in(0,1)$. To sharpen the difference caused by the BN operation, we learned different groups of RealNVP-BN and RealNVP-LN, each being trained to generate a specific pair of categories, as Figure~\ref{fig:flow} shows. Then, for each RealNVP, we computed the average log-likelihood of interpolated images $I_{\alpha}^{\text{inter}}$ \emph{w.r.t.} a specific interpolation rate $\alpha$, \emph{i.e.}, $\mathbb{E}_{\forall I_1^{\text{real}},I_2^{\text{real}}}[ \log p(I_{\alpha}^{\text{inter}}) ]$. Figure~\ref{fig:flow} shows that RealNVP-LN usually assigned much higher log-likelihood with real images (\emph{i.e.}, images at the points of $\alpha=0$ and $\alpha=1$) than interpolated images. In comparison, RealNVP-BN could not significantly distinguish real images and interpolated images. It may be because that the first derivative of the loss usually reflected prototype features for different categories. Thus, the blindness to the first derivative prevented the RealNVP from modeling prototype features of different digits.

	\begin{figure}
		\begin{center}
			\includegraphics[width=\linewidth]{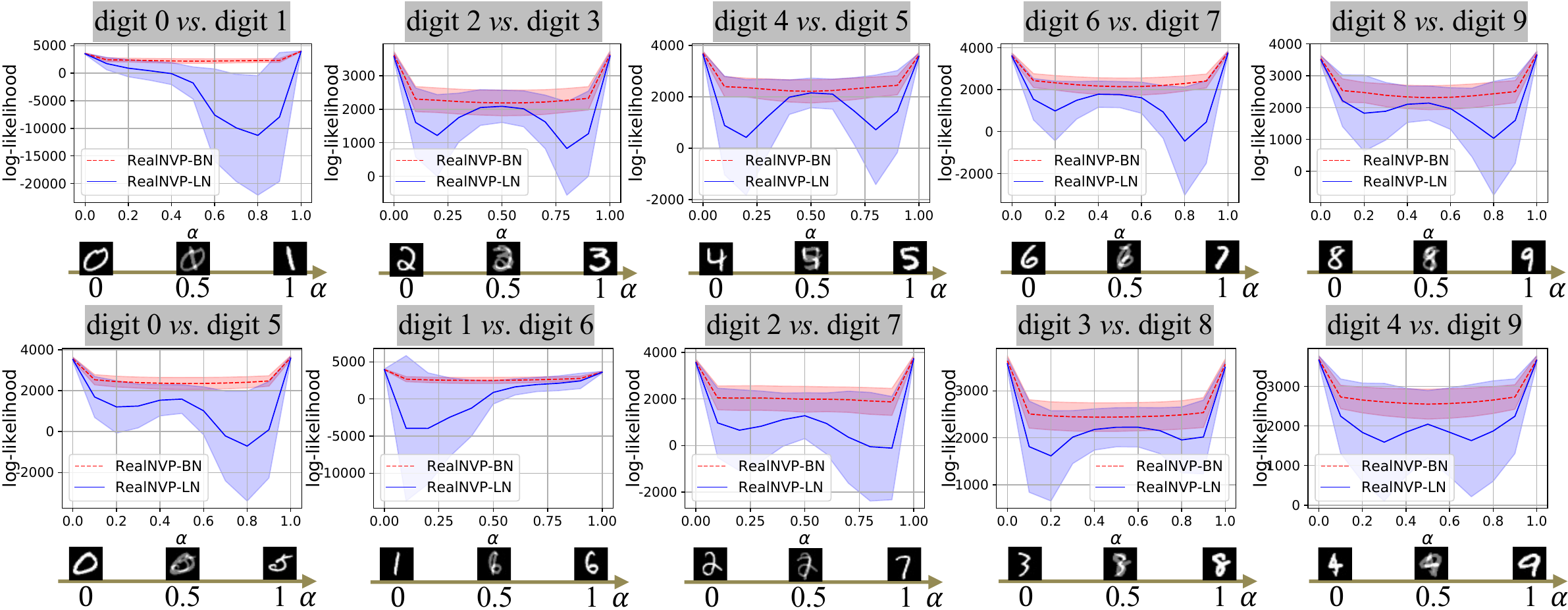}
		\end{center}
		\caption{Log-likelihood of real images (at $\alpha=0$ and $\alpha=1$) and interpolated images generated by the RealNVP-BN and the RealNVP-LN. The shaded area represents the standard deviation.}
		\vspace{-5pt}
		\label{fig:flow}
	\end{figure}

	\emph{Experiment 2:} We measured the BN's effects on feature representations for classification. We compared four groups of DNNs trained on the CIFAR-10 dataset \cite{krizhevsky2009learning}, as follows. The first group of DNNs did not contain any normalization operations, namely \textit{DNN-ori}. The second group of DNNs were obtained by adding BN operations to the DNN-ori, termed \textit{DNN-BN}. The third group of DNNs were obtained by replacing all BN operations in the DNN-BN with LN operations, namely \textit{DNN-LN}. The above three groups of DNNs were trained when each mini-batch only contained samples in a specific category. Such an experimental setting ensured that losses of all samples in the mini-batch shared the same analytic formula. The fourth group of DNNs had the same architecture as the DNN-ori, but they were trained when all samples in a mini-batch had different labels, termed \textit{DNN-ori-base}. We selected classic network architectures for classification, including VGG-11, VGG-16 \cite{vgg}, ResNet-18, ResNet-34 \cite{he2016deep}, DenseNet-121, and DenseNet-169 \cite{huang2017densely}. For VGG-11 and VGG-16, we constructed and learned all four groups of DNNs. Specifically, we added one single BN (or LN) layer before the second top FC layer of the VGG network. For DNNs already containing BN operations, \emph{i.e.}, ResNet-18/34 and DenseNet-121/169, we directly selected these DNNs as DNN-BN networks, and there were no DNN-ori networks for such DNNs. In addition, for such DNNs, we also learned two baselines, namely \textit{DNN-BN-base} and \textit{DNN-LN-base}, which were trained when all samples in a mini-batch had different labels.

	In Figure~\ref{fig:vgg}, we used t-SNE \cite{maaten2008visualizing} to visualize the input feature of the top FC layer of each DNN, in order to compare the BN's effects on the feature representation. When losses of all samples in a mini-batch shared the same analytic formula, features of the DNN-BN on different samples were far less clustered than those of the DNN-ori and the DNN-LN. This experiment indicates that the BN operation prevented the DNN from learning discriminative features of samples in each specific category. It was because the first derivatives on the average sample $\tilde{\textbf{y}}$ in a specific category usually contained information of common discriminative features shared by different samples in this category. In comparison, when losses of all samples in a mini-batch had different analytic formulas, features of the DNN-BN-base on different samples could be well clustered, which verified that the blindness problem could be alleviated in the classification task. Nevertheless, people had the right to know the BN's potential risk, although such a risk did not always damage the system in some applications.
	
	
	\begin{figure}
		\begin{center}
			\includegraphics[width=\linewidth]{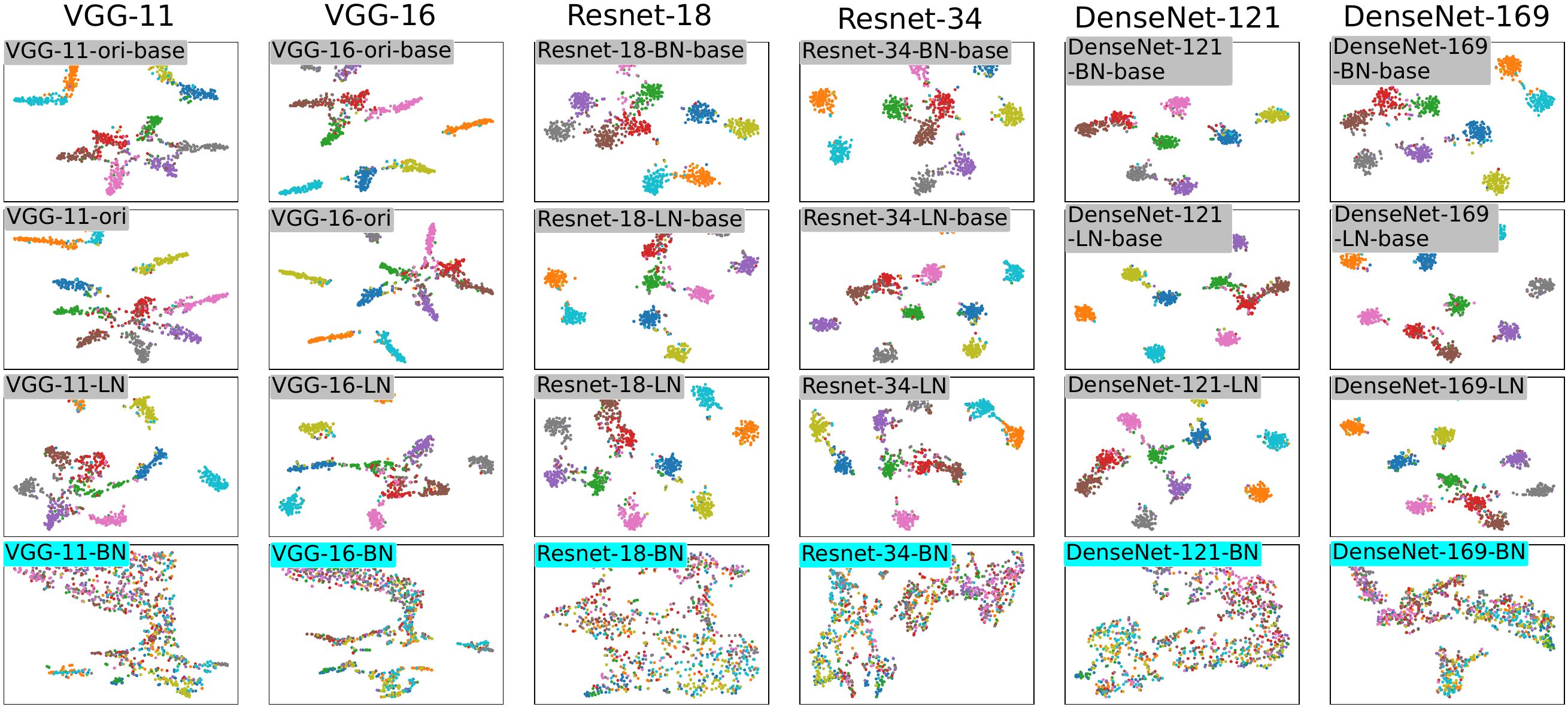}
		\end{center}
		\caption{Comparing features of DNNs with BN layers and DNNs without BN layers. Points of different colors indicate samples of different categories. The DNN-BN usually learned much less clustered features than the DNN-LN, when these two DNNs were learned with the same settings.}
		\vspace{-5pt}
		\label{fig:vgg}
		\vspace{-5pt}
	\end{figure}

	\section{Conclusions and discussion}
	In this paper, we have discovered and theoretically proven the intrinsic blindness problem with the BN operation. Such a problem may bring in an uncommon yet non-ignorable risk in the learning of DNNs. Experiments have verified our theoretical proofs. Besides, experiments have demonstrated that in specific applications, the blindness of the BN operation prevents the DNN from learning discriminative features.
	
	\textit{Advantages and disadvantages of the BN operation.} Moreover, we have also proven that it is the standardization phase of the BN operation causes the above effects of the BN operation. However, it is difficult to obtain a simple conclusion that the standardization phase is harmful. In fact, Xu \emph{et al.}~\cite{xu2019understanding} proved that the standardization phase in the LN operation re-centered gradients of the loss \emph{w.r.t.} features, and reduced the variance of these gradients. Xu \emph{et al.} also found that the standardization phase improved the performance of DNNs in experiment. Besides, Liu \emph{et al.}~\cite{liu2021trap} showed that the standardization phase could effectively alleviate the ``self-enhancement'' phenomenon, and avoided hurting the diversity of features in the DNN. Nevertheless, we have proven that the standardization phase causes the blindness to the first and second derivatives of the loss function in specific applications. To this end, we use the LN operation to replace the BN operation in such applications, which avoids the blindness problem.
		
	\small
	
	\bibliographystyle{unsrtnat}
	\bibliography{ref}

	\newpage
	\appendix
	
	\setcounter{equation}{0}
	\setcounter{theorem}{0}
	\setcounter{lemma}{0}
	\setcounter{corollary}{0}
	
	\section{Overview}
	This Appendix provides theoretical proofs of theorems, lemmas, and the corollary in the main paper, and includes more implementation details and results about experiments. In Section~\ref{appendix:pre}, we revisit the BN operation detailly. Section~\ref{appendix:lemma} provides theoretical proofs of Lemmas~\ref{le:1}, \ref{le:2}, and \ref{le:3} in the main paper. Section~\ref{appendix:theorem} provides theoretical proofs of Theorems~\ref{th:1}, \ref{th:2}, \ref{th:3}, and Corollary~\ref{co:1} in the main paper. Section~\ref{appendix:case2} provides conclusions in Case 2 and theoretical proofs of these conclusions. In Section~\ref{appendix:sigmoid}, we prove that the third and higher-order derivatives of the sigmoid function have small strengths when the classification is confident. In Section~\ref{appendix:zero}, we prove that our conclusions can be extented to DNNs have no second derivations or have zero second derivatives. In Section~\ref{appendix:reason}, we prove the reason of the blindness problem is that derivatives of $\boldsymbol{\mu}$ and $\boldsymbol{\sigma}$ in the standardization phase of the BN operation eliminate the influence of the gradient $\textbf{g}$, and eliminate most influence of the Hessian matrix $\textbf{H}$. In Section~\ref{appendix:exp}, we provide more implementation details and more results about experiments.
	
	
	\section{Preliminary: Batch Normalization}\label{appendix:pre}
	
	In this section,  in order to help readers understand the analysis in the paper, we first revisit the details of Batch Normalization (BN)~\cite{ioffe2015batch}.
	
	{\bf Notations.} Given $n$ samples in a mini-batch, let $\textbf{X}=[\textbf{x}^{(1)}, \textbf{x}^{(2)}, \ldots, \textbf{x}^{(n)}] =[ \textbf{x}_{1}^{\top}; \textbf{x}_{2}^{\top}; \ldots; \textbf{x}_{D}^{\top}]\in\mathbb{R}^{D\times n}$ denote features of these samples in an intermediate layer before a BN operation. The $i$-th column vector corresponds to the $i$-th sample with  $D$-dimensional feature, which is written as $\textbf{x}^{(i)}\in \mathbb{R}^{D}$, $i = 1, \cdots, n$; the $d$-th row vector corresponds to the $d$-th feature dimension of all the $n$ samples, the transposition of which is written as $\textbf{x}_d \in \mathbb{R}^{n}$,  $d = 1, \cdots, D$. For the convenience of readers, we represent $\textbf{X}$ as following forms for an intuitive understanding of above description.
	
	\begin{equation}\label{eq:X}
		\textbf{X}=
		\begin{bmatrix}
			\vert & \vert & \cdots & \vert \\
			\textbf{x}^{(1)}& \textbf{x}^{(2)}& \cdots & \textbf{x}^{(n)}  \\
			\vert & \vert & \cdots & \vert
		\end{bmatrix}=
		\brows{\textbf{x}_{1}^{\top}  \\ \textbf{x}_{2}^{\top}  \\ \rowsvdots \\ \textbf{x}_{D}^{\top}}=
		\begin{bmatrix}
			\textbf{X}_{1,1} & \textbf{X}_{1,2} & \cdots  & \textbf{X}_{1,n}\\
			\textbf{X}_{2,1} & \textbf{X}_{2,2} & \cdots  & \textbf{X}_{2,n}\\
			\vdots & \vdots & \ddots & \vdots \\
			\textbf{X}_{D,1} & \textbf{X}_{D,2} & \cdots  & \textbf{X}_{D,n}\\
		\end{bmatrix}. \nonumber
	\end{equation}
	
	Then, the BN operation $\textbf{Z} = \textit{BN}(\textbf{X})$ normalizes $\textbf{X}$ as follows.
	\begin{align}
		\label{eq:bn1}
		\textbf{Z} &=  \textit{diag}(\boldsymbol{\gamma})  \textbf{Y} + \boldsymbol{\beta}\textbf{1}_n^{\top}\qquad \qquad \qquad \qquad\ \ \  \text{(affine transformation phase)} \\
		\label{eq:bn2}
		\textbf{Y} &=  \textit{diag}(\boldsymbol{\sigma}\circ \boldsymbol{\sigma} +\varepsilon \textbf{1}_D)^{-\frac{1}{2}} ( \textbf{X} - \boldsymbol{\mu}\textbf{1}_n^{\top} ) 	\qquad \text{(standardization phase)}
	\end{align}
	where $\textbf{Z}=[\textbf{z}^{(1)},\textbf{z}^{(2)},\ldots,\textbf{z}^{(n)}]\in\mathbb{R}^{D \times n}$ denotes output features of the BN operation; $\textbf{Y}=[\textbf{y}^{(1)},\textbf{y}^{(2)},\ldots,\textbf{y}^{(n)}]=[ \textbf{y}_{1}^{\top}; \textbf{y}_{2}^{\top}; \ldots; \textbf{y}_{D}^{\top}] \in\mathbb{R}^{D \times n}$ denotes standardized features; $\boldsymbol{\gamma},\boldsymbol{\beta}\in\mathbb{R}^D$ are used to scale and shift the standardized features;
	$\boldsymbol{\mu} = \frac{1}{n}\textbf{X}\textbf{1}_n\in\mathbb{R}^D$; $\boldsymbol{\Sigma} = \frac{1}{n} (\textbf{X} - \boldsymbol{\mu}\textbf{1}_n^{\top}) (\textbf{X} - \boldsymbol{\mu}\textbf{1}_n^{\top})^{\top}\in\mathbb{R}^{D\times D}$ denotes the covariance matrix of $\textbf{X}$, and then $\boldsymbol{\sigma}=[\sqrt{\boldsymbol{\Sigma}_{1,1}},\ldots,\sqrt{\boldsymbol{\Sigma}_{D,D}}]^{\top}\in\mathbb{R}^D$ represents a vector of the standard deviations corresponding to diagonal elements in $\boldsymbol{\Sigma}$; $\textbf{1}_D\in\mathbb{R}^D$ and $\textbf{1}_n\in\mathbb{R}^n$ are all-one vectors; $\circ$ denotes the element-wise product; $\varepsilon$ is a tiny positive constant to avoid dividing zero; $\textit{diag}(\cdot)$ transforms a vector to a diagonal matrix. Similarly,

	\begin{equation}\label{eq:Y}
		\textbf{Y}=
		\begin{bmatrix}
			\vert & \vert & \cdots & \vert \\
			\textbf{y}^{(1)}& \textbf{y}^{(2)}& \cdots & \textbf{y}^{(n)}  \\
			\vert & \vert & \cdots & \vert
		\end{bmatrix}=
		\brows{\textbf{y}_{1}^{\top}  \\ \textbf{y}_{2}^{\top}  \\ \rowsvdots \\ \textbf{y}_{D}^{\top}  }=
		\begin{bmatrix}
			\textbf{Y}_{1,1} & \textbf{Y}_{1,2} & \cdots  & \textbf{Y}_{1,n}\\
			\textbf{Y}_{2,1} & \textbf{Y}_{2,2} & \cdots  & \textbf{Y}_{2,n}\\
			\vdots & \vdots & \ddots & \vdots \\
			\textbf{Y}_{D,1} & \textbf{Y}_{D,2} & \cdots  & \textbf{Y}_{D,n}\\
		\end{bmatrix}. \nonumber
	\end{equation}

	
	
	{\bf Since the tiny positive constant $\varepsilon$ is only used to avoid dividing zero,
		we ignore this term}
	to simplify further proofs. Therefore,
	\begin{align}
		\label{eq:bn3}
		\textbf{Y} &=  \textit{diag}(\boldsymbol{\sigma})^{-1} ( \textbf{X} - \boldsymbol{\mu}\textbf{1}_n^{\top} ) 	\qquad \text{(standardization phase)}
	\end{align}
	
	For the $d$-th feature dimension of all the $n$ samples, $d = 1, \cdots, D$,
	\begin{align}
		\label{eq:bn4}
		\textbf{y}_d &= \frac{\textbf{x}_d - \boldsymbol{\mu}_d\textbf{1}_n}{\boldsymbol{\sigma}_d} \in\mathbb{R}^{n}
	\end{align}
	
	The transposition of $\textbf{y}_d$, $i.e.$, $\textbf{y}_d^{\top}$ corresponds to the $d$-th row vector of the standardized feature matrix $\textbf{Y}$.
	
	\section{Proofs of Lemmas 1, 2, and 3}\label{appendix:lemma}
	
	In this section, we prove Lemmas 1, 2, and 3 in the main paper. These three lemmas implement loss decomposition, in order to analyze the BN's effects on different components of each term of the loss function in the Taylor series expansion.
	
	The training loss of the $i$-th sample can be represented as a function of the standardized feature $\textbf{y}^{(i)}$. We use the Taylor series expansion to decompose the loss function \emph{w.r.t.} $\textbf{y}^{(i)}$ into terms of multiple orders, as follows.
	\begin{equation}\label{eq:Taylor}
		\textrm{Loss}(\textbf{y}^{(i)};\tilde{\textbf{y}})=\textrm{Loss}(\tilde{\textbf{y}})+ (\textbf{y}^{(i)}-\tilde{\textbf{y}})^{\top}\textbf{g} +\frac{1}{2!}(\textbf{y}^{(i)}-\tilde{\textbf{y}})^{\top}\textbf{H}(\textbf{y}^{(i)}-\tilde{\textbf{y}})+ R_2(\textbf{y}^{(i)}-\tilde{\textbf{y}})
	\end{equation}
	where $\textbf{g} = \left[\textbf{g}_1, \cdots, \textbf{g}_D\right] \in \mathbb{R}^{D}$ and $\textbf{H} = \left(\textbf{H}_{j,k}\right) \in \mathbb{R}^{D \times D}$ denote the gradient and the Hessian matrix of $\textrm{Loss}(\textbf{y}^{(i)})$ at a fixed point $\tilde{\textbf{y}} = [\tilde{\textrm{y}}_1, \tilde{\textrm{y}}_2, \cdots, \tilde{\textrm{y}}_D] \in \mathbb{R}^{D}$; $R_2(\textbf{y}^{(i)}-\tilde{\textbf{y}})$ denotes the sum of high-order terms.
	
	{\bf Loss decomposition.} Let $\textrm{Loss}^{\text{batch}} \in \mathbb{R}$  denote the overall loss of samples in a mini-batch.We consider the BN operation and decompose $\textrm{Loss}^{\text{batch}}$ in the following two cases.
	
	$\textbf{Case 1}$: Let us first consider the simplest case where all samples in a mini-batch share the same analytic formula of the loss function. Therefore, if we do the Taylor series expansion at the same fixed point $\tilde{\textbf{y}}$ for loss function of each $i$-th sample $i = 1, \cdots, n$,
	the overall loss of samples in a mini-batch $\textrm{Loss}^{\text{batch}}=\sum_{i=1}^{n}\textrm{Loss}(\textbf{y}^{(i)};\tilde{\textbf{y}})$ can be re-written as the sum of four compositional terms, as follows.
	
	\begin{equation}\label{eq:Taylor_simple}
		\textrm{Loss}^{\text{batch}}(\textbf{g},\textbf{H}) \xlongequal{\textbf{decomposed into}}
		\textrm{Loss}^{\text{constant}} + \textrm{Loss}^{\text{grad}}(\textbf{g}) +\textrm{Loss}^{\text{Hessian}}(\textbf{H}) + \sum\nolimits_i R_2(\textbf{y}^{(i)}-\tilde{\textbf{y}})
	\end{equation}
	where $\textrm{Loss}^{\text{constant}}=n\textrm{Loss}(\tilde{\textbf{y}})$ is a constant \emph{w.r.t.} input features $\textbf{X}$ of the BN operation; $\textrm{Loss}^{\text{grad}}(\textbf{g})=\sum_{i=1}^{n}(\textbf{y}^{(i)} - \tilde{\textbf{y}})^{\top}\textbf{g}$ and $\textrm{Loss}^{\text{Hessian}}(\textbf{H})=\sum_{i=1}^{n}\frac{1}{2!}(\textbf{y}^{(i)}-\tilde{\textbf{y}})^{\top}\textbf{H}(\textbf{y}^{(i)}-\tilde{\textbf{y}})$ denote the first-order and second-order terms of the loss function in the Taylor expansion, respectively.
	
	\begin{proof}
		\begin{small}
			\begin{equation}
				\begin{aligned}
					\textrm{Loss}^{\text{batch}}(\textbf{g}, \textbf{H})
					&= \sum_{i = 1}^n \textrm{Loss}(\textbf{y}^{(i)};\tilde{\textbf{y}}) \\
					&= \sum_{i = 1}^n \left(\textrm{Loss}(\tilde{\textbf{y}}) + (\textbf{y}^{(i)} - \tilde{\textbf{y}})^{\top} \textbf{g} + \frac{1}{2!} (\textbf{y}^{(i)} -\tilde{\textbf{y}})^{\top} \textbf{H} (\textbf{y}^{(i)} - \tilde{\textbf{y}}) + R_2(\textbf{y}^{(i)}-\tilde{\textbf{y}})\right)\\
					&= n \textrm{Loss}(\tilde{\textbf{y}}) + \sum_{i = 1}^n (\textbf{y}^{(i)} - \tilde{\textbf{y}})^{\top}\textbf{g} + \sum_{i = 1}^n\frac{1}{2!} (\textbf{y}^{(i)} -\tilde{\textbf{y}})^{\top} \textbf{H} (\textbf{y}^{(i)} - \tilde{\textbf{y}}) + \sum_{i = 1}^n R_2(\textbf{y}^{(i)}-\tilde{\textbf{y}}). \nonumber
				\end{aligned}
			\end{equation}
		\end{small}
	\end{proof}
	
	Then,
	\begin{small}
		\begin{align}
			&\textrm{Loss}^{\text{constant}} = n \textrm{Loss}(\tilde{\textbf{y}}) \\
			&\textrm{Loss}^{\text{grad}}(\textbf{g}) = \sum_{i = 1}^n (\textbf{y}^{(i)} - \tilde{\textbf{y}})^{\top} \textbf{g} \\
			&\textrm{Loss}^{\text{Hessian}}(\textbf{H}) = \sum_{i = 1}^n\frac{1}{2} (\textbf{y}^{(i)} -\tilde{\textbf{y}})^{\top} \textbf{H} (\textbf{y}^{(i)} - \tilde{\textbf{y}})
		\end{align}
	\end{small}
	
	
	\textit{Lemma 1} further decomposes $\textrm{Loss}^{\text{Hessian}}(\textbf{H})$ into two terms, as follows.

	\begin{lemma}\label{le:1}
		Let us set $\varepsilon=0$ (the tiny positive constant $\varepsilon$ is only used to avoid dividing zero, so we can ignore $\varepsilon$ in all following paragraphs to simplify the proof). Then,
		{\rm$\textrm{Loss}^{\text{Hessian}}(\textbf{H})$} can be decomposed into two terms, $i.e.$,
		{\rm\begin{equation*}
				\textrm{Loss}^{\text{Hessian}}(\textbf{H})= \textrm{Loss}^{\text{diag}}(\textbf{H}^{\text{diag}}) + \textrm{Loss}^{\text{off}}(\textbf{H}^{\text{off}}),
		\end{equation*}}
		where {\rm\small$\textrm{Loss}^{\text{diag}}(\textbf{H}^{\text{diag}}) = \sum_{i=1}^{n}\frac{1}{2!} (\textbf{y}^{(i)}-\tilde{\textbf{y}})^{\top} \textbf{H}^{\text{diag}}(\textbf{y}^{(i)}-\tilde{\textbf{y}})$}, and {\rm\small$\textrm{Loss}^{\text{off}} (\textbf{H}^{\text{off}})= \sum_{i=1}^{n}\frac{1}{2!} (\textbf{y}^{(i)}-\tilde{\textbf{y}})^{\top} \textbf{H}^{\text{off}}(\textbf{y}^{(i)}-\tilde{\textbf{y}})$; $\textbf{H}^{\text{diag}}$} and {\rm\small$\textbf{H}^{\text{off}} = \textbf{H} - \textbf{H}^{\text{diag}}$} denote the matrix only containing diagonal elements and the matrix only containing off-diagonal elements in {\rm$\textbf{H}$}, respectively.
	\end{lemma}
	
	\begin{proof}
		\begin{small}
			\begin{equation}
				\begin{aligned}
					\textrm{Loss}^{\text{Hessian}}(\textbf{H})
					&= \sum_{i = 1}^n\frac{1}{2} (\textbf{y}^{(i)} -\tilde{\textbf{y}})^{\top} \textbf{H} (\textbf{y}^{(i)} - \tilde{\textbf{y}})\\
					&= \sum_{i = 1}^n\frac{1}{2} (\textbf{y}^{(i)} -\tilde{\textbf{y}})^{\top} \left(\textbf{H}^{\text{diag}} + \textbf{H}^{\text{off}}\right)(\textbf{y}^{(i)} - \tilde{\textbf{y}})\\
					&= \sum_{i = 1}^n\frac{1}{2}\left( (\textbf{y}^{(i)} -\tilde{\textbf{y}})^{\top} \textbf{H}^{\text{diag}}(\textbf{y}^{(i)} - \tilde{\textbf{y}}) + (\textbf{y}^{(i)} -\tilde{\textbf{y}})^{\top} \textbf{H}^{\text{off}} (\textbf{y}^{(i)} - \tilde{\textbf{y}})\right)\\
					&= \sum_{i = 1}^n\frac{1}{2} (\textbf{y}^{(i)} -\tilde{\textbf{y}})^{\top} \textbf{H}^{\text{diag}} (\textbf{y}^{(i)} - \tilde{\textbf{y}}) + \sum_{i = 1}^n\frac{1}{2} (\textbf{y}^{(i)} -\tilde{\textbf{y}})^{\top} \textbf{H}^{\text{off}} (\textbf{y}^{(i)} - \tilde{\textbf{y}})\\
					&= \textrm{Loss}^{\text{diag}}(\textbf{H}^{\text{diag}}) + \textrm{Loss}^{\text{off}}(\textbf{H}^{\text{off}})
				\end{aligned}
			\end{equation}
		\end{small}
	\end{proof}

	\textit{Lemma 2} further decomposes $\textrm{Loss}^{\text{off}}(\textbf{H}^{\text{off}})$ into two terms, as follows.
	
	\begin{lemma}\label{le:2}
		Let {\rm\small$\textbf{x}_d=[\textbf{X}_{d,1},\textbf{X}_{d,2},\ldots,\textbf{X}_{d,n}]^{\top}\in\mathbb{R}^n$} denote the $d$-th feature dimension of all the $n$ samples in a mini-batch. Then, {\rm\small$\textrm{Loss}^{\text{off}}(\textbf{H}^{\text{off}})$} can be decomposed into the loss term depending on {\rm$\textbf{x}_d$} (i.e.,
		{\rm\small$L_d(\textbf{H}_{d,:}^{\text{off}})= \textbf{H}^{\text{off}}_{d,:}(\textbf{Y} -\tilde{\textbf{y}} \textbf{1}_n ^{\top})  \textbf{y}_d$}) and the loss term independent with {\rm$\textbf{x}_d$}, thereby {\rm\small$\frac{\partial \textrm{Loss}^{\text{off}}(\textbf{H}^{\text{off}})}{\partial \textbf{x}_d} = \frac{\partial L_d(\textbf{H}_{d,:}^{\text{off}})}{\partial \textbf{x}_d}$}, where {\rm\small$\textbf{H}_{d,:}^{\text{off}}=[\textbf{H}_{d,1}^{\text{off}},\textbf{H}_{d,2}^{\text{off}},\ldots,\textbf{H}_{d,D}^{\text{off}}]\in\mathbb{R}^D$} denotes the $d$-th row of {\rm\small$\textbf{H}^{\text{off}}$}, {\rm\small$\textbf{y}_d=[\textbf{Y}_{d,1},\textbf{Y}_{d,2},\ldots,\textbf{Y}_{d,n}]^{\top}\in\mathbb{R}^n$}.
		Furthermore, {\rm\small$L_d(\textbf{H}_{d,:}^{\text{off}})$} can be decomposed as
		{\rm\begin{equation*}
				L_d(\textbf{H}_{d,:}^{\text{off}})=L_d^{\text{linear}}(\textbf{H}_{d,:}^{\text{off}})+L_d^{\text{non}}(\textbf{H}_{d,:}^{\text{off}}),
		\end{equation*}}
		where {\rm\small$L_d^{\text{linear}}(\textbf{H}_{d,:}^{\text{off}})=\textbf{H}_{d,:}^{\text{off}} \textbf{Y}^{\text{linear}}\textbf{y}_d$} and {\rm\small$L_d^{\text{non}}(\textbf{H}_{d,:}^{\text{off}})= \textbf{H}_{d,:}^{\text{off}} (\textbf{Y}^{\text{non}}-\tilde{\textbf{y}} \textbf{1}_n ^{\top})\textbf{y}_d$}. {\rm\small$\textbf{Y}^{\text{linear}}= [ \textbf{o}_d^{\top} \textbf{y}_1,  \textbf{o}_d^{\top} \textbf{y}_2, \cdots,  \textbf{o}_d^{\top} \textbf{y}_D ]^{\top} \textbf{o}_d^{\top}$} and {\rm\small$\textbf{Y}^{\text{non}}=\textbf{Y}-\textbf{Y}^{\text{linear}}$}, where {\rm\small $\textbf{o}_d=\frac{\textbf{y}_d}{\|\textbf{y}_d\|}$} denotes the unit vector in the direction of {\rm $\textbf{y}_d$}. Then, {\rm\small$L_d^{\text{linear}}(\textbf{H}_{d,:}^{\text{off}})= \lVert \textbf{y}_d \rVert  \cdot  \textbf{H}_{d,:}^{\text{off}} [ \textbf{o}_d^{\top} \textbf{y}_1,  \textbf{o}_d^{\top} \textbf{y}_2, \cdots,  \textbf{o}_d^{\top} \textbf{y}_D ]^{\top} $}.
		Therefore, {\rm\small$\frac{\partial \textrm{Loss}^{\text{off}}(\textbf{H}^{\text{off}})}{\partial \textbf{x}_d} = \frac{\partial L_d(\textbf{H}_{d,:}^{\text{off}})}{\partial \textbf{x}_d} = \frac{\partial L_d^{\text{linear}}(\textbf{H}_{d,:}^{\text{off}})}{\partial \textbf{x}_d} + \frac{\partial L_d^{\text{non}}(\textbf{H}_{d,:}^{\text{off}})}{\partial \textbf{x}_d}$}.
	\end{lemma}
	
	\begin{proof}
		We prove \textit{Lemma 2} by two steps as follows.
		
		\textit{\textbf{Step 1:}} Firstly, we prove that {\rm\small$\textrm{Loss}^{\text{off}}(\textbf{H}^{\text{off}})$} can be decomposed into the loss term depending on {\rm$\textbf{x}_d$} (i.e.,
		{\rm\small$L_d(\textbf{H}_{d,:}^{\text{off}})= \textbf{H}^{\text{off}}_{d,:}(\textbf{Y} -\tilde{\textbf{y}} \textbf{1}_n ^{\top})  \textbf{y}_d$}) and the loss term independent with {\rm$\textbf{x}_d$}, thereby {\rm\small$\frac{\partial \textrm{Loss}^{\text{off}}(\textbf{H}^{\text{off}})}{\partial \textbf{x}_d} = \frac{\partial L_d(\textbf{H}_{d,:}^{\text{off}})}{\partial \textbf{x}_d}$}.
		
		Based on \textit{Lemma}~\ref{le:1},
		\begin{small}
			\begin{align}
				\qquad\textrm{Loss}^{\text{off}}(\textbf{H}^{\text{off}}) &= \sum_{i = 1}^n\frac{1}{2} (\textbf{y}^{(i)} -\tilde{\textbf{y}})^{\top} \textbf{H}^{\text{off}} (\textbf{y}^{(i)} - \tilde{\textbf{y}}) \nonumber\\
				&= \sum_{i = 1}^n \left(\frac{1}{2} \sum_{j = 1}^D\sum_{k=1}^D \textbf{H}^{\text{off}}_{j, k} (\textbf{y}^{(i)}_j- \tilde{\textrm{y}}_j) (\textbf{y}^{(i)}_k - \tilde{\textrm{y}}_k)\right) \nonumber\\
				& = \sum_{i = 1}^n \frac{1}{2} \left(\sum_{j = 1}^D\textbf{H}^{\text{off}}_{d, j} (\textbf{y}^{(i)}_d- \tilde{\textrm{y}}_d)(\textbf{y}^{(i)}_j- \tilde{\textrm{y}}_j) + \sum_{k = 1}^D\textbf{H}^{\text{off}}_{k, d} (\textbf{y}^{(i)}_k- \tilde{\textrm{y}}_k)(\textbf{y}^{(i)}_d- \tilde{\textrm{y}}_d) \right) \nonumber\\
				& \quad + \sum_{i = 1}^n \left(\frac{1}{2} \sum_{j = 1, j\neq d}^D\sum_{k=1, k\neq d}^D \textbf{H}^{\text{off}}_{j, k} (\textbf{y}^{(i)}_j- \tilde{\textrm{y}}_j) (\textbf{y}^{(i)}_k - \tilde{\textrm{y}}_k)\right) \nonumber \\
				& \qquad \slash \slash {\rm \tiny Note \ that \ the \ term\ \textbf{H}^{\text{off}}_{d, d} (\textbf{y}^{(i)}_d- \tilde{\textrm{y}}_d) (\textbf{y}^{(i)}_d - \tilde{\textrm{y}}_d) is computed twice for summation,} \nonumber\\
				& \qquad {\rm \tiny \ \ \ but\ since \ that \ \textbf{H}^{\text{off}}_{d,d}=0 according to \textit{Lemma}~\ref{le:1}, this term is zero.} \nonumber\\
				&= \sum_{i = 1}^n \frac{1}{2} \left(\sum_{j = 1}^D\textbf{H}^{\text{off}}_{d, j} (\textbf{y}^{(i)}_d- \tilde{\textrm{y}}_d)(\textbf{y}^{(i)}_j- \tilde{\textrm{y}}_j) + \sum_{k = 1}^D\textbf{H}^{\text{off}}_{k, d} (\textbf{y}^{(i)}_k- \tilde{\textrm{y}}_k)(\textbf{y}^{(i)}_d- \tilde{\textrm{y}}_d) \right) + \textrm{L}_d^{\text{tmp-independent1}} \nonumber \\
				& \qquad \slash \slash {\rm \tiny Let\ \textrm{L}_d^{\text{tmp-independent1}} denote \sum_{i = 1}^n \left(\frac{1}{2} \sum_{j = 1, j\neq d}^D\sum_{k=1, k\neq d}^D \textbf{H}^{\text{off}}_{j, k} (\textbf{y}^{(i)}_j- \tilde{\textrm{y}}_j) (\textbf{y}^{(i)}_k - \tilde{\textrm{y}}_k)\right) for simplicity.} \nonumber\\
				&=\sum_{i = 1}^n \sum_{j = 1}^D\textbf{H}^{\text{off}}_{d, j} (\textbf{y}^{(i)}_d- \tilde{\textrm{y}}_d)(\textbf{y}^{(i)}_j- \tilde{\textrm{y}}_j) + \textrm{L}_d^{\text{tmp-independent1}} \quad  {\rm  \slash \slash (\textbf{H}^{\text{off}}\ is \ symmetric.)}  \nonumber \\
				&= \sum_{j = 1}^D\textbf{H}^{\text{off}}_{d, j}\left(\sum_{i = 1}^n (\textbf{y}^{(i)}_d- \tilde{\textrm{y}}_d)(\textbf{y}^{(i)}_j- \tilde{\textrm{y}}_j) \right) + \textrm{L}_d^{\text{tmp-independent1}} \nonumber\\
				&= \sum_{j = 1}^D\textbf{H}^{\text{off}}_{d, j} (\textbf{y}_j-\tilde{\textrm{y}}_j \textbf{1}_n)^{\top} (\textbf{y}_d-\tilde{\textrm{y}}_d \textbf{1}_n) + \textrm{L}_d^{\text{tmp-independent1}} \nonumber\\
				&= \sum_{j = 1}^D\textbf{H}^{\text{off}}_{d, j} (\textbf{y}_j-\tilde{\textrm{y}}_j \textbf{1}_n)^{\top} \textbf{y}_d-
				\sum_{j = 1}^D\textbf{H}^{\text{off}}_{d, j} (\textbf{y}_j-\tilde{\textrm{y}}_j \textbf{1}_n)^{\top} \tilde{\textrm{y}}_d \textbf{1}_n + \textrm{L}_d^{\text{tmp-independent1}} \nonumber\\
				&= \textbf{H}^{\text{off}}_{d, :}\left(\textbf{Y}-\tilde{\textbf{y}} \textbf{1}_n ^{\top}\right)  \textbf{y}_d
				- \tilde{\textrm{y}}_d \sum_{j = 1}^D\textbf{H}^{\text{off}}_{d, j} (\textbf{y}_j^{\top} \textbf{1}_n-\tilde{\textrm{y}}_j \textbf{1}_n^{\top} \textbf{1}_n)  + \textrm{L}_d^{\text{tmp-independent1}} \nonumber\\
				&= \textbf{H}^{\text{off}}_{d, :}\left(\textbf{Y}-\tilde{\textbf{y}} \textbf{1}_n ^{\top}\right)  \textbf{y}_d
				- \tilde{\textrm{y}}_d \sum_{j = 1, j \neq d}^D\textbf{H}^{\text{off}}_{d, j} (\textbf{y}_j^{\top} \textbf{1}_n-\tilde{\textrm{y}}_j \textbf{1}_n^{\top} \textbf{1}_n)+ \textrm{L}_d^{\text{tmp-independent1}} \qquad   \slash \slash \textbf{H}^{\text{off}}_{d,d}=0 \nonumber \\
				\label{eq:ldderive}
				&= \textbf{H}^{\text{off}}_{d, :}\left(\textbf{Y}-\tilde{\textbf{y}} \textbf{1}_n ^{\top}\right)  \textbf{y}_d
				- \textrm{L}_d^{\text{tmp-independent2}} + \textrm{L}_d^{\text{tmp-independent1}} \nonumber\\
				& \qquad \slash \slash {\rm \tiny Let\ \textrm{L}_d^{\text{tmp-independent2}} denote \tilde{\textrm{y}}_d \sum_{j = 1, j \neq d}^D\textbf{H}^{\text{off}}_{d, j} (\textbf{y}_j^{\top} \textbf{1}_n-\tilde{\textrm{y}}_j \textbf{1}_n^{\top} \textbf{1}_n) for simplicity.} \nonumber\\
				&= \textrm{L}_d(\textbf{H}_{d,:}^{\text{off}}) + \textrm{L}_d^{\text{independent}}. \nonumber\\
				& \qquad \slash \slash {\rm \tiny Let\ \textrm{L}_d(\textbf{H}_{d,:}^{\text{off}}) denote \textbf{H}^{\text{off}}_{d, :}\left(\textbf{Y}-\tilde{\textbf{y}} \textbf{1}_n ^{\top}\right), and let  \textrm{L}_d^{\text{independent}} denote (-\textrm{L}_d^{\text{tmp-independent2}} + \textrm{L}_d^{\text{tmp-independent1}})} \nonumber
			\end{align}
		\end{small}

		Therefore,
		\begin{small}
			\begin{align*}
				\frac{\partial \textrm{Loss}^{\text{off}}(\textbf{H}^{\text{off}})}{\partial \textbf{x}_d}
				&= \frac{\partial \left(\textrm{L}_d(\textbf{H}_{d,:}^{\text{off}}) + \textrm{L}_d^{\text{independent}}(\textbf{H}_{d,:}^{\text{off}}) \right)}{\partial \textbf{x}_d}
				= \frac{\partial \textrm{L}_d(\textbf{H}_{d,:}^{\text{off}})}{\partial \textbf{x}_d} +\frac{\partial \textrm{L}_d^{\text{independent}}(\textbf{H}_{d,:}^{\text{off}}) }{\partial \textbf{x}_d}\\
				&= \frac{\partial L_d(\textbf{H}_{d,:}^{\text{off}})}{\partial \textbf{x}_d} + \textbf{0}
				= \frac{\partial L_d(\textbf{H}_{d,:}^{\text{off}})}{\partial \textbf{x}_d}.
			\end{align*}
		\end{small}
		
		\textit{\textbf{Step 2:}} $L_d(\textbf{H}_{d,:}^{\text{off}})$ can be futher decomposed into two terms, as follows.
		
		\begin{small}
			\begin{equation}
				\begin{aligned}
					L_d(\textbf{H}_{d,:}^{\text{off}})
					&= \textbf{H}^{\text{off}}_{d, :}\left(\textbf{Y}-\tilde{\textbf{y}} \textbf{1}_n ^{\top}\right)\textbf{y}_d\\
					&= \textbf{H}^{\text{off}}_{d, :}\left( \textbf{Y}^{\text{linear}} + \textbf{Y}^{\text{non}} -\tilde{\textbf{y}} \textbf{1}_n ^{\top} \right)\textbf{y}_d
					\qquad \slash \slash \textbf{Y} = \textbf{Y}^{\text{linear}}+\textbf{Y}^{\text{non}}\\
					&= \textbf{H}^{\text{off}}_{d, :}\textbf{Y}^{\text{linear}}\textbf{y}_d
					+ \textbf{H}^{\text{off}}_{d, :}
					\left(\textbf{Y}^{\text{non}} -\tilde{\textbf{y}} \textbf{1}_n ^{\top} \right)\textbf{y}_d \\
					&= L_d^{\text{linear}}(\textbf{H}_{d,:}^{\text{off}}) + L_d^{\text{non}}(\textbf{H}_{d,:}^{\text{off}}).
				\end{aligned}
			\end{equation}
		\end{small}
		
		where {\rm  $\textbf{Y}^{\text{linear}}= \left[ \textbf{o}_d^{\top} \textbf{y}_1,  \textbf{o}_d^{\top} \textbf{y}_2, \cdots,  \textbf{o}_d^{\top} \textbf{y}_D \right]^{\top} \textbf{o}_d^{\top} \in \mathbb{R}^{D \times n}$}, $\textbf{o}_d=\frac{\textbf{y}_d}{\|\textbf{y}_d\|}$ is the unit vector in the direction of {\rm $\textbf{y}_d$}, and  $\textbf{Y}^{\text{non}} = \left[\textbf{y}_{1}- \| \textbf{y}_{1} \| \cos(\textbf{y}_1, \textbf{y}_d)\textbf{o}_d , \cdots,\textbf{y}_{D} - \|\textbf{y}_{D} \| \cos(\textbf{y}_D, \textbf{y}_d) \textbf{o}_d  \right]^{\top}$.
		
		In particular,
		
		\begin{equation*}
			\begin{aligned}
				L_d^{\text{linear}}(\textbf{H}_{d,:}^{\text{off}})
				&= \textbf{H}_{d,:}^{\text{off}} \textbf{Y}^{\text{linear}}\textbf{y}_d
				= \textbf{H}_{d,:}^{\text{off}} \left[ \textbf{o}_d^{\top} \textbf{y}_1,  \textbf{o}_d^{\top} \textbf{y}_2, \cdots,  \textbf{o}_d^{\top} \textbf{y}_D \right]^{\top} \textbf{o}_d^{\top} \textbf{y}_d\\
				&= \textbf{H}_{d,:}^{\text{off}} \left[ \textbf{o}_d^{\top} \textbf{y}_1,  \textbf{o}_d^{\top} \textbf{y}_2, \cdots,  \textbf{o}_d^{\top} \textbf{y}_D \right]^{\top} (\frac{\textbf{y}_d^{\top}}{\|\textbf{y}_d\|} \textbf{y}_d)\\
				&= \textbf{H}_{d,:}^{\text{off}} \left[ \textbf{o}_d^{\top} \textbf{y}_1,  \textbf{o}_d^{\top} \textbf{y}_2, \cdots,  \textbf{o}_d^{\top} \textbf{y}_D \right]^{\top} (\frac{\|\textbf{y}_d\|^2}{\|\textbf{y}_d\|}) \\
				&=  \lVert \textbf{y}_d \rVert  \cdot  \textbf{H}_{d,:}^{\text{off}} \left[ \textbf{o}_d^{\top} \textbf{y}_1,  \textbf{o}_d^{\top} \textbf{y}_2, \cdots,  \textbf{o}_d^{\top} \textbf{y}_D \right]^{\top}
			\end{aligned}
		\end{equation*}
		
		
		In this way, for all $d$, the gradient of $L_d(\textbf{H}_{d,:}^{\text{off}})$ \emph{w.r.t.} $\textbf{x}_d$ can be decomposed into two terms,
		
		\begin{equation}
			\begin{aligned}
				\frac{\partial \textrm{Loss}^{\text{off}}(\textbf{H}^{\text{off}})}{\partial \textbf{x}_d}
				&=\frac{\partial \left(L_d^{\text{linear}}(\textbf{H}_{d,:}^{\text{off}})+L_d^{\text{non}}(\textbf{H}_{d,:}^{\text{off}})\right)}{\partial \textbf{x}_d}\\
				&= \frac{\partial L_d^{\text{linear}}(\textbf{H}_{d,:}^{\text{off}})}{\partial \textbf{x}_d} + \frac{\partial L_d^{\text{non}}(\textbf{H}_{d,:}^{\text{off}})}{\partial \textbf{x}_d},
			\end{aligned}
		\end{equation}
		
		
		where {$\frac{\partial L_d^{\text{linear}}(\textbf{H}_{d,:}^{\text{off}})}{\partial \textbf{x}_d} $} reflects the interaction utility yielded by feature components in all dimensions that are linear-correlated with $\textbf{y}_d$, \emph{i.e.}, $[ \textbf{o}_d^{\top} \textbf{y}_1,  \textbf{o}_d^{\top} \textbf{y}_2, \cdots,  \textbf{o}_d^{\top} \textbf{y}_D ]^{\top}$, subject to $\textbf{o}_d=\frac{\textbf{y}_d}{\|\textbf{y}_d\|}$. {$\frac{\partial L_d^{\text{non}}(\textbf{H}_{d,:}^{\text{off}})}{\partial \textbf{x}_d}$} reflects the interaction utility between $\textbf{y}_d$ and feature components that are not linearly-correlated with $\textbf{y}_d$, \emph{i.e.}, removing linearly-correlated components, $\forall j, \textbf{y}_{j}- \|\textbf{y}_j\| \cos(\textbf{y}_j, \textbf{y}_d) \textbf{o}_d$.
		

	\end{proof}

	$\textbf{Case 2}$: Then let us consider a more general case where different samples in a mini-batch have different analytic formulas, \emph{e.g.}, losses are computed with different ground-truth labels.  Therefore, different samples may have different gradients $\textbf{g}^{(i)}$ and Hessian matrices $\textbf{H}^{(i)}$ at the fixed point $\tilde{\textbf{y}}$ for Taylor series expansions of losses. In this case, we decompose the loss function for each $i$-th sample as \textit{Lemma 3}.
	
	\begin{lemma}\label{le:3}
		The loss function for each $i$-th sample, $1\le i \le n$, can be written as {\rm $\textrm{Loss}(\textbf{y}^{(i)};\tilde{\textbf{y}}) = \textrm{Loss}(\tilde{\textbf{y}}) + \textrm{Loss}^{\text{grad}}(\bbar{\textbf{g}}) +\textrm{Loss}^{\text{diag}}(\bbar{\textbf{H}}^{\text{diag}}) + \textrm{Loss}^{\text{off}}(\bbar{\textbf{H}}^{\text{off}}) + \textrm{Loss}^{\text{grad}}(\textbf{g}^{\prime (i)}) +\textrm{Loss}^{\text{diag}}(\textbf{H}^{\prime (i) \text{diag}}) + \textrm{Loss}^{\text{off}}(\textbf{H}^{\prime (i) \text{off}}) + R_2(\textbf{y}^{(i)}-\tilde{\textbf{y}})$}, where {\rm$\bbar{\textbf{g}} = \frac{1}{n}\sum_{i=1}^{n}\textbf{g}^{(i)}$} and {\rm$\bbar{\textbf{H}} = \frac{1}{n}\sum_{i=1}^{n}\textbf{H}^{(i)}$} denote the average gradient and the average Hessian matrix, respectively, which are shared by all samples in the mini-batch. {\rm$\textbf{g}^{\prime (i)} = \textbf{g}^{(i)} - \bbar{\textbf{g}}$} and {\rm$\textbf{H}^{\prime (i)} = \textbf{H}^{(i)} - \bbar{\textbf{H}}$} denote the distinctive gradient and the distinctive Hessian matrix of the $i$-th sample, respectively.
	\end{lemma}
	
	\begin{proof}
		The loss function for each $i$-th sample, $1\le i \le n$, can be written as
		
		\begin{small}
			\begin{equation} \label{eq: gene taylor}
				\begin{aligned}
					\textrm{Loss}(\textbf{y}^{(i)};\tilde{\textbf{y}})=\textrm{Loss}(\tilde{\textbf{y}})+ (\textbf{y}^{(i)}-\tilde{\textbf{y}})^{\top}\textbf{g}^{(i)} +\frac{1}{2!}(\textbf{y}^{(i)}-\tilde{\textbf{y}})^{\top}\textbf{H}^{(i)}(\textbf{y}^{(i)}-\tilde{\textbf{y}})+ R_2(\textbf{y}^{(i)}-\tilde{\textbf{y}}).
				\end{aligned}
			\end{equation}
		\end{small}
		
		Let  {\rm$\bbar{\textbf{g}} = \frac{1}{n}\sum_{i=1}^{n}\textbf{g}^{(i)}$} and {\rm$\bbar{\textbf{H}} = \frac{1}{n}\sum_{i=1}^{n}\textbf{H}^{(i)}$} denote the average gradient and the average Hessian matrix, respectively; and let {\rm$\textbf{g}^{\prime (i)} = \textbf{g}^{(i)} - \bbar{\textbf{g}}$} and {\rm$\textbf{H}^{\prime (i)} = \textbf{H}^{(i)} - \bbar{\textbf{H}}$}.
		Then Equation~(\ref{eq: gene taylor}) can be re-written as
		
		\begin{small}
			\begin{equation} \label{eq: gene decompose taylor}
				\begin{aligned}
					\textrm{Loss}(\textbf{y}^{(i)};\tilde{\textbf{y}})
					&=\textrm{Loss}(\tilde{\textbf{y}})+ (\textbf{y}^{(i)}-\tilde{\textbf{y}})^{\top}\left( \bbar{\textbf{g}} + \textbf{g}^{\prime (i)} \right) +\frac{1}{2!}(\textbf{y}^{(i)}-\tilde{\textbf{y}})^{\top}\left( \bbar{\textbf{H}} + \textbf{H}^{\prime (i)} \right)(\textbf{y}^{(i)}-\tilde{\textbf{y}})+ R_2(\textbf{y}^{(i)}-\tilde{\textbf{y}})'\\
					&= \left( \textrm{Loss}(\tilde{\textbf{y}}) + \textrm{Loss}^{\text{grad}}(\bbar{\textbf{g}}) +\textrm{Loss}^{\text{diag}}(\bbar{\textbf{H}}^{\text{diag}}) + \textrm{Loss}^{\text{off}}(\bbar{\textbf{H}}^{\text{off}}) \right)\\
					& \quad + \left( \textrm{Loss}^{\text{grad}}(\textbf{g}^{\prime (i)}) +\textrm{Loss}^{\text{diag}}(\textbf{H}^{\prime (i) \text{diag}}) + \textrm{Loss}^{\text{off}}(\textbf{H}^{\prime (i) \text{off}}) + R_2(\textbf{y}^{(i)}-\tilde{\textbf{y}}) \right).
				\end{aligned}
			\end{equation}
		\end{small}
		
	\end{proof}

	
	\section{Proofs of Theorems 1, 2, 3, and Corollary 1}\label{appendix:theorem}
	
	In this section, we prove Theorems 1, 2, 3, and Corollary 1 in the main paper, which show the effects of the BN operation on each loss term in Equation~(\ref{eq:Taylor_simple}) during the back-propagation in \textbf{Case1}.
	
	{\bf Back-propagation through the BN layer.}  During the back-propagation, for all $d \in \{1, \cdots, D\}$, the gradient of $\textrm{Loss}^{\text{batch}}(\textbf{g},\textbf{H})$ $w.r.t.$ the $d$-th feature dimension of all the $n$ samples before the BN operation, $i.e.$, $\frac{\partial \textrm{Loss}^{\text{batch}}(\textbf{g},\textbf{H})}{\partial \textbf{x}_{d}}$, and the gradient of $\textrm{Loss}^{\text{batch}}(\textbf{g},\textbf{H})$ $w.r.t.$ the $d$-th feature dimension of all the $n$ samples after the standardization phase of the BN operation, $i.e.$, $\frac{\partial \textrm{Loss}^{\text{batch}}(\textbf{g},\textbf{H})}{\partial \textbf{y}_{d}}$, are connected by chain rule as follows.
	
	\begin{equation}\label{BN chain rule}
		\frac{\partial \textrm{Loss}^{\text{batch}}(\textbf{g},\textbf{H})}{\partial \textbf{x}_{d}}
		=\frac{\partial \textbf{y}_d^{\top}}{\partial \textbf{x}_{d}}\frac{\partial \textrm{Loss}^{\text{batch}}(\textbf{g},\textbf{H}) }{\partial \textbf{y}_{d}}
		=\textbf{J}_d\frac{\partial \textrm{Loss}^{\text{batch}}(\textbf{g},\textbf{H}) }{\partial \textbf{y}_{d}},
	\end{equation}
	
	where we use \(\textbf{J}_d = \frac{\partial \textbf{y}_d^{\top}}{\partial \textbf{x}_{d}}\) to  denote the Jacobian matrix of $\textbf{y}_d$ $w.r.t.$ $\textbf{x}_d$.
	
	

	According to Equation~(\ref{eq:bn4}), $\textbf{y}_d = \frac{\textbf{x}_d - \boldsymbol{\mu}_d\textbf{1}_n}{\boldsymbol{\sigma}_d}$, Therefore,

	\begin{small}\begin{equation}
			\begin{aligned}
				\textbf{J}_d
				&= \frac{\partial \textbf{y}_d^{\top}}{\partial \textbf{x}_{d}}\\
				&= \left( \frac{\partial \boldsymbol{\mu}_d}{\partial \textbf{x}_{d}}\frac{\partial \textbf{y}_d^{\top}}{\partial \boldsymbol{\mu}_d}  + \frac{\partial \boldsymbol{\sigma}_d^2}{\partial \textbf{x}_d}\frac{\partial \textbf{y}_d^{\top}}{\partial \boldsymbol{\sigma}_d^2}  + \frac{\partial \textbf{x}_d^{\top}}{\partial \textbf{x}_d} \frac{\partial \textbf{y}_d^{\top}}{\partial \textbf{x}_d} \right)  \\
				&= \left(\frac{\partial \left(\frac{1}{n} \textbf{x}_d^{\top}\textbf{1}_n\right)}{\partial \textbf{x}_d}\frac{\partial \textbf{y}_d^{\top}}{\partial \boldsymbol{\mu}_d}  + \frac{\partial \left(\frac{1}{n}(\textbf{x}_d - \boldsymbol{\mu}_d\textbf{1}_n)^{\top} \left(\textbf{x}_d - \boldsymbol{\mu}_d\textbf{1}_n\right)\right)}{\partial \textbf{x}_d}\frac{\partial \textbf{y}_d^{\top}}{\partial \boldsymbol{\sigma}_d^2}  + \frac{\partial \textbf{x}_d^{\top}}{\partial \textbf{x}_d} \frac{\partial \textbf{y}_d^{\top}}{\partial \textbf{x}_d} \right)  \\
				&= \left(\frac{1}{n}\textbf{1}_n\frac{\partial \textbf{y}_d^{\top}}{\partial \boldsymbol{\mu}_d}  + \frac{2}{n}\left(\textbf{x}_d - \boldsymbol{\mu}_d\textbf{1}_n\right)\frac{\partial \textbf{y}_d^{\top}}{\partial \boldsymbol{\sigma}_d^2}  + I \frac{\partial \textbf{y}_d^{\top}}{\partial \textbf{x}_d}\right)  \\
				&= \left(\frac{1}{n}\textbf{1}_n\frac{\partial \left( \frac{\textbf{x}_d - \boldsymbol{\mu}_d\textbf{1}_n}{\boldsymbol{\sigma}_d}\right)^{\top}}{\partial \boldsymbol{\mu}_d}  + \frac{2}{n}\left(\textbf{x}_d - \boldsymbol{\mu}_d\textbf{1}_n\right)\frac{\partial \left( \frac{\textbf{x}_d - \boldsymbol{\mu}_d\textbf{1}_n}{\boldsymbol{\sigma}_d}\right)^{\top}}{\partial \boldsymbol{\sigma}_d^2} + \frac{\partial \left( \frac{\textbf{x}_d - \boldsymbol{\mu}_d\textbf{1}_n}{\boldsymbol{\sigma}_d}\right)^{\top}}{\partial \textbf{x}_d} \right)  \\
				&= \left(-\frac{1}{n}\textbf{1}_n\frac{1}{\boldsymbol{\sigma}_d}\textbf{1}_n^{\top}  + \frac{2}{n} \left(\textbf{x}_{1} - \boldsymbol{\mu}_d\textbf{1}_n\right)\left(\textbf{x}_d - \boldsymbol{\mu}_d\textbf{1}_n\right)^{\top} \left(- \frac{1}{2} {\boldsymbol{\sigma}_d}^{-3} \right) + \frac{1}{\boldsymbol{\sigma}_d}I \right) \\
				\label{Jacobian}
				&= \frac{1}{\boldsymbol{\sigma}_d}\left(I-\frac{1}{n}\textbf{1}_n\textbf{1}_n^{\top} - \frac{1}{n\boldsymbol{\sigma}_d^2}\left(\textbf{x}_{d} - \boldsymbol{\mu}_d\textbf{1}_n\right)\left(\textbf{x}_{d} - \boldsymbol{\mu}_d\textbf{1}_n\right)^{\top} \right).
			\end{aligned}
		\end{equation}
	\end{small}

	Based on \textit{Lemma 1} and \textit{Lemma 2}, we obtain the decomposition of $\frac{\partial \textrm{Loss}^{\text{batch}}(\textbf{g}, \textbf{H})}{\partial \textbf{x}_d}$ and $\frac{\partial \textrm{Loss}^{\text{batch}}(\textbf{g}, \textbf{H})}{\partial \textbf{y}_d}$ as follows.
	
	\begin{scriptsize}
		\begin{align}
			\frac{\partial \textrm{Loss}^{\text{batch}}(\textbf{g}, \textbf{H})}{\partial \textbf{x}_d}
			&= \frac{\partial }{\partial \textbf{x}_d} \left(\textrm{Loss}^{\text{constant}} + \textrm{Loss}^{\text{grad}}(\textbf{g}) + \textrm{Loss}^{\text{diag}}(\textbf{H}^{\text{diag}}) + \textrm{Loss}^{\text{off}}(\textbf{H}^{\text{off}}) + \sum_{i = 1}^n R_2(\textbf{y}^{(i)}-\tilde{\textbf{y}})\right)\\
			&= \frac{\partial \textrm{Loss}^{\text{constant}}}{\partial \textbf{x}_d} + \frac{\partial \textrm{Loss}^{\text{grad}}(\textbf{g})}{\partial \textbf{x}_d} + \frac{\partial \textrm{Loss}^{\text{diag}}(\textbf{H}^{\text{diag}})}{\partial \textbf{x}_d} + \frac{\partial \textrm{Loss}^{\text{off}}(\textbf{H}^{\text{off}})}{\partial \textbf{x}_d}+ \frac{\partial (\sum_{i = 1}^n R_2(\textbf{y}^{(i)}-\tilde{\textbf{y}}))}{\partial \textbf{x}_d}. \\
			&= \frac{\partial \textrm{Loss}^{\text{constant}}}{\partial \textbf{x}_d} + \frac{\partial \textrm{Loss}^{\text{grad}}(\textbf{g})}{\partial \textbf{x}_d} + \frac{\partial \textrm{Loss}^{\text{diag}}(\textbf{H}^{\text{diag}})}{\partial \textbf{x}_d} + \frac{\partial L_d^{\text{linear}}(\textbf{H}_{d,:}^{\text{off}})}{\partial \textbf{x}_d} +
			\frac{\partial L_d^{\text{non}}(\textbf{H}_{d,:}^{\text{off}})}{\partial \textbf{x}_d} \\
			& \qquad + \frac{\partial (\sum_{i = 1}^n R_2(\textbf{y}^{(i)}-\tilde{\textbf{y}}))}{\partial \textbf{x}_d}. \nonumber
		\end{align}
	\end{scriptsize}
	
	\begin{scriptsize}
		\begin{align}
			\frac{\partial \textrm{Loss}^{\text{batch}}(\textbf{g}, \textbf{H})}{\partial \textbf{y}_d}
			&= \frac{\partial }{\partial \textbf{y}_d} \left(\textrm{Loss}^{\text{constant}} + \textrm{Loss}^{\text{grad}}(\textbf{g}) + \textrm{Loss}^{\text{diag}}(\textbf{H}^{\text{diag}}) + \textrm{Loss}^{\text{off}}(\textbf{H}^{\text{off}}) + \sum_{i = 1}^n R_2(\textbf{y}^{(i)}-\tilde{\textbf{y}})\right)\\
			&= \frac{\partial \textrm{Loss}^{\text{constant}}}{\partial \textbf{y}_d} + \frac{\partial \textrm{Loss}^{\text{grad}}(\textbf{g})}{\partial \textbf{y}_d} + \frac{\partial \textrm{Loss}^{\text{diag}}(\textbf{H}^{\text{diag}})}{\partial \textbf{y}_d} + \frac{\partial \textrm{Loss}^{\text{off}}(\textbf{H}^{\text{off}})}{\partial \textbf{y}_d}+ \frac{\partial (\sum_{i = 1}^n R_2(\textbf{y}^{(i)}-\tilde{\textbf{y}}))}{\partial \textbf{y}_d}. \\
			&= \frac{\partial \textrm{Loss}^{\text{constant}}}{\partial \textbf{y}_d} + \frac{\partial \textrm{Loss}^{\text{grad}}(\textbf{g})}{\partial \textbf{y}_d} + \frac{\partial \textrm{Loss}^{\text{diag}}(\textbf{H}^{\text{diag}})}{\partial \textbf{y}_d} + \frac{\partial L_d^{\text{linear}}(\textbf{H}_{d,:}^{\text{off}})}{\partial \textbf{y}_d} +
			\frac{\partial L_d^{\text{non}}(\textbf{H}_{d,:}^{\text{off}})}{\partial \textbf{y}_d} \\
			& \qquad + \frac{\partial (\sum_{i = 1}^n R_2(\textbf{y}^{(i)}-\tilde{\textbf{y}}))}{\partial \textbf{y}_d}. \nonumber
		\end{align}
	\end{scriptsize}

	Similarly by chain rule, we have
	\begin{small}
		\begin{align}
			\label{eq:chain1}
			\frac{\partial \textrm{Loss}^{\text{grad}}(\textbf{g})}{\partial \textbf{x}_{d}} &=\frac{\partial \textbf{y}_d^{\top}}{\partial \textbf{x}_{d}}\frac{\partial \textrm{Loss}^{\text{grad}}(\textbf{g})}{\partial \textbf{y}_{d}} =\textbf{J}_d\frac{\partial \textrm{Loss}^{\text{grad}}(\textbf{g})}{\partial \textbf{y}_{d}}. \\
			\label{eq:chain2}
			\frac{\partial \textrm{Loss}^{\text{diag}}(\textbf{H}^{\text{diag}})}{\partial \textbf{x}_{d}} &=\frac{\partial \textbf{y}_d^{\top}}{\partial \textbf{x}_{d}}\frac{\partial \textrm{Loss}^{\text{diag}}(\textbf{H}^{\text{diag}}) }{\partial \textbf{y}_{d}}=\textbf{J}_d\frac{\partial \textrm{Loss}^{\text{diag}}(\textbf{H}^{\text{diag}}) }{\partial \textbf{y}_{d}}. \\
			\label{eq:chain3}
			\frac{\partial \textrm{Loss}^{\text{off}}(\textbf{H}^{\text{off}})}{\partial \textbf{x}_{d}} &=\frac{\partial \textbf{y}_d^{\top}}{\partial \textbf{x}_{d}}\frac{\partial \textrm{Loss}^{\text{off}}(\textbf{H}^{\text{off}}) }{\partial \textbf{y}_{d}}=\textbf{J}_d\frac{\partial \textrm{Loss}^{\text{off}}(\textbf{H}^{\text{off}}) }{\partial \textbf{y}_{d}}. \\
			\label{eq:chain4}
			\frac{\partial L_d^{\text{linear}}(\textbf{H}_{d,:}^{\text{off}})}{\partial \textbf{x}_d} &=\frac{\partial \textbf{y}_d^{\top}}{\partial \textbf{x}_{d}}\frac{\partial L_d^{\text{linear}}(\textbf{H}_{d,:}^{\text{off}})}{\partial \textbf{y}_d}=\textbf{J}_d \frac{\partial L_d^{\text{linear}}(\textbf{H}_{d,:}^{\text{off}})}{\partial \textbf{y}_d}\\
			\label{eq:chain5}
			\frac{\partial L_d^{\text{non}}(\textbf{H}_{d,:}^{\text{off}})}{\partial \textbf{x}_d}&=\frac{\partial \textbf{y}_d^{\top}}{\partial \textbf{x}_{d}}\frac{\partial L_d^{\text{non}}(\textbf{H}_{d,:}^{\text{off}})}{\partial \textbf{y}_d}=\textbf{J}_d \frac{\partial L_d^{\text{non}}(\textbf{H}_{d,:}^{\text{off}})}{\partial \textbf{y}_d}.
		\end{align}
	\end{small}

	\begin{table}[h]
		\caption{Several notations are summarized for reference.}
		\label{notation}
		\centering
		\begin{tabular}{c|c}
			\hline
			$n$                                      &   the number of samples in a mini-batch \\
			$\textbf{x}^{(i)} \in \mathbb{R}^{D}$ & the $D$-dimensional feature of the $i$-th sample \textbf{before} the BN operation \\
			$\textbf{y}^{(i)} \in \mathbb{R}^{D}$ & the $D$-dimensional feature of the $i$-th sample \textbf{after}  the standardization phase \\
			$\textbf{x}_d\in \mathbb{R}^{n}$  & the $d$-th feature dimension of all the $n$ samples \textbf{before} the BN operation\\
			$\textbf{y}_d\in \mathbb{R}^{n}$  & the $d$-th feature dimension of all the $n$ samples \textbf{after} the standardization phase\\
			$\textbf{J}_d\in \mathbb{R}^{n \times n}$              & the Jacobian matrix of $\textbf{y}_d$  $w.r.t.$  $\textbf{x}_d$, $i.e.$ $\textbf{J}_d=\frac{\partial \textbf{y}_d^{\top}}{\partial \textbf{x}_{d}}$\\
			$\tilde{\textbf{y}} \in \mathbb{R}^{D}$          & the point to conduct Taylor expansion of the loss function\\
			$\textbf{o}_d \in \mathbb{R}^{D}$             & the unit vector in the direction of {\rm $\textbf{y}_d$} $i.e.$ $\textbf{o}_d=\frac{\textbf{y}_d}{\|\textbf{y}_{d}\|}$ \\
			\hline
		\end{tabular}
	\end{table}
	
	
	\subsection{Proof of Theorem 1: the effect of $\textrm{Loss}^{\text{constant}}$ and $\textrm{Loss}^{\text{grad}}(\textbf{g})$}
	
	In this section, we prove \textit{Theorem 1} in Section 3 of the main paper, which shows the effect of the zeroth order term and the first order term of $\textrm{Loss}^{\text{batch}}(\textbf{g}, \textbf{H})$.
	
	\begin{theorem}\label{th:1}
		{\rm$\textrm{Loss}^{\text{grad}}(\textbf{g})$} and {\rm$\textrm{Loss}^{\text{constant}}$} have no gradients on input features {\rm$\textbf{X}$} of the BN operation, i.e., {\rm$\frac{\partial \textrm{Loss}^{\text{grad}}(\textbf{g})}{\partial \textbf{X}} = \textbf{0}$} and {\rm$\frac{\partial \textrm{Loss}^{\text{constant}}}{\partial \textbf{X}} = \textbf{0}$}.
	\end{theorem}
	
	\begin{proof}
		(1) $\textrm{Loss}^{\text{constant}}=n\textrm{Loss}(\tilde{\textbf{y}})$ is a constant, thus $\frac{\partial \textrm{Loss}^{\text{constant}}}{\partial \textbf{x}_{d}} = \textbf{0}$, $ d = 1, \cdots, D$.\\
		(2) In consideration of the connection that $\frac{\partial \textrm{Loss}^{\text{grad}}(\textbf{g})}{\partial \textbf{x}_{d}} =\textbf{J}_d\frac{\partial \textrm{Loss}^{\text{grad}}(\textbf{g})}{\partial \textbf{y}_{d}}$, which is shown in Equation~(\ref{eq:chain1}), we first derive $\frac{\partial \textrm{Loss}^{\text{grad}}(\textbf{g})}{\partial \textbf{y}_d}$ as follows.
		
		\begin{small}
			\begin{equation}
				\begin{aligned}
					\frac{\partial \textrm{Loss}^{\text{grad}}(\textbf{g})}{\partial \textbf{y}_d}
					&=\frac{\partial \left(\sum_{i = 1}^n (\textbf{y}^{(i)} - \tilde{\textbf{y}})^{\top} \textbf{g}\right) }{\partial \textbf{y}_d}= \frac{\partial \left(\textbf{1}_n^{\top}\left(\textbf{Y}^{\top} - \textbf{1}_n\tilde{\textbf{y}}^{\top}\right) \textbf{g}\right)}{\partial \textbf{y}_d}\\
					&= \frac{\partial \left(\textbf{1}_n^{\top}\sum_{j = 1}^D \textbf{g}_j (\textbf{y}_j - \tilde{\textrm{y}}_j \textbf{1}_n )\right) }{\partial \textbf{y}_d}= \frac{\partial \left(\textbf{g}_d \textbf{1}_n^{\top}(\textbf{y}_d - \tilde{\textrm{y}}_d \textbf{1}_n)\right) }{\partial \textbf{y}_d}= \textbf{g}_d \textbf{1}_n \in \mathbb{R}^{n}, 
				\end{aligned}
			\end{equation}
		\end{small}
		
		where $\textbf{g}_d \in \mathbb{R}$ is the $d$-th element of the gradient $\textbf{g} \in \mathbb{R}^{D}$. Then,
		
		\begin{small}
			\begin{align*}
				\frac{\partial \textrm{Loss}^{\text{grad}}(\textbf{g})}{\partial \textbf{x}_d}
				&= \textbf{J}_d \frac{\partial \textrm{Loss}^{\text{grad}}(\textbf{g})}{\partial \textbf{y}_d}\\
				&= \frac{1}{\boldsymbol{\sigma}_d}\left(I-\frac{1}{n}\textbf{1}_n\textbf{1}_n^{\top} - \frac{1}{n\boldsymbol{\sigma}_d^2}\left(\textbf{x}_{d} - \boldsymbol{\mu}_d\textbf{1}_n\right)\left(\textbf{x}_{d} - \boldsymbol{\mu}_d\textbf{1}_n\right)^{\top} \right)\frac{\partial \textrm{Loss}^{\text{grad}}(\textbf{g})}{\partial \textbf{y}_d}\\
				&= \frac{1}{\boldsymbol{\sigma}_d}\left(I-\frac{1}{n}\textbf{1}_n\textbf{1}_n^{\top} - \frac{1}{n\boldsymbol{\sigma}_d^2}\left(\textbf{x}_{d} - \boldsymbol{\mu}_d\textbf{1}_n\right)\left(\textbf{x}_{d} - \boldsymbol{\mu}_d\textbf{1}_n\right)^{\top} \right) \left(\textbf{g}_d \textbf{1}_n\right)\\
				&= \frac{\textbf{g}_d}{\boldsymbol{\sigma}_d}\left(\textbf{1}_n-\frac{1}{n}\textbf{1}_n\textbf{1}_n^{\top}\textbf{1}_n - \frac{1}{n\boldsymbol{\sigma}_d^2}\left(\textbf{x}_{d} - \boldsymbol{\mu}_d\textbf{1}_n\right)\left(\textbf{x}_{d} - \boldsymbol{\mu}_d\textbf{1}_n\right)^{\top}\textbf{1}_n \right)\\
				&= \frac{\textbf{g}_d}{\boldsymbol{\sigma}_d}\left(\textbf{1}_n-\textbf{1}_n - \frac{1}{n\boldsymbol{\sigma}_d^2}\left(\textbf{x}_{d} - \boldsymbol{\mu}_d\textbf{1}_n\right)\left(\textbf{x}_{d}^{\top}\textbf{1}_n - \boldsymbol{\mu}_d\textbf{1}_n^{\top}\textbf{1}_n\right) \right)\\
				&= \frac{\textbf{g}_d}{\boldsymbol{\sigma}_d}\left(\textbf{0} - \frac{1}{n\boldsymbol{\sigma}_d^2}\left(\textbf{x}_{1} - \boldsymbol{\mu}_d\textbf{1}_n\right)\left(n\boldsymbol{\mu}_d - n\boldsymbol{\mu}_d\right) \right)\\
				&= \textbf{0} \in \mathbb{R}^{n} 
			\end{align*}
		\end{small}
		
	\end{proof}

	\subsection{Proof of Theorem 2: the effect of  $\textrm{Loss}^{\text{diag}}(\textbf{H}^{\text{diag}})$.}
	
	In this section, we prove \textit{Theorem 2} in Section 3 of the main paper, which shows the effects of diagonal elements in the Hessian matrix at the fixed point $\tilde{\textbf{y}}$.

	\begin{theorem}\label{th:2}
		{\rm$\textrm{Loss}^{\text{diag}}(\textbf{H}^{\text{diag}}) $} has no gradients on input features {\rm$\textbf{X}$} of the BN operation, i.e., {\rm$\frac{\partial\textrm{Loss}^{\text{diag}}(\textbf{H}^{\text{diag}}) }{\partial \textbf{X}} = 0$}.
	\end{theorem}
	
	\begin{proof}
		Since that $\frac{\partial \textrm{Loss}^{\text{diag}}(\textbf{H}^{\text{diag}})}{\partial \textbf{X}} = \left[
		\frac{\partial \textrm{Loss}^{\text{diag}}(\textbf{H}^{\text{diag}})}{\partial \textbf{x}_1},
		\frac{\partial \textrm{Loss}^{\text{diag}}(\textbf{H}^{\text{diag}})}{\partial \textbf{x}_2},
		\ldots;
		\frac{\partial \textrm{Loss}^{\text{diag}}(\textbf{H}^{\text{diag}})}{\partial \textbf{x}_D}
		\right]^{\top}$,
		it is equivalent to prove that $\frac{\partial \textrm{Loss}^{\text{diag}}(\textbf{H}^{\text{diag}})}{\partial \textbf{x}_{d}} = \textbf{0}$, ${\forall} d \in \{1, \cdots, D\}$.
		According to Equation~(\ref{eq:chain2}), $\frac{\partial \textrm{Loss}^{\text{diag}}(\textbf{H}^{\text{diag}})}{\partial \textbf{x}_{d}} =\textbf{J}_d\frac{\partial \textrm{Loss}^{\text{diag}}(\textbf{H}^{\text{diag}}) }{\partial \textbf{y}_{d}}$, thus we first derive $\frac{\partial \textrm{Loss}^{\text{diag}}(\textbf{H}^{\text{diag}})}{\partial \textbf{y}_d}$ as follows.
		
		\begin{small}
			\begin{align*}
				\frac{\partial \textrm{Loss}^{\text{diag}}(\textbf{H}^{\text{diag}})}{\partial \textbf{y}_d}
				&= \frac{\partial\left(\sum_{i = 1}^n\frac{1}{2} (\textbf{y}^{(i)} -\tilde{\textbf{y}})^{\top} \textbf{H}^{\text{diag}} (\textbf{y}^{(i)} - \tilde{\textbf{y}})\right)}{\partial \textbf{y}_d}\\
				&= \frac{1}{2}\frac{\partial\left(\sum_{i = 1}^n \sum_{j = 1}^{D} \textbf{H}_{j, j}^{\text{diag}} (\textbf{y}^{(i)}_j - \tilde{\textrm{y}}_j)^2\right)}{\partial \textbf{y}_d}
				\qquad \slash \slash \tilde{\textbf{y}}_j \in \mathbb{R}\ \textrm{is the }j\textrm{-th element of point}\ \tilde{\textbf{y}} \in \mathbb{R}^{D}\\
				&= \frac{1}{2}\frac{\partial\left(\sum_{j = 1}^{D} \textbf{H}_{j, j}^{\text{diag}}\sum_{i = 1}^n ((\textbf{y}^{(i)}_j - \tilde{\textrm{y}}_j)^2\right)}{\partial \textbf{y}_d}\\
				&= \frac{1}{2}\frac{\partial\left(\sum_{j = 1}^{D} \textbf{H}_{j, j}^{\text{diag}} \left(\textbf{y}_j - \tilde{\textrm{y}}_j \textbf{1}_n\right)^{\top}\left(\textbf{y}_j - \tilde{\textrm{y}}_j \textbf{1}_n\right) \right)}{\partial \textbf{y}_d}\\
				&= \frac{1}{2}\frac{\partial\left(\textbf{H}_{d, d}^{\text{diag}} \left(\textbf{y}_d - \tilde{\textrm{y}}_d \textbf{1}_n\right)^{\top}\left(\textbf{y}_d - \tilde{\textrm{y}}_d \textbf{1}_n\right) \right)}{\partial \textbf{y}_d}
				\qquad \slash \slash \tilde{\textbf{y}}_d \in \mathbb{R}\ \textrm{is the }d\textrm{-th element of point}\ \tilde{\textbf{y}} \in \mathbb{R}^{D}\\
				&= \textbf{H}_{d, d}^{\text{diag}} \left(\textbf{y}_d - \tilde{\textrm{y}}_d \textbf{1}_n\right),  \nonumber
			\end{align*}
		\end{small}
		
		Then,
		
		\begin{small}
			\begin{align*}
				\frac{\partial \textrm{Loss}^{\text{diag}}(\textbf{H}^{\text{diag}})}{\partial \textbf{x}_d}
				&=\textbf{J}_d\frac{\partial \textrm{Loss}^{\text{diag}}(\textbf{H}^{\text{diag}}) }{\partial \textbf{y}_{d}}
				\qquad \slash \slash {\rm \ Equation~(\ref{eq:chain2})}\\
				&= \frac{1}{\boldsymbol{\sigma}_d}\left(I-\frac{1}{n}\textbf{1}_n\textbf{1}_n^{\top} - \frac{1}{n\boldsymbol{\sigma}_d^2}\left(\textbf{x}_{d} - \boldsymbol{\mu}_d\textbf{1}_n\right)\left(\textbf{x}_{d} - \boldsymbol{\mu}_d\textbf{1}_n\right)^{\top} \right)\frac{\partial \textrm{Loss}^{\text{diag}}(\textbf{H}^{\text{diag}})}{\partial \textbf{y}_d}\\
				&= \frac{1}{\boldsymbol{\sigma}_d}\left(I-\frac{1}{n}\textbf{1}_n\textbf{1}_n^{\top} - \frac{1}{n\boldsymbol{\sigma}_d^2}\left(\textbf{x}_{d} - \boldsymbol{\mu}_d\textbf{1}_n\right)\left(\textbf{x}_{d} - \boldsymbol{\mu}_d\textbf{1}_n\right)^{\top} \right)\textbf{H}_{d, d}^{\text{diag}} \left(\textbf{y}_d - \tilde{\textrm{y}}_d \textbf{1}_n\right)\\
				&= \frac{1}{\boldsymbol{\sigma}_d}\left(I-\frac{1}{n}\textbf{1}_n\textbf{1}_n^{\top} - \frac{1}{n\boldsymbol{\sigma}_d^2}\left(\textbf{x}_{d} - \boldsymbol{\mu}_d\textbf{1}_n\right)\left(\textbf{x}_{d} - \boldsymbol{\mu}_d\textbf{1}_n\right)^{\top} \right)\textbf{H}_{d, d}^{\text{diag}} \frac{\textbf{x}_d - \boldsymbol{\mu}_d\textbf{1}_n}{\boldsymbol{\sigma}_d}\\
				& \quad - \frac{1}{\boldsymbol{\sigma}_d}\left(I-\frac{1}{n}\textbf{1}_n\textbf{1}_n^{\top} - \frac{1}{n\boldsymbol{\sigma}_d^2}\left(\textbf{x}_{d} - \boldsymbol{\mu}_d\textbf{1}_n\right)\left(\textbf{x}_{d} - \boldsymbol{\mu}_d\textbf{1}_n\right)^{\top} \right)\textbf{H}_{d, d}^{\text{diag}} \tilde{\textrm{y}}_d \textbf{1}_n\\
				&= \frac{\textbf{H}_{d, d}^{\text{diag}}}{\boldsymbol{\sigma}_d^2}\left(\left(\textbf{x}_d - \boldsymbol{\mu}_d\textbf{1}_n\right)-\frac{1}{n}\textbf{1}_n\textbf{1}_n^{\top}\left(\textbf{x}_d - \boldsymbol{\mu}_d\textbf{1}_n\right) - \frac{1}{n\boldsymbol{\sigma}_d^2}\left(\textbf{x}_{d} - \boldsymbol{\mu}_d\textbf{1}_n\right)\left(\textbf{x}_{d} - \boldsymbol{\mu}_d\textbf{1}_n\right)^{\top}\left(\textbf{x}_d - \boldsymbol{\mu}_d\textbf{1}_n\right)\right)\\
				& \quad - \frac{\textbf{H}_{d, d}^{\text{diag}} \tilde{\textrm{y}}_d}{\boldsymbol{\sigma}_d}\left(\textbf{1}_n-\frac{1}{n}\textbf{1}_n\textbf{1}_n^{\top}\textbf{1}_n - \frac{1}{n\boldsymbol{\sigma}_d^2}\left(\textbf{x}_{d} - \boldsymbol{\mu}_d\textbf{1}_n\right)\left(\textbf{x}_{d} - \boldsymbol{\mu}_d\textbf{1}_n\right)^{\top}\textbf{1}_n \right) \\
				&= \frac{\textbf{H}_{d, d}^{\text{diag}}}{\boldsymbol{\sigma}_d^2}\left(\left(\textbf{x}_d - \boldsymbol{\mu}_d\textbf{1}_n\right)-\frac{1}{n}\textbf{1}_n\left(\textbf{1}_n^{\top}\textbf{x}_d - \boldsymbol{\mu}_d\textbf{1}_n^{\top}\textbf{1}_n\right) - \frac{1}{n\boldsymbol{\sigma}_d^2}\left(\textbf{x}_{d} - \boldsymbol{\mu}_d\textbf{1}_n\right)n \boldsymbol{\sigma}_d^2\right)\\
				& \quad - \frac{\textbf{H}_{d, d}^{\text{diag}} \tilde{\textrm{y}}_d}{\boldsymbol{\sigma}_d}\left(\textbf{0} - \frac{1}{n\boldsymbol{\sigma}_d^2}\left(\textbf{x}_{d} - \boldsymbol{\mu}_d\textbf{1}_n\right)\left(\textbf{x}_{d}^{\top}\textbf{1}_n - \boldsymbol{\mu}_d\textbf{1}_n^{\top}\textbf{1}_n\right) \right) \\
				&= \frac{\textbf{H}_{d, d}^{\text{diag}}}{\boldsymbol{\sigma}_d^2}\left(\left(\textbf{x}_d - \boldsymbol{\mu}_d\textbf{1}_n\right)-\frac{1}{n}\textbf{1}_n\left(n\boldsymbol{\mu}_d - n\boldsymbol{\mu}_d\right) - \frac{n\boldsymbol{\sigma}_d^2}{n\boldsymbol{\sigma}_d^2 }\left(\textbf{x}_{d} - \boldsymbol{\mu}_d\textbf{1}_n\right)\right)\\
				& \quad - \frac{\textbf{H}_{d, d}^{\text{diag}} \tilde{\textrm{y}}_d}{\boldsymbol{\sigma}_d}\left(\textbf{0} - \frac{1}{n\boldsymbol{\sigma}_d^2}\left(\textbf{x}_{d} - \boldsymbol{\mu}_d\textbf{1}_n\right)\left(n\boldsymbol{\mu}_d - n\boldsymbol{\mu}_d\right) \right) \\
				&= \frac{\textbf{H}_{d, d}^{\text{diag}}}{\boldsymbol{\sigma}_d^2}\left(\left(\textbf{x}_d - \boldsymbol{\mu}_d\textbf{1}_n\right)-\textbf{0} - \left(\textbf{x}_{d} - \boldsymbol{\mu}_d\textbf{1}_n\right)\right)
				- \frac{\textbf{H}_{d, d}^{\text{diag}} \tilde{\textrm{y}}_d}{\boldsymbol{\sigma}_d}\left(\textbf{0} - \frac{1}{n\boldsymbol{\sigma}_d^2}\left(\textbf{x}_{d} - \boldsymbol{\mu}_d\textbf{1}_n\right)\textbf{0} \right) \\
				&= \textbf{0} \nonumber
			\end{align*}
		\end{small}

	\end{proof}
	
	\subsection{Proof of Theorem 3: the effect of  $\textrm{Loss}^{\text{off}}(\textbf{H}^{\text{off}})$.}
	
	In this section, we prove \textit{Theorem 3} in Section 3 of the main paper, which shows the effect of off-diagonal elements in the Hessian matrix at the fixed point $\tilde{\textbf{y}}$.
	
	\begin{theorem}\label{th:3}
		{\rm\small$\forall d$, $L_d^{\text{linear}}(\textbf{H}_{d,:}^{\text{off}})$} has no gradients on the $d$-th feature dimension over all $n$ samples {\rm$\textbf{x}_d\in\mathbb{R}^n$}, i.e.,
		{\rm\small\begin{equation*}
				\forall d,\ \frac{\partial L_d^{\text{linear}}(\textbf{H}_{d,:}^{\text{off}})}{\partial \textbf{x}_d} = \textbf{0},
		\end{equation*}}
		thus {\rm\small$\frac{\partial ^2 L_d^{\text{linear}}(\textbf{H}_{d,:}^{\text{off}})}{\partial \textbf{x}_d \partial \textbf{H}_{d,:}^{\text{off}}} = \frac{\partial ^2 }{\partial \textbf{x}_d \partial \textbf{H}_{d,:}^{\text{off}}} (A  [ \textbf{o}_d^{\top} \textbf{y}_1,  \textbf{o}_d^{\top} \textbf{y}_2, \cdots,  \textbf{o}_d^{\top} \textbf{y}_D ]^{\top}   ) = \textbf{0}$}, s.t. {\rm\small$A=\lVert \textbf{y}_d \rVert  \cdot  \textbf{H}_{d,:}^{\text{off}} $.} On the other hand, {\rm\small$\frac{\partial^2 L_d ^{\text{non}}(\textbf{H}_{d,:}^{\text{off}})}{\partial \textbf{x}_{d} \partial \textbf{H}_{d, :}^{\text{off}}} =  \frac{1}{\boldsymbol{\sigma}_d} \cdot [(\textbf{y}_{1}- \| \textbf{y}_{1} \| \cos(\textbf{y}_1, \textbf{y}_d)\textbf{o}_d ), \ldots,(\textbf{y}_{D} - \|\textbf{y}_{D} \| \cos(\textbf{y}_D, \textbf{y}_d) \textbf{o}_d  )]  \ne \textbf{0}$}.
	\end{theorem}
	
	\begin{proof}
		(1) For $L_d^{\text{linear}}(\textbf{H}_{d,:}^{\text{off}})=\textbf{H}_{d,:}^{\text{off}} \textbf{Y}^{\text{linear}}\textbf{y}_d$,  we prove that $\frac{\partial L_d^{\text{linear}}(\textbf{H}_{d,:}^{\text{off}})}{\partial \textbf{x}_d} = \textbf{0}$  as follows.

		\begin{small}
			\begin{align*}\label{eq:LDlinear}
				\frac{\partial L_d^{\text{linear}}(\textbf{H}_{d,:}^{\text{off}})}{\partial \textbf{x}_{d}}
				&= \textbf{J}_d \frac{\partial L_d^{\text{linear}}(\textbf{H}_{d,:}^{\text{off}})}{\partial \textbf{y}_{d}} \quad \qquad {\rm \slash \slash Equation~(\ref{eq:chain4})}\\
				&= \textbf{J}_d \frac{\partial\left( \textbf{H}^{\text{off}}_{d, :}\textbf{Y}^{\text{linear}}\textbf{y}_d\right)}{\partial \textbf{y}_d}  \\
				&= \textbf{J}_d \left(\textbf{H}^{\text{off}}_{d, :}\textbf{Y}^{\text{linear}}\right)^{\top}\\
				&= \textbf{J}_d \left(\textbf{H}^{\text{off}}_{d, :} \left[ \textbf{o}_d^{\top} \textbf{y}_1,  \textbf{o}_d^{\top} \textbf{y}_2, \cdots,  \textbf{o}_d^{\top} \textbf{y}_D \right]^{\top} \textbf{o}_d^{\top}\right)^{\top} \\
				&=  \textbf{J}_d \textbf{o}_d \left[ \textbf{o}_d^{\top} \textbf{y}_1,  \textbf{o}_d^{\top} \textbf{y}_2, \cdots,  \textbf{o}_d^{\top} \textbf{y}_D \right] \left(\textbf{H}^{\text{off}}_{d, :}\right)^{\top}\\
				&=\textbf{J}_d \textbf{o}_d \boldsymbol\lambda_{d}^{\top} \left(\textbf{H}^{\text{off}}_{d, :}\right)^{\top}
				\qquad \slash \slash  {\rm \small Let \ \boldsymbol\lambda_d \in \mathbb{R}^{D} denote \left[ \textbf{o}_d^{\top} \textbf{y}_1,  \textbf{o}_d^{\top} \textbf{y}_2, \cdots,  \textbf{o}_d^{\top} \textbf{y}_D \right]^{\top} for simplicity.}\\
				&= \textbf{J}_d \frac{\textbf{y}_d}{\|\textbf{y}_d\|} \boldsymbol\lambda_{d}^{\top} \left(\textbf{H}^{\text{off}}_{d, :}\right)^{\top}
				\qquad \slash \slash  {\rm \small  \textbf{o}_d=\frac{\textbf{y}_d}{\|\textbf{y}_d\|}}\\
				&=  \textbf{J}_d \left(\frac{\textbf{x}_d - \mu_d \textbf{1}_n}{\sigma_d \|\textbf{y}_d\|}\right) \boldsymbol\lambda_{d}^{\top} \left(\textbf{H}^{\text{off}}_{d, :}\right)^{\top}
				\\
				&= \frac{1}{\sigma_d}\left(I-\frac{1}{n}\textbf{1}_n\textbf{1}_n^{\top} - \frac{1}{n\sigma_d^2}\left(\textbf{x}_{d} - \mu_d\textbf{1}_n\right)\left(\textbf{x}_{d} - \mu_d\textbf{1}_n\right)^{\top} \right)\left(\frac{\textbf{x}_d - \mu_d \textbf{1}_n}{\sigma_d \|\textbf{y}_d\|}\right) \boldsymbol\lambda_{d}^{\top} \left(\textbf{H}^{\text{off}}_{d, :}\right)^{\top}
				\qquad {\rm \slash \slash Equation~(\ref{Jacobian})}\\
				&= \frac{\boldsymbol\lambda_{d}^{\top} \left(\textbf{H}^{\text{off}}_{d, :}\right)^{\top}}{\sigma_d^2 \|\textbf{y}_d\|}\left(\left(\textbf{x}_{d} - \mu_d\textbf{1}_n\right)-\frac{1}{n}\textbf{1}_n\textbf{1}_n^{\top}\left(\textbf{x}_{d} - \mu_d\textbf{1}_n\right) - \frac{\left(\textbf{x}_{d} - \mu_d\textbf{1}_n\right)\left(\textbf{x}_{d} - \mu_d\textbf{1}_n\right)^{\top}\left(\textbf{x}_{d} - \mu_d\textbf{1}_n\right)}{n\sigma_d^2} \right)  \\
				&= \frac{\boldsymbol\lambda_{d}^{\top} \left(\textbf{H}^{\text{off}}_{d, :}\right)^{\top}}{\sigma_d^2 \|\textbf{y}_d\|}\left(\left(\textbf{x}_{d} - \mu_d\textbf{1}_n\right)-\frac{1}{n}\textbf{1}_n\left(\textbf{1}_n^{\top}\textbf{x}_{d} - \mu_d\textbf{1}_n^{\top}\textbf{1}_n\right) - \frac{1}{n\sigma_d^2}\left(\textbf{x}_{d} - \mu_d\textbf{1}_n\right)n\sigma_d^2 \right)  \\
				&= \frac{\boldsymbol\lambda_{d}^{\top} \left(\textbf{H}^{\text{off}}_{d, :}\right)^{\top}}{\sigma_d^2 \|\textbf{y}_d\|}\left(\left(\textbf{x}_{d} - \mu_d\textbf{1}_n\right)-\frac{1}{n}\textbf{1}_n\left(n\mu_{d} - n\mu_d\right) - \frac{n\sigma_d^2}{n\sigma_d^2}\left(\textbf{x}_{d} - \mu_d\textbf{1}_n\right) \right)  \\
				&= \frac{\boldsymbol\lambda_{d}^{\top} \left(\textbf{H}^{\text{off}}_{d, :}\right)^{\top}}{\sigma_d^2 \|\textbf{y}_d\| }\left(\left(\textbf{x}_{d} - \mu_d\textbf{1}_n\right)-\textbf{0} - \left(\textbf{x}_{d} - \mu_d\textbf{1}_n\right) \right)  \\
				& = \textbf{0} \nonumber
			\end{align*}
		\end{small}
		
		(2) For $L_d^{\text{non}}(\textbf{H}_{d,:}^{\text{off}})$, we first prove the mathematical expression of $\frac{\partial L_d ^{\text{non}}(\textbf{H}_{d,:}^{\text{off}})}{\partial \textbf{x}_{d}}$ as follows.
		
		\begin{small}
			\begin{equation}
				\begin{aligned}
					\frac{\partial L_d^{\text{non}}(\textbf{H}_{d,:}^{\text{off}})}{\partial \textbf{x}_{d}}
					&= \textbf{J}_d \frac{\partial\left( \textbf{H}^{\text{off}}_{d, :}(\textbf{Y}^{\text{non}}-\tilde{\textbf{y}} \textbf{1}_n ^{\top})\textbf{y}_d\right)}{\partial \textbf{y}_d}  \quad {\rm \slash \slash Equation~(\ref{eq:chain5})}\\
					&= \textbf{J}_d \left(\textbf{H}^{\text{off}}_{d, :}\textbf{Y}^{\text{non}}\right)^{\top}
					- \textbf{J}_d \left(\textbf{H}^{\text{off}}_{d, :}(\tilde{\textbf{y}} \textbf{1}_n ^{\top})\right)^{\top}
					\\
					&= \textbf{J}_d (\textbf{Y}^{\text{non}})^{\top}\left(\textbf{H}^{\text{off}}_{d, :}\right)^{\top}
					- \textbf{J}_d (\tilde{\textbf{y}} \textbf{1}_n ^{\top})^{\top}\left(\textbf{H}^{\text{off}}_{d, :}\right)^{\top}
					\\
					&= \frac{1}{\boldsymbol{\sigma}_d}\left(I-\frac{1}{n}\textbf{1}_n\textbf{1}_n^{\top} -\frac{1}{n\boldsymbol{\sigma}_d^2}\left(\textbf{x}_{d} - \boldsymbol{\mu}_d\textbf{1}_n\right)\left(\textbf{x}_{d} - \boldsymbol{\mu}_d\textbf{1}_n\right)^{\top}\right)(\textbf{Y}^{\text{non}})^{\top}\left(\textbf{H}^{\text{off}}_{d, :}\right)^{\top} \\
					&\quad - \frac{1}{\boldsymbol{\sigma}_d}\left(I-\frac{1}{n}\textbf{1}_n\textbf{1}_n^{\top} -\frac{1}{n\boldsymbol{\sigma}_d^2}\left(\textbf{x}_{d} - \boldsymbol{\mu}_d\textbf{1}_n\right)\left(\textbf{x}_{d} - \boldsymbol{\mu}_d\textbf{1}_n\right)^{\top}\right)(\tilde{\textbf{y}} \textbf{1}_n ^{\top})^{\top}\left(\textbf{H}^{\text{off}}_{d, :}\right)^{\top}\\
					&= \frac{1}{\boldsymbol{\sigma}_d}\left(I-\frac{1}{n}\textbf{1}_n\textbf{1}_n^{\top} - \frac{1}{n}\textbf{y}_{d} \textbf{y}_{d}^{\top} \right)(\textbf{Y}^{\text{non}})^{\top}\left(\textbf{H}^{\text{off}}_{d, :}\right)^{\top} - \frac{1}{\boldsymbol{\sigma}_d}\left(I-\frac{1}{n}\textbf{1}_n\textbf{1}_n^{\top} - \frac{1}{n}\textbf{y}_{d} \textbf{y}_{d}^{\top} \right)
					(\textbf{1}_n \tilde{\textbf{y}}^{\top})
					\left(\textbf{H}^{\text{off}}_{d, :}\right)^{\top}\\
					&= \frac{1}{\boldsymbol{\sigma}_d}\left(I-\frac{1}{n}\textbf{1}_n\textbf{1}_n^{\top} - \frac{1}{n}\textbf{y}_{d} \textbf{y}_{d}^{\top} \right)(\textbf{Y}^{\text{non}})^{\top}\left(\textbf{H}^{\text{off}}_{d, :}\right)^{\top} - \frac{1}{\boldsymbol{\sigma}_d}\left((\textbf{1}_n-\frac{1}{n}\textbf{1}_n\textbf{1}_n^{\top}\textbf{1}_n) - \frac{1}{n}\textbf{y}_{d} (\textbf{y}_{d}^{\top}\textbf{1}_n) \right)
					\tilde{\textbf{y}}^{\top}
					\left(\textbf{H}^{\text{off}}_{d, :}\right)^{\top}\\
					&= \frac{1}{\boldsymbol{\sigma}_d}\left(I-\frac{1}{n}\textbf{1}_n\textbf{1}_n^{\top} - \frac{1}{n}\textbf{y}_{d} \textbf{y}_{d}^{\top} \right)(\textbf{Y}^{\text{non}})^{\top}\left(\textbf{H}^{\text{off}}_{d, :}\right)^{\top} - \textbf{0}\\
					&= \frac{1}{\boldsymbol{\sigma}_d}\left((\textbf{Y}^{\text{non}})^{\top}-\frac{1}{n}\textbf{1}_n\textbf{1}_n^{\top}(\textbf{Y}^{\text{non}})^{\top} - \frac{1}{n}\textbf{y}_{d} \textbf{y}_{d}^{\top} (\textbf{Y}^{\text{non}})^{\top}\right)\left(\textbf{H}^{\text{off}}_{d, :}\right)^{\top} \\
					&= \frac{1}{\boldsymbol{\sigma}_d}\left(\textbf{Y}^{\text{non}}\right)^{\top}\left(\textbf{H}^{\text{off}}_{d, :}\right)^{\top}
					- \frac{1}{\boldsymbol{n \sigma}_d}\textbf{1}_n\textbf{1}_n^{\top}\left(\textbf{Y}^{\text{non}}\right)^{\top}\left(\textbf{H}^{\text{off}}_{d, :}\right)^{\top}
					- \frac{1}{\boldsymbol{n \sigma}_d}\textbf{y}_{d} \textbf{y}_{d}^{\top} \left(\textbf{Y}^{\text{non}}\right)^{\top}\left(\textbf{H}^{\text{off}}_{d, :}\right)^{\top}\\
					&= \frac{1}{\boldsymbol{\sigma}_d}\left(\textbf{Y}^{\text{non}}\right)^{\top}\left(\textbf{H}^{\text{off}}_{d, :}\right)^{\top}
					- \frac{1}{\boldsymbol{n \sigma}_d}\textbf{1}_n \left[ \textbf{1}_n^{\top}\left(\textbf{Y}^{\text{non}}\right)^{\top}\right]\left(\textbf{H}^{\text{off}}_{d, :}\right)^{\top}
					- \frac{1}{\boldsymbol{n \sigma}_d}\textbf{y}_{d} \left[ \textbf{y}_{d}^{\top} \left(\textbf{Y}^{\text{non}}\right)^{\top} \right]\left(\textbf{H}^{\text{off}}_{d, :}\right)^{\top}\\
					&= \frac{1}{\boldsymbol{\sigma}_d}\left(\textbf{Y}^{\text{non}}\right)^{\top}\left(\textbf{H}^{\text{off}}_{d, :}\right)^{\top} - \textbf{0} - \textbf{0} \quad {\rm \small \slash \slash \textbf{1}_n^{\top}(\textbf{Y}^{\text{non}})^{\top} = \textbf{0} \ and\ \textbf{y}_{d}^{\top} (\textbf{Y}^{\text{non}})^{\top} = \textbf{0}  \ are\ proved \ later .}\\
					&= \frac{1}{\boldsymbol{\sigma}_d}\left(\textbf{Y}^{\text{non}}\right)^{\top}\left(\textbf{H}^{\text{off}}_{d, :}\right)^{\top}
					\\ \nonumber
				\end{aligned}
			\end{equation}
		\end{small}
		
		where $\textbf{1}_n^{\top}(\textbf{Y}^{\text{non}})^{\top} = \textbf{0}$ and $\textbf{y}_{d}^{\top} (\textbf{Y}^{\text{non}})^{\top} = \textbf{0}$  are computed as follows.
		
		Note that $\textbf{Y}^{\text{linear}}= \left[ \textbf{o}_d^{\top} \textbf{y}_1,  \textbf{o}_d^{\top} \textbf{y}_2, \cdots,  \textbf{o}_d^{\top} \textbf{y}_D \right]^{\top} \textbf{o}_d^{\top} = \boldsymbol\lambda_d \textbf{o}_d^{\top}$, where $\textbf{o}_d^{\top} \textbf{y}_j \in \mathbb{R}$, $j=1, \cdots, D$.
		
		\begin{small}
			\begin{equation}
				\begin{aligned}
					\textbf{1}_n^{\top}(\textbf{Y}^{\text{non}})^{\top}
					&= \textbf{1}_n^{\top}(\textbf{Y} - \textbf{Y}^{\text{linear}})^{\top}
					=  \textbf{1}_n^{\top}\textbf{Y}^{\top} - \textbf{1}_n^{\top}  \textbf{o}_d \boldsymbol\lambda_d^{\top} \\
					&= \textbf{1}_n^{\top}\textbf{Y}^{\top} - \textbf{1}_n^{\top} \frac{\textbf{y}_d}{\|\textbf{y}_d\|}\boldsymbol\lambda_d^{\top} = \textbf{0} - 0 \cdot \frac{\boldsymbol\lambda_d^{\top}}{\|\textbf{y}_d\|} \\
					&= \textbf{0} \in \mathbb{R}^{1 \times D}. \nonumber
				\end{aligned}
			\end{equation}

			\begin{align*}
				\textbf{y}_{d}^{\top} (\textbf{Y}^{\text{non}})^{\top}
				&=\textbf{y}_{d}^{\top} \left( \textbf{Y} - \textbf{Y}^{\text{linear}}\right)^{\top}\\
				&=\textbf{y}_{d}^{\top} \left(\textbf{y}_1 - (\textbf{o}_d^{\top} \textbf{y}_1) \textbf{o}_d, \cdots,  \textbf{y}_D - (\textbf{o}_d^{\top} \textbf{y}_D) \textbf{o}_d\right)\\
				&=\textbf{y}_{d}^{\top} \left( \textbf{y}_1 -\frac{\textbf{y}_d^{\top}\textbf{y}_1}{\|\textbf{y}_d\|} \cdot \frac{\textbf{y}_d}{\|\textbf{y}_d\|}, \cdots,  \textbf{y}_D -\frac{\textbf{y}_d^{\top}\textbf{y}_D}{\|\textbf{y}_d\|} \cdot \frac{\textbf{y}_d}{\|\textbf{y}_d\|} \right)\\
				&= \left(\textbf{y}_{d}^{\top} \left(\textbf{y}_1 -\frac{\textbf{y}_d^{\top}\textbf{y}_1}{\|\textbf{y}_d\|} \cdot \frac{\textbf{y}_d}{\|\textbf{y}_d\|} \right) , \cdots,  \textbf{y}_{d}^{\top} \left(\textbf{y}_D -\frac{\textbf{y}_d^{\top}\textbf{y}_D}{\|\textbf{y}_d\|} \cdot \frac{\textbf{y}_d}{\|\textbf{y}_d\|} \right)  \right) \\
				&= \left( \textbf{y}_d^{\top}\textbf{y}_1 - \frac{\textbf{y}_d^{\top}\textbf{y}_1}{\|\textbf{y}_d\|^2} \textbf{y}_d^{\top}\textbf{y}_d, \cdots, \textbf{y}_d^{\top}\textbf{y}_D - \frac{\textbf{y}_d^{\top}\textbf{y}_D}{\|\textbf{y}_d\|^2} \textbf{y}_d^{\top}\textbf{y}_d \right) \\
				&= \left( \textbf{y}_1^{\top}\textbf{y}_d - \textbf{y}_1^{\top}\textbf{y}_d, \cdots, \textbf{y}_D^{\top}\textbf{y}_d - \textbf{y}_D^{\top}\textbf{y}_d \right) \\
				&= \left(0, \cdots, 0 \right) \\
				&= \textbf{0} \in \mathbb{R}^{1 \times D} \nonumber
			\end{align*}
		\end{small}
		
		Then,
		\begin{small}
			\begin{equation*}
				\begin{aligned}
					\frac{\partial^2 L_d ^{\text{non}}(\textbf{H}_{d,:}^{\text{off}})}{\partial \textbf{x}_{d} \partial \textbf{H}_{d, :}^{\text{off}}}
					&=\frac{\partial}{ \partial \textbf{H}_{d, :}^{\text{off}}}
					\left(\frac{\partial L_d ^{\text{non}}(\textbf{H}_{d,:}^{\text{off}})}{\partial \textbf{x}_{d}}\right)
					= \frac{\partial}{ \partial \textbf{H}_{d, :}^{\text{off}}}
					\left(\frac{1}{\boldsymbol{\sigma}_d}\left(\textbf{Y}^{\text{non}}\right)^{\top}\left(\textbf{H}^{\text{off}}_{d, :}\right)^{\top} \right)^{\top}
					=\frac{1}{\boldsymbol{\sigma}_d}\left(\textbf{Y}^{\text{non}}\right)^{\top}\\
					&=  \frac{1}{\boldsymbol{\sigma}_d} \cdot \left[\textbf{y}_{1}- \| \textbf{y}_{1} \| \cos(\textbf{y}_1, \textbf{y}_d)\textbf{o}_d , \cdots,\textbf{y}_{D} - \|\textbf{y}_{D} \| \cos(\textbf{y}_D, \textbf{y}_d) \textbf{o}_d  \right].
				\end{aligned}
			\end{equation*}
		\end{small}
		
	\end{proof}
	
	Note that {$\frac{\partial L_d^{\text{linear}}(\textbf{H}_{d,:}^{\text{off}})}{\partial \textbf{x}_d} $} reflects the interaction utility yielded by feature components in all dimensions that are linear-correlated with $\textbf{y}_d$, \emph{i.e.}, $[ \textbf{o}_d^{\top} \textbf{y}_1,  \textbf{o}_d^{\top} \textbf{y}_2, \cdots,  \textbf{o}_d^{\top} \textbf{y}_D ]^{\top}$, subject to $\textbf{o}_d=\frac{\textbf{y}_d}{\|\textbf{y}_d\|}$. {$\frac{\partial L_d^{\text{non}}(\textbf{H}_{d,:}^{\text{off}})}{\partial \textbf{x}_d}$} reflects the interaction utility between $\textbf{y}_d$ and feature components that are not linearly-correlated with $\textbf{y}_d$, \emph{i.e.}, removing linearly-correlated components, $\forall j, \textbf{y}_{j}- \|\textbf{y}_j\| \cos(\textbf{y}_j, \textbf{y}_d) \textbf{o}_d$.
	Empirically, linearly-correlated feature components usually represent similar concepts, thereby having stronger interaction utilities than non-correlated feature components. According to Theorem~\ref{th:3}, the interaction utility between linearly-correlated feature components cannot pass through the BN operation, \emph{i.e.}, {$\frac{\partial L_d^{\text{linear}}(\textbf{H}_{d,:}^{\text{off}})}{\partial \textbf{x}_d} = \textbf{0}$}. Therefore, for each dimension $d$, we can consider that a considerable ratio of the influence of {$\frac{\partial \textrm{Loss}^{\text{off}}(\textbf{H}^{\text{off}})}{\partial \textbf{x}_d}$} cannot pass through the BN operation.
	
	\subsection{Proof of Corollary 1. }
	
	\begin{corollary}\label{co:1}
		Based on Theorems~\ref{th:1}, \ref{th:2}, and \ref{th:3}, we can prove that in the training phase of a neural network with BN operations, {\rm\small$\frac{\partial ^2\textrm{Loss}^{\text{batch}}(\textbf{g},\textbf{H})}{\partial \textbf{X} \partial \textbf{g}}=\textbf{0},\frac{\partial ^2\textrm{Loss}^{\text{batch}}(\textbf{g},\textbf{H})}{\partial \textbf{X} \partial \textbf{H}^{\text{diag}}}=\textbf{0}$}, and {\rm\small$\forall d,\frac{\partial ^2\textrm{Loss}^{\text{batch}}(\textbf{g},\textbf{H})}{\partial \textbf{x}_d \partial \textbf{H}_{d,:}^{\text{off}}}=\frac{\partial ^2 L_d^{\text{linear}}(\textbf{H}_{d,:}^{\text{off}})}{\partial \textbf{x}_d \partial \textbf{H}_{d,:}^{\text{off}}}+\frac{\partial ^2 L_d^{\text{non}}(\textbf{H}_{d,:}^{\text{off}})}{\partial \textbf{x}_d \partial \textbf{H}_{d,:}^{\text{off}}}$}; where {\rm\small$\frac{\partial ^2 L_d^{\text{linear}}(\textbf{H}_{d,:}^{\text{off}})}{\partial \textbf{x}_d \partial \textbf{H}_{d,:}^{\text{off}}} = \textbf{0}$;}. In contrast, in the testing phase, {\rm\small$\frac{\partial ^2\textrm{Loss}^{\text{batch}}(\textbf{g},\textbf{H})}{\partial \textbf{X} \partial \textbf{g}}\ne\textbf{0},\frac{\partial ^2\textrm{Loss}^{\text{batch}}(\textbf{g},\textbf{H})}{\partial \textbf{X} \partial \textbf{H}^{\text{diag}}}\ne\textbf{0}$}, and {\rm\small$\frac{\partial ^2 L_d^{\text{linear}}(\textbf{H}_{d,:}^{\text{off}})}{\partial \textbf{x}_d \partial \textbf{H}_{d,:}^{\text{off}}} \ne \textbf{0}$}.
	\end{corollary}
	
	\begin{proof}
		Remember that $\textrm{Loss}^{\text{batch}}(\textbf{g}, \textbf{H})$ is decomposed into five components, as follows,
		
		
		{\small
			\begin{equation}\label{eq:Taylor_complex}
				\textrm{Loss}^{\text{batch}}(\textbf{g},\textbf{H}) =
				\textrm{Loss}^{\text{constant}} + \textrm{Loss}^{\text{grad}}(\textbf{g}) +\textrm{Loss}^{\text{diag}}(\textbf{H}^{\text{diag}}) + \textrm{Loss}^{\text{off}}(\textbf{H}^{\text{off}}) + \sum\nolimits_i R_2(\textbf{y}^{(i)}-\tilde{\textbf{y}})
			\end{equation}
		}

		Therefore, we get
		
		\begin{small}
			\begin{equation*}
				\begin{aligned}
					\frac{\partial ^2\textrm{Loss}^{\text{batch}}(\textbf{g},\textbf{H})}{\partial \textbf{X} \partial \textbf{g}}
					&= \frac{\partial ^2}{\partial \textbf{X} \partial \textbf{g}}\left(\textrm{Loss}^{\text{constant}} + \textrm{Loss}^{\text{grad}}(\textbf{g}) +\textrm{Loss}^{\text{diag}}(\textbf{H}^{\text{diag}}) + \textrm{Loss}^{\text{off}}(\textbf{H}^{\text{off}}) + \sum\nolimits_i R_2(\textbf{y}^{(i)}-\tilde{\textbf{y}})\right)\\
					&= \frac{\partial ^2}{\partial \textbf{X} \partial \textbf{g}}\left( \textrm{Loss}^{\text{grad}}(\textbf{g}) \right)
					= \frac{\partial}{\partial \textbf{g}}\left(\frac{\partial \textrm{Loss}^{\text{grad}}(\textbf{g})}{\partial \textbf{X}}  \right)
					\xlongequal{\textbf{Thm 1}} \frac{\partial}{\partial \textbf{g}}\left(\textbf{0}  \right) = \textbf{0} ;
				\end{aligned}
			\end{equation*}
		\end{small}
		
		\begin{small}
			\begin{equation*}
				\begin{aligned}
					\frac{\partial ^2\textrm{Loss}^{\text{batch}}(\textbf{g},\textbf{H})}{\partial \textbf{X} \partial \textbf{H}^{\text{diag}}}
					&= \frac{\partial ^2}{\partial \textbf{X} \partial \textbf{H}^{\text{diag}}}\left(\textrm{Loss}^{\text{constant}} + \textrm{Loss}^{\text{grad}}(\textbf{g}) +\textrm{Loss}^{\text{diag}}(\textbf{H}^{\text{diag}}) + \textrm{Loss}^{\text{off}}(\textbf{H}^{\text{off}}) + \sum\nolimits_i R_2(\textbf{y}^{(i)}-\tilde{\textbf{y}})\right)\\
					&= \frac{\partial ^2}{\partial \textbf{X} \partial \textbf{H}^{\text{diag}}}\left( \textrm{Loss}^{\text{diag}}(\textbf{H}^{\text{diag}}) \right)
					= \frac{\partial}{\partial \textbf{H}^{\text{diag}}}\left(\frac{\partial \textrm{Loss}^{\text{diag}}(\textbf{H}^{\text{diag}})}{\partial \textbf{X}}  \right)
					\xlongequal{\textbf{Thm 2}} \frac{\partial}{\partial \textbf{H}^{\text{diag}}}\left(\textbf{0}  \right) = \textbf{0} ;
				\end{aligned}
			\end{equation*}
		\end{small}
		
		\begin{small}
			\begin{equation*}
				\begin{aligned}
					\frac{\partial ^2\textrm{Loss}^{\text{batch}}(\textbf{g},\textbf{H})}{\partial \textbf{x}_d \partial  \textbf{H}_{d,:}^{\text{off}}}
					&= \frac{\partial ^2}{\partial \textbf{x}_d \partial  \textbf{H}_{d,:}^{\text{off}}}\left(\textrm{Loss}^{\text{constant}} + \textrm{Loss}^{\text{grad}}(\textbf{g}) +\textrm{Loss}^{\text{diag}}(\textbf{H}^{\text{diag}}) + \textrm{Loss}^{\text{off}}(\textbf{H}^{\text{off}}) + \sum\nolimits_i R_2(\textbf{y}^{(i)}-\tilde{\textbf{y}})\right)\\
					&= \frac{\partial ^2}{\partial \textbf{x}_d \partial  \textbf{H}_{d,:}^{\text{off}}}\left( \textrm{Loss}^{\text{off}}(\textbf{H}^{\text{off}}) \right)
					= \frac{\partial}{\partial \textbf{H}_{d,:}^{\text{off}}}\left(\frac{\partial \textrm{Loss}^{\text{off}}(\textbf{H}^{\text{off}})}{\partial \textbf{x}_d}  \right)
					\xlongequal{\textbf{Lemma 2}} \frac{\partial}{\partial \textbf{H}_{d,:}^{\text{off}}}\left(\frac{\partial L_d(\textbf{H}^{\text{off}})}{\partial \textbf{x}_d}  \right)\\
					&\xlongequal{\textbf{Lemma 2}} \frac{\partial}{\partial \textbf{H}_{d,:}^{\text{off}}}\left(\frac{\partial \left(L_d^{\text{linear}}(\textbf{H}_{d,:}^{\text{off}})+L_d^{\text{non}}(\textbf{H}_{d,:}^{\text{off}})\right)}{\partial \textbf{x}_d}  \right)
					=\frac{\partial ^2 L_d^{\text{linear}}(\textbf{H}_{d,:}^{\text{off}})}{\partial \textbf{x}_d \partial \textbf{H}_{d,:}^{\text{off}}}+\frac{\partial ^2 L_d^{\text{non}}(\textbf{H}_{d,:}^{\text{off}})}{\partial \textbf{x}_d \partial \textbf{H}_{d,:}^{\text{off}}},
				\end{aligned}
			\end{equation*}
		\end{small}
		
		where
		\begin{small}
			\begin{equation*}
				\begin{aligned}
					\frac{\partial ^2 L_d^{\text{linear}}(\textbf{H}_{d,:}^{\text{off}})}{\partial \textbf{x}_d \partial \textbf{H}_{d,:}^{\text{off}}}
					&= \frac{\partial}{\partial \textbf{H}_{d,:}^{\text{off}}}\left(\frac{\partial L_d^{\text{linear}}(\textbf{H}_{d,:}^{\text{off}})}{\partial \textbf{x}_d}  \right)
					\xlongequal{\textbf{Thm 3}} \frac{\partial}{\partial  \textbf{H}_{d,:}^{\text{off}}}\left(\textbf{0}  \right) = \textbf{0}.
				\end{aligned}
			\end{equation*}
		\end{small}
		
		In contrast, in the testing phase,  people use the  the population statistics, rather than mini-batch statistics for the standardization phase. Therefore, the means and variances are fixed during inference, the standardization phase is simply a linear transform applied to each activation. In this way, the Jacobian matrix $\textbf{J}_d = \frac{\partial \textbf{y}_d^{\top}}{\partial \textbf{x}_{d}}$ reduces to an identity matrix, $i.e.$, $\textbf{J}_d = I \in \mathbb{R}^{n \times n}$, which causes that the gradient of each loss term in Equation~(\ref{eq:Taylor_complex}) $w.r.t.$ $\textbf{x}_d$ is equal to the gradient  $w.r.t.$ $\textbf{y}_d$. Specifically,
		
		\begin{small}
			\begin{equation}
				\begin{aligned}
					\frac{\partial \textrm{Loss}^{\text{grad}}(\textbf{g})}{\partial \textbf{x}_d}
					&= \textbf{J}_d \frac{\partial \textrm{Loss}^{\text{grad}}(\textbf{g})}{\partial \textbf{y}_d}
					=I \frac{\partial \textrm{Loss}^{\text{grad}}(\textbf{g})}{\partial \textbf{y}_d}
					= I \left(\textbf{g}_d \textbf{1}_n\right) = \textbf{g}_d \textbf{1}_n,
				\end{aligned}
			\end{equation}
			
			\begin{equation}
				\begin{aligned}
					\frac{\partial \textrm{Loss}^{\text{diag}}(\textbf{H}^{\text{diag}})}{\partial \textbf{x}_d}
					&= \textbf{J}_d \frac{\partial \textrm{Loss}^{\text{diag}}(\textbf{H}^{\text{diag}})}{\partial \textbf{y}_d}
					=I \frac{\partial \textrm{Loss}^{\text{diag}}(\textbf{H}^{\text{diag}})}{\partial \textbf{y}_d}
					= \textbf{H}_{d, d}^{\text{diag}} \left(\textbf{y}_d - \tilde{\textrm{y}}_d \textbf{1}_n\right)   ,
				\end{aligned}
			\end{equation}
			
			\begin{equation}
				\begin{aligned}
					\frac{\partial  L_d^{\text{linear}}(\textbf{H}_{d,:}^{\text{off}})}{\partial \textbf{x}_d}
					&= \textbf{J}_d \frac{\partial  L_d^{\text{linear}}(\textbf{H}_{d,:}^{\text{off}})}{\partial \textbf{y}_d}
					=I \frac{\partial L_d^{\text{linear}}(\textbf{H}_{d,:}^{\text{off}})}{\partial \textbf{y}_d}
					= \left(\textbf{H}^{\text{off}}_{d, :}\textbf{Y}^{\text{linear}}\right)^{\top}.
				\end{aligned}
			\end{equation}
		\end{small}
		
		Further, we get
		\begin{small}
			\begin{equation*}
				\begin{aligned}
					\frac{\partial ^2\textrm{Loss}^{\text{batch}}(\textbf{g},\textbf{H})}{\partial \textbf{X} \partial \textbf{g}}
					&=  \frac{\partial ^2}{\partial \textbf{X} \partial \textbf{g}}\left( \textrm{Loss}^{\text{grad}}(\textbf{g}) \right)
					= \frac{\partial}{\partial \textbf{g}}\left(\frac{\partial \textrm{Loss}^{\text{grad}}(\textbf{g})}{\partial \textbf{X}}  \right)\\
					&= \frac{\partial}{\partial \textbf{g}}\left( \left[
					\left(\frac{\partial \textrm{Loss}^{\text{grad}}(\textbf{g})}{\partial \textbf{x}_1} \right)^{\top};
					\left(\frac{\partial \textrm{Loss}^{\text{grad}}(\textbf{g})}{\partial \textbf{x}_2} \right)^{\top};
					\ldots;
					\left(\frac{\partial \textrm{Loss}^{\text{grad}}(\textbf{g})}{\partial \textbf{x}_D}  \right)^{\top}
					\right] \right) \\
					&= \frac{\partial}{\partial \textbf{g}}\left( \left[
					\textbf{g}_1 \textbf{1}_{n}^{\top};
					\textbf{g}_2 \textbf{1}_{n}^{\top};
					\ldots;
					\textbf{g}_D \textbf{1}_{n}^{\top}
					\right] \right)
					=  \frac{\partial}{\partial \textbf{g}}\left( \textbf{g}  \textbf{1}_{n}^{\top}  \right) \neq  \textbf{0} \in \mathbb{R}^{D^2 \times n},\\
				\end{aligned}
			\end{equation*}
			
			\begin{equation*}
				\begin{aligned}
					\frac{\partial ^2\textrm{Loss}^{\text{batch}}(\textbf{g},\textbf{H})}{\partial \textbf{X} \partial \textbf{H}^{\text{diag}}}
					&=  \frac{\partial ^2}{\partial \textbf{X} \partial \textbf{H}^{\text{diag}}}\left( \textrm{Loss}^{\text{diag}}(\textbf{H}^{\text{diag}}) \right)
					= \frac{\partial}{\partial \textbf{H}^{\text{diag}}}\left(\frac{\partial \textrm{Loss}^{\text{diag}}(\textbf{H}^{\text{diag}})}{\partial \textbf{X}}  \right)\\
					&= \frac{\partial}{\partial \textbf{H}^{\text{diag}}}\left( \frac{\partial \textrm{Loss}^{\text{diag}}(\textbf{H}^{\text{diag}})}{\partial \textbf{X}} = \left[
					\frac{\partial \textrm{Loss}^{\text{diag}}(\textbf{H}^{\text{diag}})}{\partial \textbf{x}_1},
					\frac{\partial \textrm{Loss}^{\text{diag}}(\textbf{H}^{\text{diag}})}{\partial \textbf{x}_2},
					\ldots;
					\frac{\partial \textrm{Loss}^{\text{diag}}(\textbf{H}^{\text{diag}})}{\partial \textbf{x}_D}
					\right]^{\top} \right) \\
					&= \frac{\partial}{\partial \textbf{H}^{\text{diag}}}\left( \left[
					\textbf{H}_{1, 1}^{\text{diag}} \left(\textbf{y}_1^{\top} - \tilde{\textrm{y}}_1 \textbf{1}_n^{\top}\right);
					\textbf{H}_{2, 2}^{\text{diag}} \left(\textbf{y}_2^{\top} - \tilde{\textrm{y}}_2 \textbf{1}_n^{\top}\right);
					\ldots;
					\textbf{H}_{D, 1}^{\text{diag}} \left(\textbf{y}_D^{\top} - \tilde{\textrm{y}}_D \textbf{1}_n^{\top}\right)
					\right] \right)\\
					&=  \frac{\partial}{\partial \textbf{H}^{\text{diag}}}\left(  \textbf{H}^{\text{diag}} \left(\textbf{Y} - \tilde{\textbf{y}}  \textbf{1}_{n}^{\top} \right) \right)
					\neq  \textbf{0} \in \mathbb{R}^{D^2 \times D\cdot n},\\
				\end{aligned}
			\end{equation*}
			
			\begin{equation*}
				\begin{aligned}
					\frac{\partial ^2 L_d^{\text{linear}}(\textbf{H}_{d,:}^{\text{off}}) }{\partial \textbf{x}_d \partial \textbf{H}_{d,:}^{\text{off}}}
					&= \frac{\partial}{\partial \textbf{H}_{d,:}^{\text{off}}}\left(\frac{\partial  L_d^{\text{linear}}(\textbf{H}_{d,:}^{\text{off}})}{\partial \textbf{x}_d}  \right)
					= \frac{\partial}{\partial \textbf{H}_{d,:}^{\text{off}}}\left( \left(\textbf{H}^{\text{off}}_{d, :}\textbf{Y}^{\text{linear}}\right)^{\top} \right)
					\neq  \textbf{0} \in \mathbb{R}^{n \times D}.\\
				\end{aligned}
			\end{equation*}
		\end{small}
		
	\end{proof}
	
	\section{Conclusions of Case 2 and their proofs}\label{appendix:case2}
	
	In this section, we discuss the general case where different samples in a mini-batch have different analytic formulas of loss functions. In this case, we can still get a shared Taylor expansion series by \textit{Lemma 3}.
	
	Just like Case 1, we prove that $\textrm{Loss}^{\text{grad}}(\bbar{\textbf{g}})$ and $\textrm{Loss}^{\text{diag}}(\bbar{\textbf{H}}^{\text{diag}})$ have no gradients on features before the BN operation, and $\textrm{Loss}^{\text{off}}(\bbar{\textbf{H}}^{\text{off}})$ does not have strong effects on features before the BN operation.
	
	\begin{proof}
		Based on Equation~(\ref{eq: gene decompose taylor}) in \textit{Lemma 3}, we get an analytic formula shared by all samples in the mini-batch, $i.e.$, $\textrm{Loss}^{\text{share}} =\textrm{Loss}(\tilde{\textbf{y}}) + \textrm{Loss}^{\text{grad}}(\bbar{\textbf{g}}) +\textrm{Loss}^{\text{diag}}(\bbar{\textbf{H}}^{\text{diag}}) + \textrm{Loss}^{\text{off}}(\bbar{\textbf{H}}^{\text{off}})$, which reduces to Case1. Theorems~\ref{th:1}, \ref{th:2}, and \ref{th:3} tells that $\frac{\partial \textrm{Loss}^{\text{grad}}(\bbar{\textbf{g}})}{\partial \textbf{x}_d} = \textbf{0}$,  $\frac{\partial \textrm{Loss}^{\text{diag}}(\bbar{\textbf{H}}^{\text{diag}})}{\partial \textbf{x}_{d}} = \textbf{0}$, $\frac{\partial L_d^{\rm linear}(\bbar{\textbf{H}}_{d,:}^{\text{off}})}{\partial \textbf{x}_{d}}= \textbf{0}$.
	\end{proof}
	
	\section{Proof of Discussion: the third and higher-order derivatives of the sigmoid function have small strengths when the classification is confident}\label{appendix:sigmoid}
	
	In this section, we prove that the third and higher-order derivatives of the sigmoid function have small strengths when the classification is confident, which is mentioned in the discussion in the main paper on applying the same Taylor series expansion to all samples.
	
	
	\begin{proof}
		Let us take the loss function as $L= -\log \textrm{Sigmoid}(z) = - \log \frac{e^z}{1+e^z} = -z + \log (1+e^z)$ for example, where $z \in \mathbb{R}$ is the output of the network. Let $L^{(m)}(z) = \frac{d^m L}{dz}(z)$ denote the $m$-th derivative of $L$ $w.r.t.$ $z$.
		We use the mathematical induction to prove that the highest order term of $L^{(m)}(z)$ is $\frac{{(-1)}^{m} e^{(m-1)z}}{(1+e^z)^m}$, thus we can write $L^{(m)}(z)$ as $L^{(m)}(z) = \frac{\sum_{k=1}^{m-2} a_{k}^{(m)} e^{k\cdot z} + {(-1)}^{m} e^{(m-1)z}}{(1+e^z)^m}$, where $a_{k}^{(m)}$ is the coefficient of $e^{k\cdot z}$ in  $L^{(m)}(z)$.

		\textit{Base case}: When $m = 1$,
		
		\begin{align}
			&L^{\prime}(z) =  \frac{-1}{1+e^z} = \frac{(-1)^{1}e^{(1-1)z}}{(1+e^z)^1}, \\
		\end{align}
		
		Actually, cases for $m = 2,3,4$ can also be verified directly,
		
		\begin{align}
			&L^{\prime \prime}(z) = \frac{e^z}{(1+e^z)^2},\\
			&L^{(3)}(z) = \frac{e^z-e^{2z}}{(1+e^z)^3},\\
			&L^{(4)}(z) = \frac{e^z-4e^{2z}+e^{3z}}{(1+e^z)^4}.
		\end{align}
		
		\textit{Inductive step}: Assuming that the highest order term of $L^{(m)}(z)$ is $\frac{{(-1)}^{m} e^{(m-1)z}}{(1+e^z)^m}$, $L^{(m)}(z) = \frac{\sum_{k=1}^{m-2} a_{k}^{(m)} e^{k\cdot z} + {(-1)}^{m} e^{(m-1)z}}{(1+e^z)^m}$. Then
		the derivative of $ \frac{{(-1)}^{m}e^{(m-1)z}}{(1+e^z)^m}$ gives the highest order term of $L^{(m+1)}(z)$,
		
		\begin{equation*}
			\begin{aligned}
				& \qquad   \frac{d}{dz}\left({(-1)}^{m} \cdot \frac{ e^{(m-1)z}}{(1+e^z)^m}\right) \\
				&=  {(-1)}^{m} \cdot \frac{ \left(e^{(m-1)z}\cdot(m-1)\cdot(1+e^z)^m - e^{(m-1)z}\cdot m\cdot (1+e^z)^{m-1}e^z\right)}{(1+e^z)^{2m}}\\
				&= {(-1)}^{m} \cdot \frac{e^{(m-1)z}\cdot(m-1)\cdot(1+e^z) - e^{(m-1)z}\cdot m\cdot e^z}{(1+e^z)^{m+1}}\\
				&= {(-1)}^{m} \cdot \frac{(m-1)\cdot e^{(m-1)z} + (m-1)\cdot e^{mz} -  m\cdot e^{mz}}{(1+e^z)^{m+1}}\\
				&= {(-1)}^{m} \cdot \frac{(m-1)\cdot e^{(m-1)z}  - e^{mz}}{(1+e^z)^{m+1}}\\
				&= \frac{{(-1)}^{m} \cdot (m-1)\cdot e^{(m-1)z}}{(1+e^z)^{m+1}} + \frac{{(-1)}^{m+1} \cdot  e^{mz}}{(1+e^z)^{m+1}}.
			\end{aligned}
		\end{equation*}
		
		\textit{Conclusion}: Since both the base case and the inductive step have been proven to be true, we have proved that the highest order term of $L^{(m)}(z)$ is $\frac{{(-1)}^{m} e^{(m-1)z}}{(1+e^z)^m}$,  $L^{(m)}(z)$ can be written as $L^{(m)}(z) = \frac{\sum_{k=1}^{m-2} a_{k}^{(m)} e^{k\cdot z} + {(-1)}^{m} e^{(m-1)z}}{(1+e^z)^m}$.
		
		Further, $\lim\limits_{z \rightarrow + \infty} \frac{\sum_{k=1}^{m-2} a_{k}^{(m)} e^{k\cdot z} + {(-1)}^{m} e^{(m-1)z}}{(1+e^z)^m} =0 $ and $\lim\limits_{z \rightarrow  -\infty} \frac{\sum_{k=1}^{m-2} a_{k}^{(m)} e^{k\cdot z} + {(-1)}^{m} e^{(m-1)z}}{(1+e^z)^m} =0 $ mean that, when $z \in \mathbb{R}$ is far away from zero, each $m$-th derivative  $L^{(m)}(z)$ has a small absolute value.
		
		In this way,
		when the classification is confident, the output of the network $z \in \mathbb{R}$ would be far away from zero, thereby the third and higher-order derivatives of the sigmoid function have small strengths according to above analysis.
	\end{proof}

	\section{Proof of Discussion: our conclusions can be extended to DNNs having no second derivatives or zero second derivatives}\label{appendix:zero}
	
	In this section, we prove that if the neural network has no second derivations or has zero second derivatives, all our conclusions are still valid for the equivalent Hessian matrix.
	
	\begin{proof}
		(1) Firstly, the finite difference method~\citep{2009Scientific} it a classic method to compute derivatives numerically, even for functions which are not differentiable. Thus, we can compute second derivatives by the finite difference method to get an equivalent Hessian matrix.
		
		(2) Secondly, all proofs of theorems in this paper are not involved with how we compute the Hessian matrix. Thus all our conclusions are valid for the equivalent Hessian matrix. In fact, the proof is the same as Section~\ref{appendix:lemma} and ~\ref{appendix:theorem}.
	\end{proof}

	\section{Proof of the reason of blindness}\label{appendix:reason}

	In this section, we discuss the reason for the blindness of the BN operation. According to Equations~(\ref{eq:bn1}) and (\ref{eq:bn2}), the BN operation contains two phases, the affine transformation phase and the standardization phase. We prove that derivatives of $\boldsymbol{\mu}$ and $\boldsymbol{\sigma}$ in the standardization phase eliminate the influence of $\textrm{Loss}^{\text{grad}}(\textbf{g}),\textrm{Loss}^{\text{diag}}(\textbf{H}^{\text{diag}})$, and $L_d^{\text{linear}}(\textbf{H}_{d,:}^{\text{off}})$, as follows.
	
	\begin{proof}
		As we can see in the proofs for Theorems~\ref{th:1}, \ref{th:2}, \ref{th:3},  and Corollary~\ref{co:1},  $\textrm{Loss}^{\text{grad}}(\textbf{g}),\textrm{Loss}^{\text{diag}}(\textbf{H}^{\text{diag}})$, and $L_d^{\text{linear}}(\textbf{H}_{d,:}^{\text{off}})$ still have influence on each $\textbf{y}_d$, their gradients $w.r.t.$ $\textbf{y}_d$ are not zero. However, the gradient of each loss term $w.r.t.$ $\textbf{x}_d$ become zero, which is resulted by the elimination effect of the Jacobian matrix $\textbf{J}_d  = \frac{1}{\boldsymbol{\sigma}_d}\left(I-\frac{1}{n}\textbf{1}_n\textbf{1}_n^{\top} - \frac{1}{n \boldsymbol{\sigma}_d^2 }\left(\textbf{x}_{d} - \boldsymbol{\mu}_d\textbf{1}_n\right)\left(\textbf{x}_{d} - \boldsymbol{\mu}_d\textbf{1}_n\right)^{\top} \right) \in \mathbb{R}^{n \times n}$. Specifically, let \textit{loss} denote any one of the above loss terms, we have
		
		\begin{equation*}
			\begin{aligned}
				\frac{\partial \textit{loss}}{\partial \textbf{x}_{d}} = = \frac{\partial \textbf{y}_d^{\top}}{\partial \textbf{x}_{d}}\frac{\partial \textit{loss} }{\partial \textbf{y}_{d}}= \textbf{J}_d\frac{\partial \textit{loss} }{\partial \textbf{y}_{d}},
			\end{aligned}
		\end{equation*}
		
		For $\textrm{Loss}^{\text{grad}}(\textbf{g})$, $\frac{\partial \textrm{Loss}^{\text{grad}}(\textbf{g})}{\partial \textbf{y}_d}
		= \textbf{g}_d \textbf{1}_n \in \mathbb{R}^{n}$. $\textbf{J}_d \cdot \textbf{1}_n = \textbf{0}$, thereby  $\frac{\partial \textrm{Loss}^{\text{grad}}(\textbf{g})}{\partial \textbf{x}_d}
		= \textbf{0}$;
		for $\textrm{Loss}^{\text{diag}}(\textbf{H}^{\text{diag}})$, $ \frac{\partial \textrm{Loss}^{\text{diag}}(\textbf{H}^{\text{diag}})}{\partial \textbf{y}_d} = \textbf{H}_{d, d}^{\text{diag}} \left(\textbf{y}_d - \tilde{\textrm{y}}_d \textbf{1}_n\right)$. $\textbf{J}_d \cdot \textbf{1}_n = \textbf{0}$ and $\textbf{J}_d \cdot \textbf{y}_d = \textbf{0}$, thereby $ \frac{\partial \textrm{Loss}^{\text{diag}}(\textbf{H}^{\text{diag}})}{\partial \textbf{x}_d} = \textbf{0}$;
		for $L_d^{\text{linear}}(\textbf{H}_{d,:}^{\text{off}})$,  the gradient  $w.r.t.$ $\textbf{x}_d$ becoming zero is also caused by the effect that $\textbf{J}_d \cdot \textbf{y}_d = \textbf{0}$.
		
		However, as we have shown in the proof of Corollary~\ref{co:1},  when the means and variances are fixed,  the Jacobian matrix $\textbf{J}_d = \frac{\partial \textbf{y}_d^{\top}}{\partial \textbf{x}_{d}}$ reduces to an identity matrix, which would not eliminate each $\frac{\partial \textit{loss} }{\partial \textbf{y}_{d}}$ to zero. Therefore,  we can attribute the blindnesss problem  to derivatives of $\boldsymbol{\mu}$ and $\boldsymbol{\sigma}$ in the standardization phase.
		
	\end{proof}

	\section{More experimental settings and results}\label{appendix:exp}
	
	\begin{table}[t]
		\centering
		\caption{Architectures of the revised AlexNets for experimental verification of Theorem~\ref{th:3}.}
		\vspace{1pt}
		\label{tab:alexnet}
		\resizebox{0.55\linewidth}{!}{\begin{tabular}{cccc}
				\toprule
				Traditional AlexNet & AlexNet-1     &  AlexNet-2 &  AlexNet-3 \\
				\midrule
				Conv & Conv & Conv & Conv\\
				ReLU & ReLU &  ReLU & ReLU \\
				MaxPool & MaxPool & MaxPool & MaxPool \\
				Conv & Conv & Conv & Conv\\
				ReLU & ReLU &  ReLU & ReLU \\
				MaxPool & MaxPool & MaxPool & MaxPool \\
				Conv& Conv & Conv & Conv\\
				ReLU & ReLU &  ReLU & ReLU \\
				Conv& Conv & Conv & Conv\\
				ReLU & ReLU &  ReLU & ReLU \\
				Conv& Conv & Conv & Conv\\
				ReLU & ReLU &  ReLU & ReLU \\
				MaxPool & MaxPool & MaxPool & MaxPool \\
				FC & FC & FC & FC \\
				ReLU, dropout & ReLU &  ReLU & ReLU \\
				FC & FC & FC & FC \\
				ReLU, dropout & ReLU &  ReLU & ReLU \\
				FC & FC & FC & FC \\
				& ReLU &  ReLU & ReLU \\
				& FC & FC & FC \\
				& ReLU &  ReLU & ReLU \\
				& FC & FC & FC \\
				& ReLU &  ReLU & ReLU \\
				& FC & FC & \textbf{BN} \\
				& ReLU &  ReLU & FC \\
				& FC & \textbf{BN} &  ReLU\\
				& ReLU & FC  &  FC\\
				& \textbf{BN} & ReLU & ReLU \\
				& FC & FC & FC \\
				\bottomrule
		\end{tabular}}
		\vspace{-5pt}
	\end{table}
	
	\subsection{Architectures of the revised AlexNets and the revised LeNets for experimental verification of Theorem~\ref{th:3}}
	
	This subsection provides more experimental details about the verification of Theorem~\ref{th:3} in Section 4 in the paper. In order to comprehensively test the BN on different layers, we revised the AlexNet~\cite{krizhevsky2012imagenet} by adding five additional FC layers before the top FC layer, and revised the LeNet~\cite{lecun1989backpropagation} by adding seven additional FC layers before the top FC layer. Each added FC layer contained 20 neurons and followed by a ReLU layer. For each DNN, we added a BN layer (where $\varepsilon=1e-20$) before the 1st, 2nd, and 3rd top FC layers, respectively, to construct \textit{AlexNet-1, AlexNet-2, AlexNet-3} and \textit{LeNet-1, LeNet-2, LeNet-3}, as shown in Tables~\ref{tab:alexnet} and~\ref{tab:lenet}. All DNNs were trained on the MNIST dataset~\cite{lecun1998gradient}.
	
	Note that accept for adding additional FC layers on the top of the traditional AlexNet and the traditional LeNet, we also made the following revisions to improve the training efficiency of the above DNNs. For the revised architectures of AlexNe, we removed the dropout operation to facilitate the calculation of the Hessian matrix~\cite{cohen2020gradient}. For the revised architectures of LeNet, we replaced all Sigmoid functions with currently-widely-used ReLU functions, and replaced all AveragePool operations with currently-widely-used MaxPool operations.

	\begin{table}[t]
		\centering
		\caption{Architectures of the revised LeNets for experimental verification of Theorem~\ref{th:3}.}
		\vspace{1pt}
		\label{tab:lenet}
		\resizebox{0.5\linewidth}{!}{\begin{tabular}{cccc}
				\toprule
				Traditional LeNet & LeNet-1     &  LeNet-2 &  LeNet-3 \\
				\midrule
				Conv & Conv & Conv & Conv\\
				Sigmoid & ReLU &  ReLU & ReLU \\
				AveragePool & MaxPool & MaxPool & MaxPool \\
				Conv& Conv & Conv & Conv\\
				Sigmoid & ReLU &  ReLU & ReLU \\
				AveragePool & MaxPool & MaxPool & MaxPool \\
				FC & FC & FC & FC \\
				Sigmoid & ReLU &  ReLU & ReLU \\
				FC & FC & FC & FC \\
				Sigmoid & ReLU &  ReLU & ReLU \\
				FC & FC & FC & FC \\
				& ReLU &  ReLU & ReLU \\
				& FC & FC & FC \\
				& ReLU &  ReLU & ReLU \\
				& FC & FC & FC \\
				& ReLU &  ReLU & ReLU \\
				& FC & FC & FC \\
				& ReLU &  ReLU & ReLU \\
				& FC & FC & FC \\
				& ReLU &  ReLU & ReLU \\
				& FC & FC & \textbf{BN} \\
				& ReLU &  ReLU & FC \\
				& FC & \textbf{BN} &  ReLU\\
				& ReLU & FC  &  FC\\
				& \textbf{BN} & ReLU & ReLU \\
				& FC & FC & FC \\
				\bottomrule
		\end{tabular}}
		\vspace{-5pt}
	\end{table}

	\subsection{Details about how to inverse features in RealNVP-LN}
	
	This subsection provides details about how to inverse features in RealNVP-LN, which was constructed by replacing all BN operations with LN operations in the traditional architecture of RealNVP. Given an input image, the traditional RealNVP (\emph{i.e.}, the RealNVP-BN in the paper) contained two learning phases, the forward phase that outputted a latent variable and predicted the log-likelihood of the input image estimated by the model, and the inverse phase that inversed the input image from a output variable of the forward phase. In the inverse phase of the RealNVP-BN, the BN operation depended on a population mean and a population variance to inverse the normalized features to features before the BN operation. The population mean and the population variance were computed based on paremeters $\mathbf{\mu}$ and $\mathbf{\sigma}$ of all mini-batches, thereby completely representing all mini-batches. In order to inverse features in RealNVP-LN, each time we performed the LN operation, we performed the BN operation in parallel, and updated a population mean and a population variance after the conduction of the BN operation. Then, in the inverse phase of the RealNVP-LN, we made the LN operation depend on such population mean and such population variance of the BN operaiton.
	

	\subsection{More RealNVP models with various revised architectures}

	This subsection provides results on more RealNVP models with various revised architectures, which also yielded the similar conclusion as the conclusion in the paper, \emph{i.e.}, the RealNVP-BN could not significantly distinguish real images and interpolated images. Specifically, we revised the traditional RealNVP architecture to obtain two different versions of RealNVP, \emph{i.e.}, a shallow one (termed RealNVP-shallow) and a narrow one (termed RealNVP-narrow). The RealNVP-shallow was constructed by reducing the number of residual blocks in each residual module to 1/2 of that of the traditional RealNVP, and the RealNVP-narrow was constructed by reducing the number of feature maps of each residual block to 1/2 of that of the traditional RealNVP. For the RealNVP-shallow and the RealNVP-narrow, we also constructed the DNN-BN and the DNN-LN. According to Figure~\ref{fig:app_flow}, we could obtain the same conclusion that the RealNVP-shallow-LN and the RealNVP-narrow-LN usually assigned much higher log-likelihood with real images (\emph{i.e.}, images at the points of $\alpha=0$ and $\alpha=1$) than interpolated images. In comparison, the RealNVP-shallow-BN and the RealNVP-narrow-BN could not significantly distinguish real images and interpolated images.

	\begin{figure}
		\begin{center}
			\includegraphics[width=\linewidth]{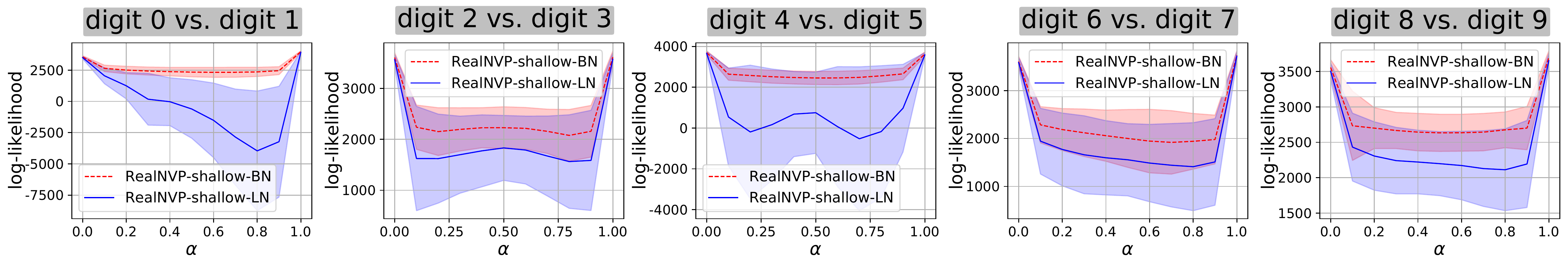}
			\includegraphics[width=\linewidth]{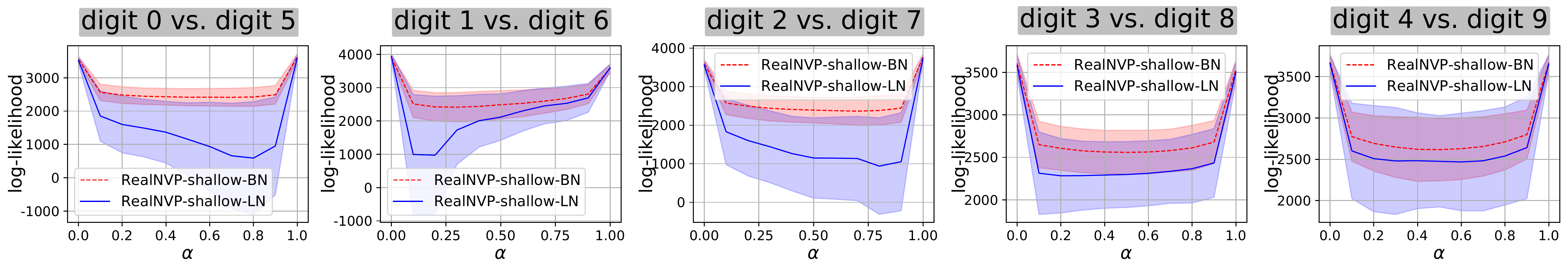}
			\includegraphics[width=\linewidth]{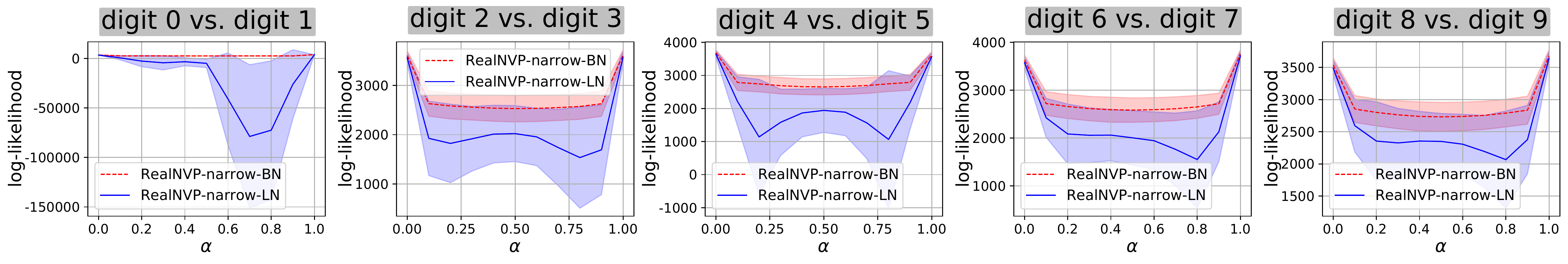}
			\includegraphics[width=\linewidth]{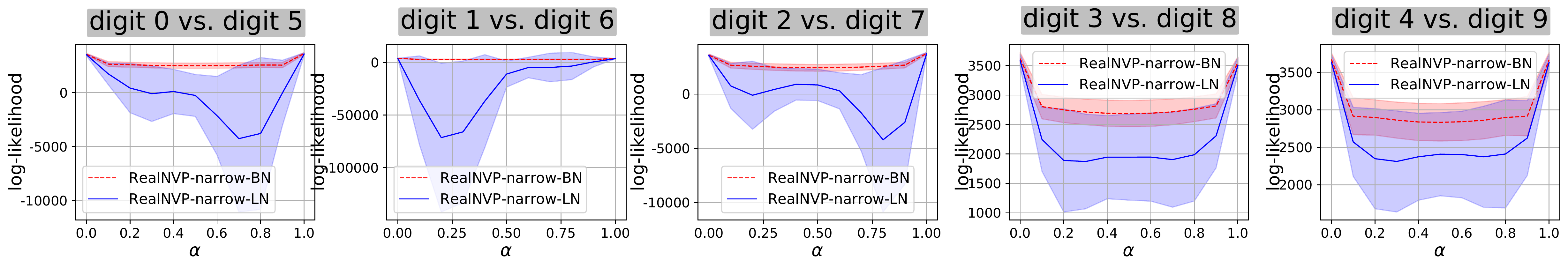}
		\end{center}
		\caption{Log-likelihood of real images (at $\alpha=0$ and $\alpha=1$) and interpolated images generated by the RealNVP-shallow-BN, RealNVP-shallow-LN, RealNVP-narrow-BN, and RealNVP-narrow-LN. The shaded area represents the standard deviation.}
		\vspace{-5pt}
		\label{fig:app_flow}
	\end{figure}

\end{document}